\documentclass[10pt]{article} 
\usepackage[accepted]{tmlr}

\usepackage{amsmath,amsfonts,bm}

\def\eqref#1{equation~\ref{#1}}

\def\1{\bm{1}}

\DeclareMathAlphabet{\mathsfit}{\encodingdefault}{\sfdefault}{m}{sl}
\SetMathAlphabet{\mathsfit}{bold}{\encodingdefault}{\sfdefault}{bx}{n}

\usepackage{enumitem}
\usepackage[normalem]{ulem}
\usepackage{lineno}

\usepackage{algorithm}
\usepackage{algcompatible} 
\usepackage{amsthm}
\usepackage{amsmath,amssymb,amsfonts}
\usepackage{bm}
\usepackage{mathtools}
\usepackage{siunitx}

\usepackage{graphicx}
\usepackage{float}
\usepackage{caption}
\usepackage{subcaption}

\usepackage{hyperref}
\usepackage{url}

\hypersetup{
    colorlinks=true,  
    linkcolor=black,   
    citecolor=black,   
    urlcolor=black,     
}

\usepackage{xr-hyper}

\newtheorem{theorem}{Theorem}

\newtheorem{corollary}{Corollary}
\newtheorem{remark}{Remark}
\newtheorem{definition}{Definition}

\usepackage{booktabs}
\usepackage{array}
\usepackage{multirow}
\usepackage{multicol}
\usepackage{makecell}
\usepackage{microtype}
\newcolumntype{B}{!{\vrule width 1pt}}


\usepackage[table]{xcolor}

\newif\ifshowchanges
\showchangesfalse   

\definecolor{revision}{RGB}{200,0,0}    
\definecolor{bestOurs}{RGB}{0,0,255}    
\definecolor{bestOther}{RGB}{255,0,0}   
\definecolor{btmgray}{gray}{0.9}
\definecolor{revdelLight}{RGB}{230,230,255} 

\ifshowchanges
    \definecolor{revaddLight}{RGB}{255,230,230} 
\else
    \definecolor{revaddLight}{RGB}{255,255,255} 
\fi

\newcommand{\revadd}[1]{%
    \ifshowchanges
        {\color{revision}#1}%
    \else
        #1%
    \fi
}

\newcommand{\bestOurs}[1]{\textbf{\textcolor{bestOurs}{#1}}}
\newcommand{\bestOther}[1]{\textbf{\textcolor{bestOther}{#1}}}

\title{Geometric Characterisation and Structured Trajectory Surrogates for Clinical Dataset Condensation\thanks{An earlier version of this work appeared as a preprint.}}

\author{\name Pafue Christy Nganjimi 
      \email pafue.nganjimi@eng.ox.ac.uk \\
      \addr Department of Engineering Science\\
      University of Oxford, UK
      \AND
      \name Andrew Soltan 
      \email andrew.soltan@oncology.ox.ac.uk \\
      \addr Nuffield Department of Primary Care\\
      University of Oxford, UK
      \AND
      \name Danielle Belgrave 
      \email danielle.x.belgrave@gsk.com\\
      \addr GlaxoSmithKline\\
      London, UK
      \AND
      \name Lei Clifton 
      \email lei.clifton@phc.ox.ac.uk\\
      \addr Nuffield Department of Primary Care\\
      University of Oxford, UK
      \AND
      \name David Clifton 
      \email david.clifton@eng.ox.ac.uk\\
      \addr Department of Engineering Science\\
      University of Oxford, UK
      \AND
      \name Anshul Thakur \email anshul.thakur@eng.ox.ac.uk\\
      \addr Department of Engineering Science\\
      University of Oxford, UK
      }

\def\month{10}  
\def\year{2026} 
\def\openreview{\url{https://openreview.net/forum?id=9qTlh7f0z3}} 

\begin{document}

\maketitle
\begin{abstract}
Dataset condensation constructs compact synthetic datasets that retain the training utility of large real-world datasets, enabling efficient model development and potentially supporting downstream research in governed domains such as healthcare. Trajectory matching is a widely used condensation approach that supervises synthetic data using changes in model parameters observed during training on real data, yet the structure of this supervision signal remains poorly understood. In this paper, we provide a geometric characterisation of trajectory matching, showing that a fixed synthetic dataset can only reproduce a limited span of such training-induced parameter changes. When the resulting supervision signal is spectrally broad, this creates a conditional representability bottleneck. Motivated by this mismatch, we propose Bézier Trajectory Matching (BTM), which replaces \revadd{stochastic gradient descent} \revadd{(}SGD\revadd{)} trajectories with quadratic Bézier trajectory surrogates between initial and final model states. These surrogates are optimised to reduce average loss along the path while replacing broad SGD-derived supervision with a more structured, lower-rank signal that is better aligned with the optimisation constraints of a fixed synthetic dataset, and they substantially reduce trajectory storage. Experiments on five clinical datasets demonstrate that BTM consistently matches or improves upon standard trajectory matching, with the largest gains in low-prevalence and low-synthetic-budget settings. These results indicate that effective trajectory matching depends on structuring the supervision signal rather than reproducing stochastic optimisation paths.
\end{abstract}

\section{INTRODUCTION}
\label{sec:intro}

Modern machine learning systems increasingly rely on large-scale datasets \citep{pastorino2019benefits}, yet storing, sharing, and repeatedly training on such datasets remains computationally and operationally expensive. In clinical settings, this challenge is compounded by strict governance and access constraints on electronic health records (EHRs) \citep{thakur2024data, shabani2019gdpr}, which limit data sharing, restrict multi-centre collaboration, and can make external validation across diverse real-world cohorts more difficult. Dataset condensation \citep{wang2018dataset, zhao2020dataset} addresses this problem by constructing compact synthetic datasets that retain the training utility of the original data. In governed domains, such condensed datasets may support downstream methodological research and model development where direct data access is restricted. However, they do not by themselves provide formal privacy guarantees, and their safe deployment would require combining them with differential privacy \citep{carlini2022no}.

\begin{figure*}[t]
    \centering
    \begin{subfigure}[t]{0.32\linewidth}
        \centering
        \includegraphics[width=\linewidth]{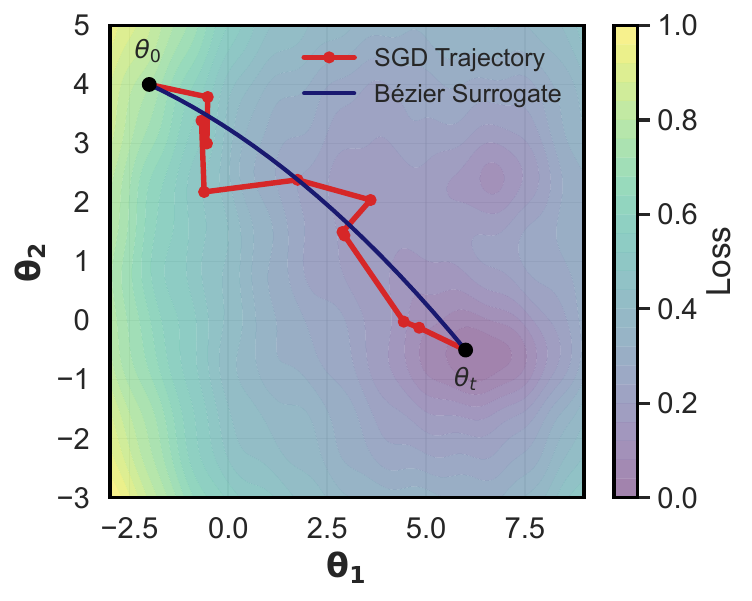}
        \subcaption{Parameter-space trajectories}
        \label{fig:tm_param_space}
    \end{subfigure}
    \hfill
    \begin{subfigure}[t]{0.32\linewidth}
        \centering
        \includegraphics[width=\linewidth]{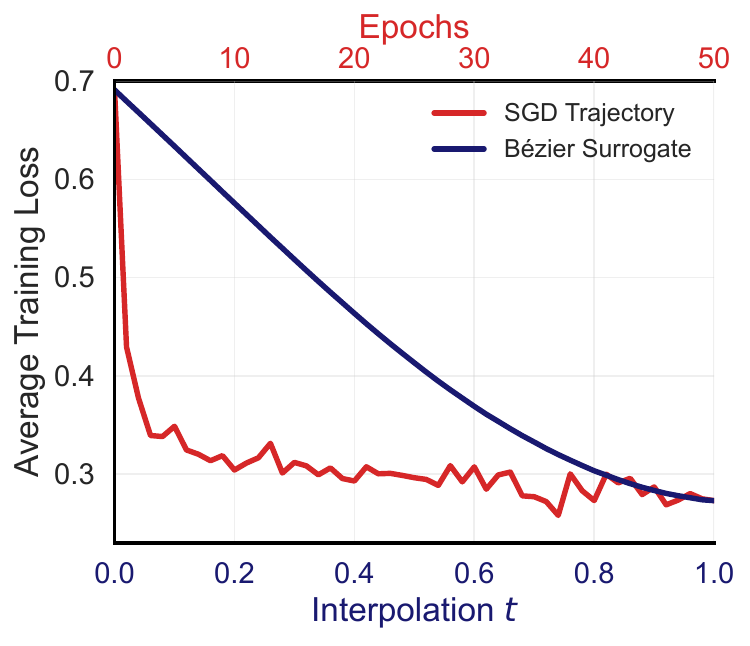}
        \subcaption{Training loss profiles}
        \label{fig:tm_loss_path}
    \end{subfigure}
    \hfill
    \begin{subfigure}[t]{0.32\linewidth}
        \centering
        \includegraphics[width=\linewidth]{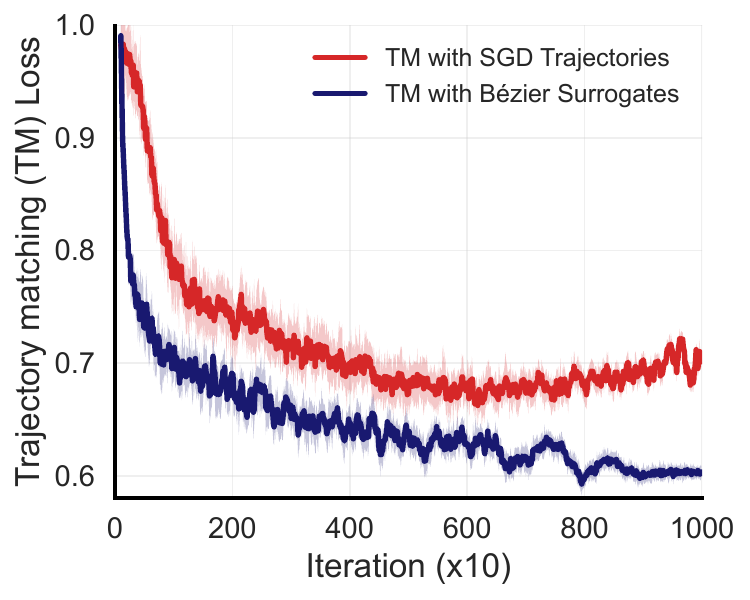}
        \subcaption{Trajectory matching optimisation}
        \label{fig:tm_convergence}
    \end{subfigure}
    \caption{
    \textbf{Illustrative comparison between raw SGD teacher trajectories and Bézier trajectory surrogates used in BTM.}
    (a) Between the same initial and final model states, raw SGD can follow jagged, directionally variable paths, whereas a Bézier surrogate provides a smoother low-complexity connection.
    (b) The two forms of supervision also differ functionally. SGD training exhibits irregular epoch-to-epoch loss behaviour, while the Bézier surrogate defines a smooth loss profile between the same endpoints.
    (c) In trajectory matching, replacing raw SGD segments with structured surrogate supervision can yield a more stable optimisation objective and faster convergence.
    }
    \label{fig:tm_vs_sgd}
\end{figure*}

Trajectory matching \citep{cazenavette2022dataset,du2023ftd, guo2024datm} is a widely used approach to dataset condensation. Rather than matching gradients or feature statistics at a single point, trajectory matching supervises synthetic data using changes in model parameters observed during training on real data. By transferring optimisation dynamics into synthetic datasets, trajectory matching provides a direct mechanism for encoding task-relevant structure. 
Despite strong empirical performance, progress in trajectory matching has largely relied on empirical heuristics—such as trajectory regularisation, selecting which parts of the training trajectory to match, \revadd{trajectory} smoothing \citep{du2023ftd, guo2024datm, zhong2025towards}, and scalability improvements \citep{cui2023scaling}—without a principled understanding of the supervision signal being matched.

This leaves a basic question unresolved. Which aspects of real-data training trajectories can a fixed synthetic dataset actually reproduce? In trajectory matching, supervision is conveyed through changes in model parameters observed during training on real data, but the structure of this signal and its representational demands are not well understood. Without such an understanding, it remains unclear why trajectory matching performs well in some settings, degrades in others, or how its supervision should be structured to make condensation more effective.

In this work, we answer this question through a geometric characterisation of trajectory matching. We show that, for any fixed synthetic dataset, the trajectory matching residual is lower bounded by how well the parameter changes induced by training on real data can be approximated within the gradient span reachable from that dataset. This leads to a conditional representability bottleneck. When the resulting supervision signal is spectrally broad relative to the reachable span, some components necessarily lie outside what the synthetic dataset can reproduce, and a non-zero matching error remains unavoidable.

We then study this mismatch empirically for the supervision signal induced by standard \revadd{stochastic gradient descent} \revadd{(}SGD\revadd{)} training. We find that the parameter changes used by trajectory matching often exhibit broad spectral structure, with non-negligible mass distributed across many weakly aligned directions. This effect appears across stochastic mini-batching, random initialisations, and distinct trained solutions, although we do not attempt to isolate the contribution of each source. The mismatch is especially pronounced in low-prevalence settings, where updates associated with the minority class are sparser and less consistently reinforced during training. As a result, standard trajectory matching may require a fixed synthetic dataset to reproduce a supervision signal that is broader than its reachable span, offering a plausible explanation for weaker performance in rare-event clinical prediction.

To mitigate this limitation, we propose Bézier Trajectory Matching (BTM), which replaces discrete SGD trajectories with task-optimised quadratic Bézier surrogates connecting the initial and final model states. The approach is inspired by mode connectivity \citep{garipov2018loss, draxler2018essentially,thakur2025optimising}, but our setting differs from the classical case because one endpoint is a random initialisation rather than a second converged solution. A quadratic Bézier curve provides a simple nonlinear path between the endpoints through a single learned control point, allowing the surrogate to capture curvature without storing a long sequence of SGD iterates. In practice, this yields smoother start-to-solution paths whose loss decreases more consistently from initialisation to solution. Compared with standard SGD trajectories, the resulting supervision signal is more structured and has lower effective rank, making it better aligned with the representational constraints of a fixed synthetic dataset while also substantially reducing trajectory storage. Figure~\ref{fig:tm_vs_sgd} provides an intuitive comparison between raw SGD supervision and the structured Bézier surrogates used in BTM. The figure illustrates the smoother parameter-space geometry and loss profile of the surrogate paths, together with their effect on trajectory-matching optimisation.

We evaluate BTM on five real-world clinical datasets spanning tabular and time-series modalities. Across these tasks, BTM matches or outperforms prior baselines, with the largest gains in low-prevalence and low-synthetic-budget settings, where identifying positive cases reliably is especially important. 

The main contributions of this paper are summarised as follows:
\begin{itemize}[topsep=0pt, itemsep=1.25pt]
    \item \textsc{Geometric characterisation of trajectory matching:} We formulate trajectory matching for a fixed synthetic dataset as a constrained subspace-approximation problem and show that its residual is lower bounded by the projection error of real-data parameter changes onto the dataset's reachable gradient span. This reveals a conditional representability bottleneck when that span is limited.

    \item \textsc{Empirical analysis of standard trajectory supervision:} We show that supervision induced by standard SGD trajectories is often spectrally broad, especially in low-prevalence settings, helping to explain the mismatch between real-data training dynamics and what compact synthetic datasets can reproduce.

    \item \textsc{Bézier Trajectory Matching:} We introduce BTM, which replaces stored SGD trajectories with task-optimised quadratic Bézier surrogates, yielding a more structured supervision signal and a substantially more compact trajectory representation.
\end{itemize}

The remainder of the paper is organised as follows. Section~\ref{sec:rel_works} reviews related work, Section~\ref{sec:background} introduces background on dataset condensation and trajectory matching, Section~\ref{sec:geometric_tm} develops the geometric analysis, Section~\ref{sec:method} introduces BTM, Section~\ref{sec:exp_setup} describes the experimental setup, and Section~\ref{sec:results} presents the empirical results.

\section{RELATED WORK}
\label{sec:rel_works}

\textsc{\textbf{Dataset Condensation:}} Dataset condensation methods differ primarily in the training signal used to optimise synthetic data. Gradient-based approaches match gradients induced by real and synthetic batches, with later extensions incorporating augmentation-aware objectives \citep{zhao2020dataset, zhao2021dataset}. Feature- and distribution-based methods instead align representations of real and synthetic data in embedding space, for example through feature alignment or distribution matching \citep{zhao2023dataset, wang2022cafe}. Kernel-based approaches provide more principled closed-form or function-space objectives, but are often less scalable to larger datasets and architectures \citep{nguyen2021dataset, nguyen2022dataset}. Other scalable or decoupled approaches prioritise efficiency, realism, or data reconstruction rather than explicitly modelling optimisation dynamics \citep{kanagavelu2024medsynth, yin2023sre2l, sun2023rded}.

\textbf{Trajectory-Based Condensation:} Trajectory matching \citep{cazenavette2022dataset} is particularly relevant to our work because it supervises synthetic data using parameter displacements observed along real training trajectories. Subsequent work has improved this paradigm through flatter or regularised trajectories \citep{du2023ftd}, stage- or difficulty-aware selection of which parts of the trajectory to match \citep{guo2024datm}, and memory-efficient scaling to larger datasets such as ImageNet \citep{cui2023scaling}. More recently, Zhong et al. \citep{zhong2025towards} proposed matching convexified trajectories to improve stability, convergence speed, and storage efficiency. Their convexified teacher trajectory provides smoother guidance, but it is not explicitly designed to remain in a low-loss region of the parameter space. As a result, in highly non-convex landscapes, smoothing by convexification may trade trajectory stability against fidelity to the original optimisation dynamics.


\textbf{Mode Connectivity:} Mode connectivity studies low-loss curves between trained models \citep{garipov2018loss, draxler2018essentially}, revealing substantial geometric structure in neural network parameter space. Related work has shown that these curves can be exploited for ensembling and wider optima \citep{izmailov2018averaging}, and that their optimisation can benefit from explicitly accounting for weight-space symmetries \citep{tatro2020optimizing}. Recent work has also used related mode-connectivity ideas in clinical machine learning settings \citep{thakur2025optimising}. Our setting differs from the classical mode-connectivity case because one endpoint is a random initialisation rather than a second converged solution. We therefore use the same geometric machinery to learn a smoother descent surrogate, rather than to claim a classical low-loss connection between two minima.

\textbf{Positioning of Our Work:} This paper sits at the intersection of trajectory-based condensation and geometric modelling of weight-space curves. Rather than proposing another heuristic modification of stored SGD trajectories, we study the representational structure of the supervision signal induced by trajectory matching. This yields a geometric account of when trajectory supervision is limited by the reachable gradient span of a fixed synthetic dataset, and motivates Bézier trajectory surrogates as a more structured alternative to raw SGD trajectories. In contrast to prior trajectory-based methods, including recent approaches based on convexified teacher trajectories \citep{zhong2025towards}, our work asks which aspects of trajectory-induced supervision are reproducible by a fixed synthetic dataset. Whereas convexified trajectories primarily smooth teacher dynamics for improved stability and storage efficiency, BTM replaces each teacher trajectory with a task-optimised quadratic Bézier surrogate motivated by this representability bottleneck, yielding a supervision signal that is more structured and better aligned with what a fixed synthetic dataset can reproduce.

\section{Background}
\label{sec:background}

\subsection{Dataset Condensation Problem}

Dataset condensation aims to represent the task-relevant training signal of a large dataset \(\mathcal{D}\) with a much smaller synthetic dataset \(\tilde{\mathcal{D}}\), such that models trained on \(\tilde{\mathcal{D}}\) achieve comparable generalisation performance.

Formally, let
\[
\mathcal{D} = \{(\mathbf{x}_i,\mathbf{y}_i)\}_{i=1}^{|\mathcal{D}|},
\qquad
\tilde{\mathcal{D}} = \{(\tilde{\mathbf{x}}_j,\tilde{\mathbf{y}}_j)\}_{j=1}^{|\tilde{\mathcal{D}}|},
\qquad
|\tilde{\mathcal{D}}| \ll |\mathcal{D}|,
\]
and let \(\boldsymbol{\theta} \in \Theta = \mathbb{R}^p\) denote model parameters. Given an initialisation \(\hat{\boldsymbol{\theta}}_0 \sim P\) and a fixed training procedure \(\hat{\boldsymbol{\theta}}_N = \mathrm{Train}(\hat{\boldsymbol{\theta}}_0,\tilde{\mathcal{D}})\), dataset condensation seeks a synthetic dataset satisfying
\begin{equation}
\tilde{\mathcal{D}}^\star
\in
\arg\min_{\tilde{\mathcal{D}}}\;
\mathbb{E}_{\hat{\boldsymbol{\theta}}_0\sim P}
\big[\mathcal{L}_{\mathrm{val}}(\hat{\boldsymbol{\theta}}_N)\big]
\quad
\text{s.t.}
\quad
\hat{\boldsymbol{\theta}}_N = \mathrm{Train}(\hat{\boldsymbol{\theta}}_0,\tilde{\mathcal{D}}),
\end{equation}
where \(\mathcal{L}_{\mathrm{val}}\) denotes a validation loss. The challenge is that \(\tilde{\mathcal{D}}\) must induce training dynamics that generalise well despite being orders of magnitude smaller than \(\mathcal{D}\).

\subsection{Trajectory Matching}
\label{subsec:tm_background}

A prominent approach to dataset condensation is trajectory matching, which uses a teacher optimisation trajectory as supervision for learning synthetic data. Let \(\{\boldsymbol{\theta}_{\tau}\}_{\tau=0}^{T} \subset \Theta\) denote the parameter sequence obtained by training a teacher model on \(\mathcal{D}\). We consider a collection of index pairs \(\{(s_k,e_k)\}_{k=1}^K\), with \(0 \le s_k < e_k \le T\), defining the endpoints of \(K\) trajectory segments. Each segment induces a teacher displacement
\[
\boldsymbol{\Delta}_k := \boldsymbol{\theta}_{e_k} - \boldsymbol{\theta}_{s_k}.
\]

To match segment \(k\), a student model is initialised at \(\hat{\boldsymbol{\theta}}_0 := \boldsymbol{\theta}_{s_k}\) and trained on \(\tilde{\mathcal{D}}\) for \(N\) inner steps. Under the first-order, no-momentum update used in standard trajectory matching implementations, the student evolves as
\begin{equation}
\label{eq:student-update}
\hat{\boldsymbol{\theta}}_{n+1}
=
\hat{\boldsymbol{\theta}}_{n}
- \eta_s\, \mathbf{g}_n,
\qquad
\mathbf{g}_n
:=
\nabla_{\boldsymbol{\theta}}
\ell(\hat{\boldsymbol{\theta}}_{n}; B_n),
\quad n=0,\dots,N-1,
\end{equation}
where \(\eta_s\) is the student learning rate, \(B_n \subset \tilde{\mathcal{D}}\) is a synthetic mini-batch, and \(\ell\) denotes the mini-batch loss. This yields the student displacement
\[
\boldsymbol{\delta}_k
:=
\hat{\boldsymbol{\theta}}_{N} - \boldsymbol{\theta}_{s_k}.
\]

The standard per-segment trajectory matching objective is
\begin{equation}
\label{eq:tm_loss}
\mathcal{L}_{\mathrm{TM}}(s_k,e_k;\tilde{\mathcal{D}})
:=
\left\|
\hat{\boldsymbol{\theta}}_{N} - \boldsymbol{\theta}_{e_k}
\right\|_2^2
=
\left\|
\boldsymbol{\delta}_k - \boldsymbol{\Delta}_k
\right\|_2^2.
\end{equation}

Trajectory matching can therefore be viewed as learning a synthetic dataset whose induced optimisation steps reproduce teacher displacements in parameter space. In the next section, we introduce our geometric analysis of this objective and show that the student can only realise displacements lying in a restricted span determined by the gradients accessible from the synthetic data.

\section{A Geometric View of Trajectory Matching}
\label{sec:geometric_tm}

This section develops a geometric analysis of trajectory matching. For a fixed synthetic dataset and inner-loop optimiser, the student cannot realise arbitrary teacher displacements; it can move only along directions generated by gradients accessible from \(\tilde{\mathcal{D}}\). This induces a constrained subspace of parameter space that determines which teacher updates are representable.

Two consequences follow. For an individual segment, the trajectory matching loss is lower bounded by the component of the teacher displacement lying outside the student’s reachable span. Across segments, this leads to a rank-based bottleneck showing that broad-spectrum teacher supervision cannot, in general, be reproduced within a low-dimensional reachable span. These results explain why raw SGD trajectories can be difficult to match with compact synthetic datasets and motivate structured surrogate supervision.

\subsection{Reachable Gradient Span}
\label{subsec:reachable_span}

To formalise the representational constraint imposed by the synthetic inner loop, we define the reachable gradient span.

\begin{definition}[Reachable gradient span]
Let \(\{\mathbf{g}_n\}_{n=0}^{N-1}\) be the gradients in \eqref{eq:student-update} for segment \(k\). The reachable gradient span for segment \(k\) is
\begin{equation}
\label{eq:grad_span}
\mathcal{G}_k
:=
\mathrm{span}\{\mathbf{g}_0,\dots,\mathbf{g}_{N-1}\}
\subset \mathbb{R}^p.
\end{equation}
\end{definition}

Under the first-order, no-momentum update in \eqref{eq:student-update}, the student displacement satisfies
\begin{equation}
\boldsymbol{\delta}_k = -\eta_s \sum_{n=0}^{N-1} \mathbf{g}_n,
\end{equation}
and therefore \(\boldsymbol{\delta}_k \in \mathcal{G}_k\). For a fixed segment, trajectory matching is therefore not an unconstrained regression problem in parameter space. The student can approximate the teacher displacement \(\boldsymbol{\Delta}_k\) only by vectors lying in the subspace \(\mathcal{G}_k\).

\subsection{Projection Lower Bound}
\label{subsec:projection_lb}

The subspace constraint above leads directly to a geometric lower bound on the trajectory matching objective.

\begin{theorem}[Projection lower bound]
\label{thm:projection-lb}
Fix a segment \(k\) with endpoint indices \((s_k,e_k)\), let \(P_{\mathcal{G}_k}\) denote the orthogonal projector onto \(\mathcal{G}_k\), and let \(I\) denote the identity operator on \(\mathbb{R}^p\). Then
\begin{equation}
\mathcal{L}_{\mathrm{TM}}(s_k,e_k;\tilde{\mathcal{D}})
\;\ge\;
\big\|
(I - P_{\mathcal{G}_k}) \boldsymbol{\Delta}_k
\big\|_2^2.
\end{equation}

\end{theorem}

\begin{proof}
Section~\ref{app:proof-projection-lb} of the appendix provides the complete proof.
\end{proof}

\begin{remark}[Best-in-span versus realised displacement]
\label{rem:best-in-span}

\revadd{By the orthogonal projection theorem~\citep{golub2013matrix}, the minimiser of
\[
\min_{\mathbf{v}\in\mathcal{G}_k}\|\mathbf{v}-\boldsymbol{\Delta}_k\|_2^2
\]
is \(P_{\mathcal{G}_k}\boldsymbol{\Delta}_k\). Thus, the bound is tight as a best-in-subspace approximation statement. However, for a fixed synthetic dataset, the realised student displacement $\boldsymbol{\delta}_k$ need not coincide with this optimum. The result should therefore be interpreted as a geometric lower bound rather than as an attainability guarantee.}
\end{remark}

Theorem~\ref{thm:projection-lb} shows that any component of the teacher displacement orthogonal to the student’s reachable span produces irreducible approximation error. Mismatch is therefore unavoidable whenever the teacher supervision places substantial mass outside the directions that the synthetic inner loop can generate.

\subsection{A Rank Bottleneck for Shared Synthetic Supervision}
\label{subsec:rank_bottleneck}

The single-segment view above extends naturally to the full collection of teacher displacements supervised by a shared synthetic dataset. Let
\begin{equation}
\label{eq:teacher_disp}
\mathbf{A}
:=
[\boldsymbol{\Delta}_1,\dots,\boldsymbol{\Delta}_K]
\in \mathbb{R}^{p\times K},
\qquad
r_{\mathbf{A}} := \operatorname{rank}(\mathbf{A}),
\end{equation}

denote the matrix of teacher displacements, and define the global reachable span induced by \(\tilde{\mathcal{D}}\) as
\begin{equation}
\label{eq:global_span}
\mathcal{G}
:=
\mathrm{span}\{\mathbf{g}^{(k)}_n : k=1,\dots,K,\; n=0,\dots,N-1\},
\qquad
r := \dim(\mathcal{G}).
\end{equation}

For a fixed synthetic dataset, trajectory matching can therefore be viewed as approximating the displacement matrix \(\mathbf{A}\) within an \(r\)-dimensional subspace \(\mathcal{G}\). This yields a global representability bottleneck.


\begin{theorem}[Rank bottleneck under an \(r\)-dimensional reachable span]
\label{thm:rank-bottleneck}
Let \(\mathbf{A} \in \mathbb{R}^{p\times K}\) be the teacher displacement matrix defined in \eqref{eq:teacher_disp}, let \(\mathcal{G} \subset \mathbb{R}^p\) be the global reachable span defined in \eqref{eq:global_span}, with \(r=\dim(\mathcal{G})\), let \(P_{\mathcal{G}}\) denote its orthogonal projector, let \(I\) denote the identity operator on \(\mathbb{R}^p\), and define \(\mathcal{L}_{\mathrm{TM}}^{(k)} := \mathcal{L}_{\mathrm{TM}}(s_k,e_k;\tilde{\mathcal{D}})\). Then
\begin{equation}
\frac{1}{K}\sum_{k=1}^K \mathcal{L}_{\mathrm{TM}}^{(k)}
\;\ge\;
\frac{1}{K}
\big\|
(I - P_{\mathcal{G}})\mathbf{A}
\big\|_F^2.
\end{equation}
Moreover, among all \(r\)-dimensional subspaces \(\mathcal{S} \subset \mathbb{R}^p\), the minimum residual is
\begin{equation}
\min_{\dim(\mathcal{S})=r}
\frac{1}{K}
\big\|
(I - P_{\mathcal{S}})\mathbf{A}
\big\|_F^2
=
\frac{1}{K}
\sum_{j=r+1}^{r_{\mathbf{A}}}
\sigma_j(\mathbf{A})^2,
\end{equation}
where \(P_{\mathcal{S}}\) denotes the orthogonal projector onto \(\mathcal{S}\), and \(\sigma_j(\mathbf{A})\) denotes the \(j\)-th singular value of \(\mathbf{A}\).
\end{theorem}

\begin{proof}
Section~\ref{app:proof-rank-bottleneck} of the appendix provides the complete proof.
\end{proof}

\begin{remark}
The displacements \(\{\boldsymbol{\Delta}_k\}_{k=1}^K\) may be drawn from one or multiple teacher trajectories~\citep{cazenavette2022dataset}. The analysis depends only on the resulting collection of displacement vectors and applies in either case.
\end{remark}

\begin{remark}[Conditional nature of the bottleneck]
Theorem~\ref{thm:rank-bottleneck} is conditional on the effective dimension \(r=\dim(\mathcal{G})\) induced by a fixed synthetic dataset, inner-loop horizon, and model class. It does not assert a universal, architecture-independent upper bound on \(r\). Instead, it quantifies the residual for a given \(r\)-dimensional reachable span. 
\end{remark}

Theorem~\ref{thm:rank-bottleneck} identifies a representational limitation of trajectory matching under constrained synthetic supervision. When substantial energy in \(\mathbf{A}\) lies outside the effective reachable span, a positive residual remains unavoidable. The difficulty of trajectory matching is therefore governed not only by the nominal rank of \(\mathbf{A}\), but more generally by its spectrum and effective rank relative to the student’s reachable span.

These results suggest that effective trajectory supervision should be concentrated in a small number of coherent directions. The next section introduces BTM, which is designed to impose exactly this kind of structured supervision.

\section{Bézier Trajectory Matching}
\label{sec:method}

The rank bottleneck above suggests that effective supervision should preserve functionally meaningful progress while avoiding diffuse high-rank variability. BTM addresses this limitation by replacing stochastic optimisation trajectories with learnable quadratic Bézier curves that act as structured low-rank surrogates. These surrogate paths are optimised to reduce average loss along the path, suppress unnecessary high-loss excursions, and empirically encourage a smoother, typically decreasing expected loss profile from initialisation to solution. The induced displacement directions are therefore confined to a low-dimensional subspace, yielding supervision that is better aligned with the constraints of student optimisation.

The construction of the surrogate trajectories is described first. Their displacement structure and the resulting supervision are then analysed.

\subsection{Bézier Surrogate Trajectories}

Surrogate trajectories are constructed between model initialisations and their corresponding trained solutions using quadratic Bézier curves. Let \(\bigl\{(\boldsymbol{\theta}^{(m)}_0,\boldsymbol{\theta}^{(m)}_T)\bigr\}_{m=1}^{M}\) denote initialisation--solution pairs obtained from real-data SGD trajectories. For each pair, define a quadratic Bézier trajectory
\begin{equation}
\label{eq:btm_bezier}
\Phi^{(m)}(t)
=
(1-t)^2 \boldsymbol{\theta}^{(m)}_0
+
2t(1-t)\boldsymbol{\phi}^{(m)}
+
t^2 \boldsymbol{\theta}^{(m)}_T,
\qquad t \in [0,1],
\end{equation}
where \(\boldsymbol{\phi}^{(m)} \in \mathbb{R}^p\) is a learnable control point. The straight line is recovered by the special choice
\[
\boldsymbol{\phi}^{(m)} = \tfrac12\bigl(\boldsymbol{\theta}^{(m)}_0 + \boldsymbol{\theta}^{(m)}_T\bigr),
\]
so the quadratic family strictly contains the linear interpolation baseline.

Each control point is optimised to reduce the average loss along the path,
\begin{equation}
\label{eq:btm_phi}
\boldsymbol{\phi}^{(m)\star}
\in
\operatorname*{arg\,min}_{\boldsymbol{\phi}}
\int_0^1
\mathcal{L}\bigl(\Phi_{\boldsymbol{\phi}}^{(m)}(t); \mathcal{D}\bigr)\,dt,
\end{equation}
where \(\mathcal{L}(\boldsymbol{\theta};\mathcal{D})\) denotes the dataset-level training loss evaluated on \(\mathcal{D}\).

In practice, the integral is approximated using discrete samples of \(t \in [0,1]\). This optimisation is related to mode connectivity~\citep{garipov2018loss}, but differs in an important respect. Because one endpoint is the random initialisation, the curve is not interpreted as a classical mode connection between two low-loss modes. Instead, minimising the average loss along the path encourages a smoother curve with lower average loss and, empirically, a typically decreasing expected loss profile from initialisation to solution.

Each surrogate trajectory is specified by three parameter vectors, \(\boldsymbol{\theta}^{(m)}_0\), \(\boldsymbol{\phi}^{(m)\star}\), and \(\boldsymbol{\theta}^{(m)}_T\). Storing \(M\) surrogates therefore requires \(\mathcal{O}(3M)\) parameter vectors. Standard trajectory matching, by contrast, stores \(T\) checkpoints per trajectory and therefore requires \(\mathcal{O}(TM)\) parameter vectors. The surrogate representation thus yields a substantial reduction in memory when \(T \gg 1\).

\subsection{Low-Rank Supervision from Bézier Surrogates}

The surrogate trajectories above induce the displacement supervision used in trajectory matching. For a surrogate trajectory \(\Phi^{(m)}\), define segment displacements
\begin{equation}
\label{eq:bez-displacement}
\boldsymbol{\Delta}^{(m)}(t_s,t_e)
=
\Phi^{(m)}(t_e)-\Phi^{(m)}(t_s),
\qquad 0 \le t_s < t_e \le 1.
\end{equation}

Collecting such displacements across sampled segments yields a surrogate displacement matrix analogous to \(\mathbf{A}\) in Section~\ref{sec:geometric_tm}. The key structural property is that all segment displacements from a single quadratic Bézier trajectory lie in a two-dimensional subspace. For notational simplicity, the surrogate index \(m\) is suppressed in the theorem statement.


\begin{theorem}[Low-dimensional displacement structure of quadratic Bézier surrogates]
\label{thm:bezier-span-method}
Let \(\Phi(t)\) be a quadratic Bézier trajectory as defined in \eqref{eq:btm_bezier}. Then, for any \(0 \le t_s < t_e \le 1\),
\begin{equation}
\label{eq:bezier-displacement-expansion}
\Phi(t_e)-\Phi(t_s)
=
(t_e-t_s)(\boldsymbol{\theta}_T-\boldsymbol{\theta}_0)
+
(t_e-t_s)(1-t_s-t_e)
\bigl(2\boldsymbol{\phi}-\boldsymbol{\theta}_0-\boldsymbol{\theta}_T\bigr).
\end{equation}
Consequently, every segment displacement lies in
\[
\mathrm{span}\Bigl\{
\boldsymbol{\theta}_T-\boldsymbol{\theta}_0,\;
2\boldsymbol{\phi}-\boldsymbol{\theta}_0-\boldsymbol{\theta}_T
\Bigr\}.
\]
Hence,
\begin{equation*}
V_{\Phi}
:=
\mathrm{span}
\bigl\{
\Phi(t_e)-\Phi(t_s)
:\; 0 \le t_s < t_e \le 1
\bigr\}
\end{equation*}
satisfies \(\dim(V_{\Phi}) \le 2\). In particular, any displacement matrix constructed from segments of a single surrogate trajectory has rank at most \(2\).
\end{theorem}

\begin{proof}
Section~\ref{app:proof-bezier-span} of the appendix provides the complete proof.
\end{proof}

For \(M\) surrogate trajectories, the combined supervision therefore has nominal rank at most \(2M\). Empirically, however, the effective rank is substantially lower and typically scales closer to \(M\) than to \(2M\) (Figure~\ref{fig:velocity-span-grid}). BTM therefore replaces diffuse high-rank supervision with a controlled low-rank alternative that is better aligned with the constraints of student optimisation. Rather than matching a broad spectrum of weakly aligned directions, the synthetic dataset needs only to reproduce a small number of coherent displacement directions that capture the dominant structure of optimisation.

\begin{figure*}[t]
\centering

\begin{subfigure}{0.325\textwidth}
    \centering
    \includegraphics[width=\linewidth]{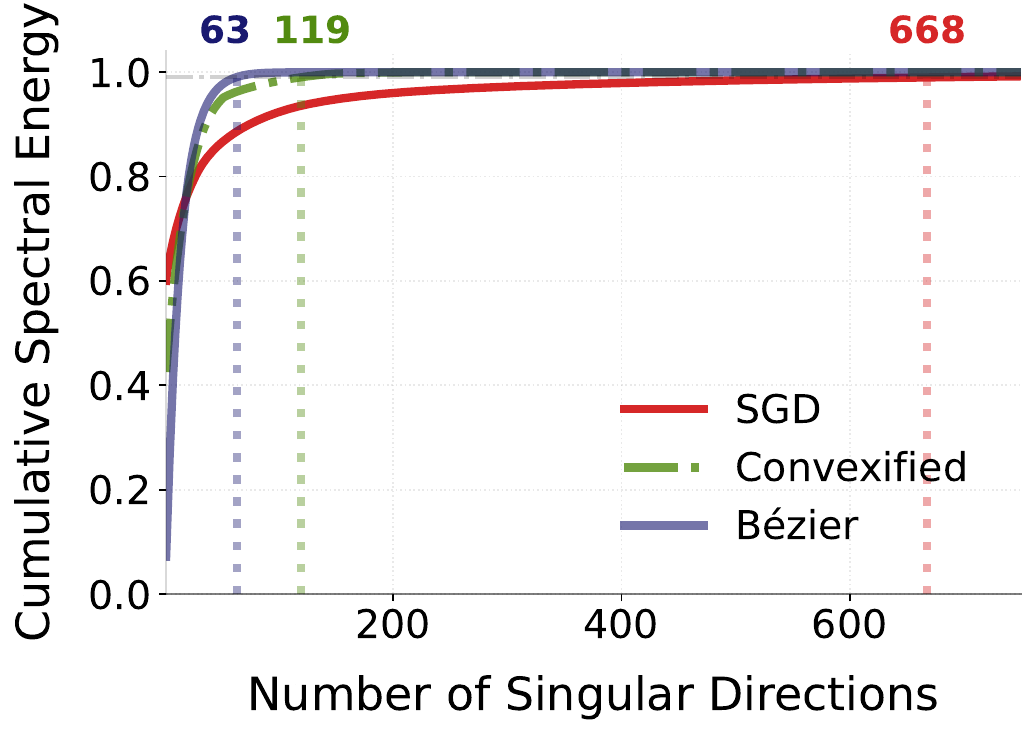}
    \caption{Portsmouth (NHS)}
\end{subfigure}
\hfill
\begin{subfigure}{0.325\textwidth}
    \centering
    \includegraphics[width=\linewidth]{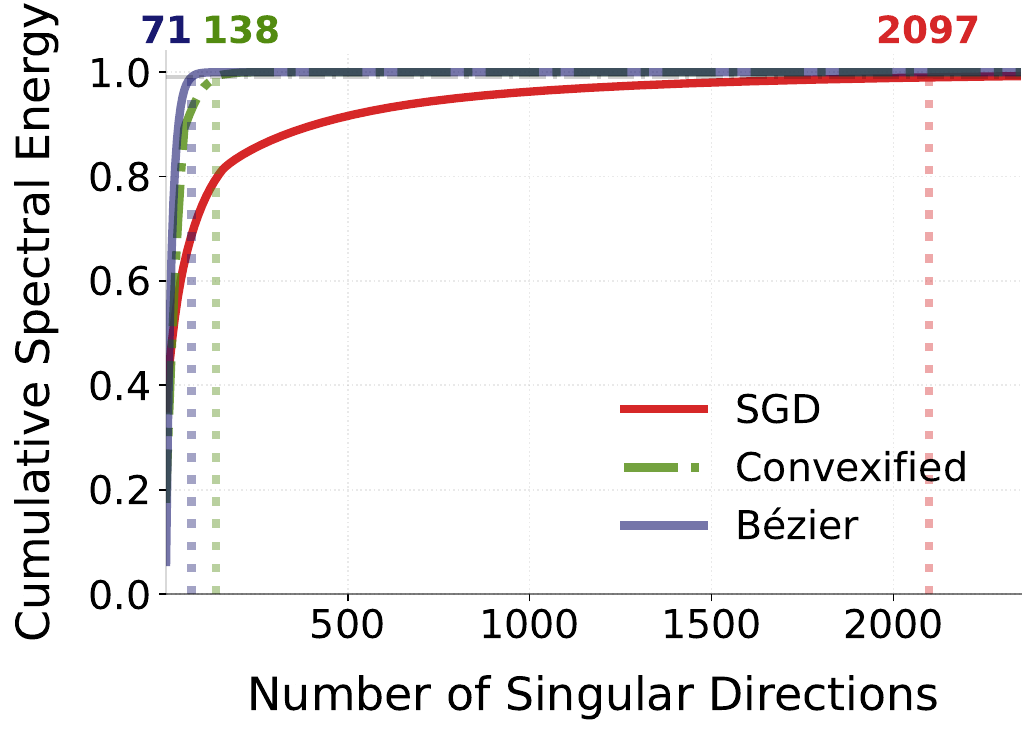}
    \caption{eICU}
\end{subfigure}
\hfill
\begin{subfigure}{0.325\textwidth}
    \centering
    \includegraphics[width=\linewidth]{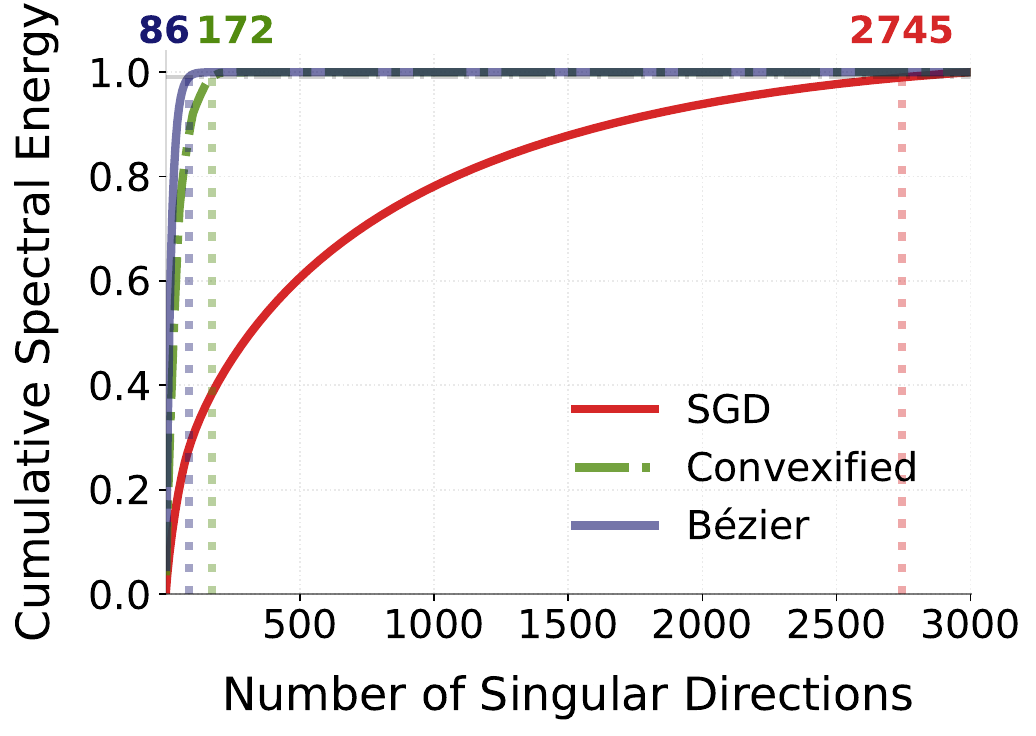}
    \caption{MIMIC-III}
\end{subfigure}

\caption{
    \textbf{Spectral concentration of teacher displacement supervision.}   
    Cumulative spectral energy of displacement matrices computed from 50 teacher trajectories per method. Raw SGD supervision spreads energy across many directions, convexified trajectories are intermediate, and Bézier surrogates concentrate most spectral mass in fewer directions.
}
\label{fig:velocity-span-grid}
\end{figure*}

\subsection{Spectral Structure of Trajectory Supervision}
\label{subsec:spectral_structure}

The geometric analysis in Section~\ref{sec:geometric_tm} shows that the difficulty of trajectory matching depends not only on the magnitude of teacher displacements, but also on how their spectral energy is distributed. By Theorem~\ref{thm:rank-bottleneck}, for any fixed effective reachable dimension \(r\), weaker spectral concentration of the displacement matrix \(\mathbf{A}\) implies larger tail energy beyond its first \(r\) singular directions and therefore a stronger conditional representability bottleneck. The low-dimensional displacement structure established above suggests that Bézier surrogates should induce supervision that is more spectrally concentrated than raw SGD trajectories.

Figure~\ref{fig:velocity-span-grid} compares the cumulative spectral energy of displacement matrices computed from 50 teacher trajectories per method across datasets. Raw SGD supervision exhibits weak spectral concentration, with energy spread across many weakly aligned directions. Convexified trajectories are more concentrated but remain relatively diffuse, whereas Bézier surrogates concentrate most of their spectral mass in far fewer directions. By Theorem~\ref{thm:rank-bottleneck}, stronger spectral concentration means that a larger fraction of the supervision signal can be captured within any fixed reachable span of dimension \(r\), leaving a smaller residual tail and therefore a weaker conditional representability bottleneck. This comparison should therefore be interpreted as a teacher-side diagnostic of the bottleneck, since it characterises the spectrum of \(\mathbf{A}\) in practice rather than the student’s reachable span itself.

This contrast in spectral concentration is especially relevant in clinical settings, where low-prevalence outcomes often give rise to sparse and inconsistently aligned minority-class updates. These can introduce low-energy but task-relevant components into the supervision signal, making faithful reproduction more difficult for a compact synthetic dataset. The stronger concentration observed for Bézier surrogates is therefore consistent with the design goal of BTM, namely to preserve dominant, functionally meaningful progress while suppressing diffuse high-rank variability that is harder for a fixed synthetic dataset to reproduce.

\subsection{Bézier Trajectory Matching Objective}

BTM replaces the teacher displacements \(\boldsymbol{\Delta}_k\) in trajectory matching with segment displacements sampled from learned Bézier surrogates, as defined in \eqref{eq:bez-displacement}. In all other respects, the training pipeline follows standard trajectory matching.

Assume access to a collection of \(M\) optimised Bézier surrogates \(\{\Phi^{(m)}\}_{m=1}^M\) constructed from real-data training, together with a synthetic dataset \(\tilde{\mathcal{D}} := \{(\tilde{\mathbf{x}}_i,\tilde{\mathbf{y}}_i)\}_{i=1}^{|\tilde{\mathcal{D}}|}\). The synthetic inputs are stacked into a matrix \(\tilde{\mathbf{X}} \in \mathbb{R}^{|\tilde{\mathcal{D}}| \times d_x}\), and the corresponding labels into a tensor \(\tilde{\mathbf{Y}} \in \mathbb{R}^{|\tilde{\mathcal{D}}| \times C}\).

At each outer iteration, a surrogate index \(m \sim \mathcal{U}\{1,\dots,M\}\) and time parameters \(0 \le t_s < t_e \le 1\) are sampled, and the corresponding segment endpoints are set to
\begin{equation}
\boldsymbol{\theta}_s := \Phi^{(m)}(t_s),
\qquad
\boldsymbol{\theta}_e := \Phi^{(m)}(t_e).
\end{equation}
A student model is initialised at \(\hat{\boldsymbol{\theta}}_0 := \boldsymbol{\theta}_s\) and trained for \(N\) steps using the update in \eqref{eq:student-update}, producing \(\hat{\boldsymbol{\theta}}_{N}\). In practice, optimisation uses the normalised trajectory matching loss
\begin{equation}
\label{eq:btm_loss_final}
\mathcal{L}_{\mathrm{BTM}}
:=
\frac{
\|\hat{\boldsymbol{\theta}}_{N} - \boldsymbol{\theta}_e\|_2^2
}{
\|\boldsymbol{\theta}_s - \boldsymbol{\theta}_e\|_2^2
}.
\end{equation}

For a fixed sampled segment, the denominator depends only on the surrogate endpoints and is therefore a positive constant with respect to the student trajectory. Consequently, the student displacement that minimises equation \ref{eq:btm_loss_final} is the same as the one that minimises the unnormalised residual in equation \ref{eq:tm_loss}, so the representability analysis in Section~\ref{sec:geometric_tm} still applies at the segment level. The effect of the normalisation is to reduce scale variation across segments, making short and long segments more comparable, while reweighting them by inverse squared segment length. It should therefore not be interpreted as a purely directional objective, but rather as a stabilised implementation of the same underlying segment-level matching problem.



The synthetic dataset is optimised by differentiating \(\mathcal{L}_{\mathrm{BTM}}\) through the unrolled student updates. Following standard first-order trajectory matching, the meta-gradient is approximated using the gradient with respect to the final student parameters,
\begin{equation}
\label{eq:btm_grad_final}
\boldsymbol{g}_L
=:
\nabla_{\hat{\boldsymbol{\theta}}_{N}}
\mathcal{L}_{\mathrm{BTM}}
=
\frac{
2(\hat{\boldsymbol{\theta}}_{N} - \boldsymbol{\theta}_e)
}{
\|\boldsymbol{\theta}_s - \boldsymbol{\theta}_e\|_2^2
}.
\end{equation}

The corresponding approximate meta-gradient with respect to the synthetic inputs is
\begin{equation}
\label{eq:btm_meta_grad}
\nabla_{\tilde{\mathbf{X}}}\mathcal{L}_{\mathrm{BTM}}
\;\approx\;
-\eta_s
\sum_{n=0}^{N-1}
\frac{1}{|B_n|}
\sum_{(\tilde{\mathbf{x}}, \tilde{\mathbf{y}}) \in B_n}
\nabla_{\tilde{\mathbf{X}}}
\Big\langle
\nabla_{\boldsymbol{\theta}}
\ell(\hat{\boldsymbol{\theta}}_{n}; \tilde{\mathbf{x}}, \tilde{\mathbf{y}}),
\boldsymbol{g}_L
\Big\rangle,
\end{equation}
where \(\eta_s\) is the student learning rate and \(B_n\) denotes the synthetic mini-batch at step \(n\).

The synthetic inputs are updated by gradient descent,
\begin{equation}
\tilde{\mathbf{X}}
\leftarrow
\tilde{\mathbf{X}}
-
\eta_{\mathbf{x}}
\nabla_{\tilde{\mathbf{X}}}\mathcal{L}_{\mathrm{BTM}},
\end{equation}
and the labels \(\tilde{\mathbf{Y}}\) can also be optimised analogously. The complete procedure is given in Algorithm~\ref{alg:condensation} in the appendix.

\subsection{Functional Control of Surrogate Supervision}
\label{subsec:functional_justification}

The low-rank analysis above concerns the geometry of the supervision signal.
A complementary question is whether replacing the raw teacher trajectory with
a structured surrogate introduces an uncontrolled change in the supervision
signal. The next result shows that this departure admits an explicit geometric
decomposition. A quadratic Bézier surrogate has constant second derivative,
and its prediction-space discrepancy from the piecewise-linear teacher
interpolation is controlled by the sum of a surrogate-dependent deviation from
the common endpoint chord and the teacher trajectory's own deviation from that
chord, under a parameter-Lipschitz assumption.

\begin{theorem}[Geometric regularity and controlled prediction-space deviation of Bézier surrogates]
\label{thm:bez-sur}

Let $\{\boldsymbol{\theta}_{\tau}\}_{\tau=0}^{T}\subset\mathbb{R}^p$
denote a teacher trajectory with initialisation $\boldsymbol{\theta}_0$
and endpoint $\boldsymbol{\theta}_T$, and let
$\gamma:[0,1]\to\mathbb{R}^p$ denote its piecewise-linear interpolation.
Define the endpoint chord
\[
c(t):=(1-t)\boldsymbol{\theta}_0+t\boldsymbol{\theta}_T,
\]
and let $\Phi(t):=\Phi_{\boldsymbol{\phi}^{\star}}(t)$ denote the
optimised quadratic Bézier surrogate. Define
\[
\kappa
:=
2\left\|
\boldsymbol{\theta}_0
-2\boldsymbol{\phi}^{\star}
+\boldsymbol{\theta}_T
\right\|_2,
\qquad
D
:=
\sup_{t\in[0,1]}
\left\|
\gamma(t)-c(t)
\right\|_2.
\]
Assume that for every $\boldsymbol{x}\in\mathcal{X}$, the model map
$f_{\boldsymbol{\theta}}(\boldsymbol{x})$ is $L_f$-Lipschitz in
$\boldsymbol{\theta}$ on the set
$\{\Phi(t),\gamma(t):t\in[0,1]\}$. Then the following hold.

\begin{enumerate}
\item[\textbf{(i)}]
\textbf{Geometric regularity.}
\begin{equation}
\sup_{t\in[0,1]}
\left\|
\Phi''(t)
\right\|_2
=
\kappa,
\qquad
\sup_{t\in[0,1]}
\left\|
\Phi(t)-c(t)
\right\|_2
=
\frac{\kappa}{8}.
\end{equation}

\item[\textbf{(ii)}]
\textbf{Controlled prediction-space deviation.}
\begin{equation}
\sup_{\boldsymbol{x}\in\mathcal{X},\,t\in[0,1]}
\left\|
f_{\Phi(t)}(\boldsymbol{x})
-
f_{\gamma(t)}(\boldsymbol{x})
\right\|_2
\le
L_f
\left(
\frac{\kappa}{8}+D
\right).
\end{equation}
\end{enumerate}
\end{theorem}

\begin{proof}
Section~\ref{app:proof-bez-sur} of the appendix provides the complete proof.
\end{proof}

The endpoint chord $c$ serves only as a common geometric reference. The term $\kappa/8$ measures the surrogate's deviation from this chord, whereas $D$ reflects the variation already baked into the raw teacher trajectory. So the bound cleanly separates what the surrogate is responsible for from what the teacher path was already doing on its own. That fits with BTM's aim of providing cleaner supervision rather than replicating the noisy SGD path point-for-point. Empirical magnitudes for these quantities are reported in Section~\ref{app:thm4-magnitudes} of the appendix.


\begin{corollary}[Conditional supervision sufficiency]
\label{cor:bez_supervision}

Assume the setting of Theorem~\ref{thm:bez-sur}, and further assume that 
the same $L_f$-Lipschitz condition holds on the set $\{\gamma_{\mathrm{stu}}(t),\Phi(t),\gamma(t):t\in[0,1]\}$.
Let $\gamma_{\mathrm{stu}}:[0,1]\to\mathbb{R}^p$ denote a student trajectory
induced by the synthetic dataset, and suppose that
\[
\sup_{t\in[0,1]}
\left\|
\gamma_{\mathrm{stu}}(t)-\Phi(t)
\right\|_2
\le
\varepsilon_{\mathrm{syn}}.
\]
Then
\begin{equation}
\sup_{\boldsymbol{x}\in\mathcal{X},\,t\in[0,1]}
\left\|
f_{\gamma_{\mathrm{stu}}(t)}(\boldsymbol{x})
-
f_{\gamma(t)}(\boldsymbol{x})
\right\|_2
\le
L_f
\left(
\varepsilon_{\mathrm{syn}}
+
\frac{\kappa}{8}
+
D
\right).
\end{equation}
\end{corollary}

\begin{proof}
Section~\ref{app:cor} of the appendix provides the complete proof.
\end{proof}

Corollary~\ref{cor:bez_supervision} connects this surrogate-level control to the condensation objective. It is a conditional sufficiency statement rather than an optimisation guarantee, meaning that if the synthetic-data-induced student tracks the surrogate within $\varepsilon_{\mathrm{syn}}$, its teacher-relative prediction discrepancy is controlled by the student tracking error, the surrogate-dependent term, and the teacher-path term.

Together, Theorem~\ref{thm:bez-sur} and Corollary~\ref{cor:bez_supervision} complement the low-rank representability analysis by showing that surrogate simplification is controlled rather than arbitrary. Along with the empirical ablations in Section~\ref{subsec:path_ablation}, this supports Bézier curves as a low-complexity surrogate family for trajectory supervision.

\section{Experimental Setup}
\label{sec:exp_setup}

{\bfseries\scshape Datasets.} Evaluation is conducted on five real-world clinical datasets spanning multiple institutions, modalities, and prediction settings. These include three large de-identified emergency department cohorts collected from UK NHS Trusts in Oxford, Portsmouth, and Birmingham \citep{soltan2024scalable}, together with two widely used public ICU benchmarks, eICU \citep{pollard2018eicu} and MIMIC-III \citep{johnson2016mimic}. As summarised in Table~\ref{tab:dataset_summary}, the resulting benchmark suite covers both tabular and multivariate time-series data, and includes binary and multi-label classification tasks. This combination enables evaluation across both institution-specific operational datasets and established public critical-care benchmarks, providing a diverse and clinically realistic test bed for dataset condensation. Additional dataset details are provided in Section~\ref{sec:dataset_details} of the appendix.

Clinical datasets differ from standard vision benchmarks in ways that are directly relevant to trajectory matching. Outcomes are often highly imbalanced, and patient populations are heterogeneous and noisy, leading to optimisation dynamics with sparse and inconsistently aligned gradient signals. As discussed in Section~\ref{subsec:spectral_structure}, this tends to produce diffuse supervision spread across many weakly reinforced directions, which is difficult to capture within the limited reachable span of a compact synthetic dataset. Clinical settings therefore provide a particularly stringent and realistic test bed for trajectory-based condensation.

\begin{table}[t]
\centering
\caption{Summary of datasets used in this work.}
\label{tab:dataset_summary}
\small
\setlength{\tabcolsep}{5pt}
\renewcommand{\arraystretch}{1.12}
\setlength{\lightrulewidth}{0.35pt}
\setlength{\cmidrulewidth}{0.35pt}
\begin{tabular}{llccl}
\toprule
\textsc{Source} & \textsc{Dataset} & \textsc{Modality} & \textsc{\# Samples} & \textsc{Task} \\
\arrayrulecolor{gray!60}
\midrule
\arrayrulecolor{black}
\multirow{3}{*}{\makecell[l]{\textsc{UK NHS}\\ \textsc{ED Cohorts}}}
& \textsc{Oxford}      & Tabular     & 161,955 & \multirow{3}{*}{\makecell[l]{COVID-19 diagnosis}} \\
& \textsc{Portsmouth} & Tabular     & 38,717  & \\
& \textsc{Birmingham} & Tabular     & 95,236  & \\
\addlinespace[1pt]
\arrayrulecolor{gray!60}
\cmidrule(lr){1-5}
\arrayrulecolor{black}
\addlinespace[1pt]
\multirow{2}{*}{\makecell[l]{\textsc{Public ICU}\\ \textsc{Benchmarks}}}
& \textsc{eICU}        & Tabular     & 49,305  & \makecell[l]{In-hospital mortality} \\
\arrayrulecolor{gray!60}
\cmidrule(lr){2-5}
\arrayrulecolor{black}
& \multirow{2}{*}{\textsc{MIMIC-III}} & \multirow{2}{*}{Time-series} & 21,156 & In-hospital mortality \\
&                                     &                              & 21,728 & Phenotyping (25 classes, multi-label) \\
\bottomrule
\end{tabular}
\end{table}

{\bfseries\scshape Baselines.} The comparison primarily focuses on trajectory-matching-based dataset condensation methods, as these provide the most direct point of reference for BTM. The selected methods differ primarily in how the trajectory supervision signal is constructed or modified.

The baselines include Matching Training Trajectories (MTT) \citep{cazenavette2022dataset}, which directly matches multi-step parameter updates between real and synthetic training; Flat Trajectory Distillation (FTD) \citep{du2023ftd}, which promotes flatter trajectories to mitigate accumulated matching error; Difficulty-Aligned Trajectory Matching (DATM) \citep{guo2024datm}, which samples trajectory segments from different training stages to vary the difficulty of the supervision signal; and Matching Convexified Trajectories (MCT) \citep{zhong2025towards}, which replaces SGD trajectories with convexified surrogates to improve stability and efficiency. TrajEctory matching with Soft Label Assignment (TESLA) \citep{cui2023scaling} is not included, as it primarily addresses scalability by reducing memory requirements with respect to the number of unrolled optimisation steps rather than modifying the structure of the supervision signal itself. Random subset selection is included as a lower-bound baseline, and full-dataset training as an upper bound.

For broader comparison beyond trajectory-matching methods, we additionally includee Dataset Condensation by Minimizing Maximum Mean Discrepancy (M3D) \citep{zhang2024m3d}, as a representative distribution-matching method and Dataset Condensation with Gradient Matching (DC) \citep{zhao2020dataset}, as a representative gradient-matching method.

{\bfseries\scshape Experimental Protocol.}
Synthetic datasets are constructed at 50, 100, 200, and 500 instances per class (\emph{ipc}). For multi-label phenotyping, we interpret \emph{ipc} as a per-label budget (\emph{plb}) and map it to the total synthetic dataset size as $|\tilde{\mathcal{D}}| = \text{\emph{plb}} \times \frac{L}{\bar{c}}$, where $L$ is the number of labels and $\bar{c}$ is the average label cardinality (i.e., the mean number of positive labels per sample). This ensures that the effective supervision budget is comparable to the single-label setting while accounting for label overlap. All methods are adapted to clinical data modalities following their original implementations, with fixed hard labels used throughout. Across all datasets, 65\%, 15\%, and 20\% of the samples are used for training, validation, and testing, respectively. Following standard evaluation practice, randomly initialised models are trained from scratch on the condensed data and evaluated on the held-out test sets.  Performance is measured using the area under the receiver operating characteristic curve (AUROC) and the area under the precision--recall curve (AUPRC). For the multi-label phenotyping, we report \emph{macro}-averaged AUROC and AUPRC across labels. All results are reported over $10$ random initialisations.

Teacher trajectories are obtained by training models on the real training data with standard optimisation, with hyperparameters such as the learning rate, optimiser, and training horizon selected on the validation set. The main BTM-specific hyperparameters, including the number of student optimisation steps, the segment length \((t_e - t_s)\), and the optimisation settings for the Bézier control point, are selected in the same way. Full implementation details \revadd{for all methods} are provided in Section~\ref{sec:implementation} of the appendix.

Evaluation architectures are fixed across all methods. A shallow multi-layer perceptron (MLP) \citep{rumelhart1986learning} is used for tabular data, and a temporal convolutional network (TCN) \citep{bai2018empirical} for time-series data. Consistent with prior work, the same architecture is used for both condensation and evaluation, except in cross-architecture experiments. Synthetic inputs are initialised from real samples.

\section{Results and Discussion}
\label{sec:results}

\subsection{Downstream Utility of Condensed Data}

The primary evaluation criterion is downstream task performance after training on condensed data. Particular emphasis is placed on low-data and low-prevalence regimes, where diffuse supervision and weakly aligned gradients make trajectory-based condensation especially challenging.

\begin{table*}[!t]
\centering
\caption{
Test performance on the three NHS emergency department cohorts for COVID-19 prediction across synthetic budgets. Results are reported as mean $\pm$ standard deviation over 10 random network initialisations. Best results at each \emph{ipc} are highlighted in \textcolor{blue}{blue} for BTM and \textcolor{red}{red} for the strongest baseline.
}
\label{tab:curial_results}

\textbf{(a) Portsmouth dataset} (5.3\% prevalence)\\[0.5ex]
\begin{sc}
\resizebox{0.85\linewidth}{!}{
    \begin{tabular}{lccccBcccc}
    \toprule
    & \multicolumn{4}{cB}{\textbf{AUROC}} & \multicolumn{4}{c}{\textbf{AUPRC}} \\
    \cmidrule(lr{0.5em}){2-5} \cmidrule(lr{0.5em}){6-9}
    \textbf{Method} & \textbf{50} & \textbf{100} & \textbf{200} & \textbf{500} & \textbf{50} & \textbf{100} & \textbf{200} & \textbf{500} \\
    \midrule
    Random     & $0.835_{\pm 0.006}$ & $0.863_{\pm 0.008}$ & $0.880_{\pm 0.003}$ & $0.877_{\pm 0.008}$
               & $0.289_{\pm 0.018}$ & $0.314_{\pm 0.025}$ & $0.388_{\pm 0.016}$ & $0.436_{\pm 0.026}$ \\
    M3D        & $0.861_{\pm 0.004}$ & $0.878_{\pm 0.003}$ & $0.879_{\pm 0.002}$ & $0.891_{\pm 0.001}$
               & $0.348_{\pm 0.013}$ & $0.405_{\pm 0.011}$ & $0.416_{\pm 0.011}$ & $0.527_{\pm 0.011}$ \\
    DC         & $0.871_{\pm 0.002}$ & $0.879_{\pm 0.002}$ & $0.884_{\pm 0.002}$ & $0.893_{\pm 0.002}$
               & $0.474_{\pm 0.009}$ & $0.480_{\pm 0.008}$ & $0.505_{\pm 0.011}$ & $0.535_{\pm 0.009}$ \\
    MTT        & $0.853_{\pm 0.002}$ & $0.868_{\pm 0.001}$ & $0.874_{\pm 0.001}$ & $0.900_{\pm 0.001}$
               & $0.518_{\pm 0.001}$ & $0.542_{\pm 0.004}$ & $0.530_{\pm 0.001}$ & $0.526_{\pm 0.001}$ \\
    FTD        & $0.835_{\pm 0.002}$ & $0.871_{\pm 0.001}$ & $0.883_{\pm 0.001}$ & $0.895_{\pm 0.001}$
               & $0.498_{\pm 0.002}$ & $0.532_{\pm 0.002}$ & $0.549_{\pm 0.001}$ & $0.590_{\pm 0.001}$ \\
    MCT        & $0.867_{\pm 0.001}$ & $0.879_{\pm 0.001}$ & $0.885_{\pm 0.001}$ & $0.893_{\pm 0.001}$
               & $0.476_{\pm 0.001}$ & $0.493_{\pm 0.003}$ & $0.508_{\pm 0.001}$ & $0.541_{\pm 0.005}$ \\
    DATM       & $0.872_{\pm 0.001}$ & \bestOther{$0.884_{\pm 0.001}$} & $0.879_{\pm 0.001}$ & $0.890_{\pm 0.003}$
               & \bestOther{$0.561_{\pm 0.001}$} & $0.560_{\pm 0.001}$ & $0.563_{\pm 0.002}$ & $0.580_{\pm 0.001}$ \\
    \rowcolor{btmgray}
    BTM (Ours) & \bestOurs{$0.876_{\pm 0.002}$} & $0.880_{\pm 0.002}$ & \bestOurs{$0.907_{\pm 0.001}$} & \bestOurs{$0.901_{\pm 0.001}$}
               & $0.557_{\pm 0.001}$ & \bestOurs{$0.572_{\pm 0.001}$} & \bestOurs{$0.596_{\pm 0.001}$} & \bestOurs{$0.609_{\pm 0.001}$} \\
    \midrule
    \emph{\textbf{Full Dataset}} & \multicolumn{4}{cB}{$\mathbf{0.906_{\pm 0.002}}$} & \multicolumn{4}{c}{$\mathbf{0.610_{\pm 0.004}}$} \\
    \bottomrule
    \end{tabular}
}
\end{sc}
\par\vspace{3ex}

\textbf{(b) Oxford dataset} (1.7\% prevalence)\\[0.5ex]
\begin{sc}
\resizebox{0.85\linewidth}{!}{
    \begin{tabular}{lccccBcccc}
    \toprule
    & \multicolumn{4}{cB}{\textbf{AUROC}} & \multicolumn{4}{c}{\textbf{AUPRC}} \\
    \cmidrule(lr{0.5em}){2-5} \cmidrule(lr{0.5em}){6-9}
    \textbf{Method} & \textbf{50} & \textbf{100} & \textbf{200} & \textbf{500} & \textbf{50} & \textbf{100} & \textbf{200} & \textbf{500} \\
    \midrule
    Random     & $0.831_{\pm 0.003}$ & $0.856_{\pm 0.004}$ & $0.870_{\pm 0.005}$ & $0.879_{\pm 0.007}$
               & $0.149_{\pm 0.002}$ & $0.171_{\pm 0.006}$ & $0.230_{\pm 0.008}$ & $0.256_{\pm 0.007}$ \\

    M3D        & $0.840_{\pm 0.004}$ & $0.857_{\pm 0.003}$ & $0.867_{\pm 0.003}$ & $0.883_{\pm 0.002}$ 
               & $0.162_{\pm 0.015}$ & $0.190_{\pm 0.010}$ & $0.266_{\pm 0.010}$ & $0.290_{\pm 0.012}$ \\
    
    DC         & $0.854_{\pm 0.003}$ & $0.858_{\pm 0.003}$ & $0.873_{\pm 0.003}$ & $0.887_{\pm 0.003}$
               & $0.261_{\pm 0.021}$ & $0.297_{\pm 0.015}$ & $0.326_{\pm 0.012}$ & $0.317_{\pm 0.013}$ \\
               
    MTT        & $0.834_{\pm 0.002}$ & $0.861_{\pm 0.004}$ & $0.855_{\pm 0.008}$ & $0.869_{\pm 0.001}$
               & $0.370_{\pm 0.008}$ & $0.391_{\pm 0.006}$ & $0.405_{\pm 0.012}$ & $0.410_{\pm 0.007}$ \\
    FTD        & $0.851_{\pm 0.008}$ & $0.850_{\pm 0.006}$ & $0.852_{\pm 0.005}$ & $0.874_{\pm 0.006}$
               & $0.368_{\pm 0.001}$ & $0.385_{\pm 0.009}$ & $0.400_{\pm 0.009}$ & $0.410_{\pm 0.005}$ \\
    MCT        & $0.845_{\pm 0.001}$ & $0.862_{\pm 0.003}$ & $0.882_{\pm 0.001}$ & $0.874_{\pm 0.003}$
               & $0.341_{\pm 0.014}$ & $0.369_{\pm 0.016}$ & $0.358_{\pm 0.004}$ & $0.352_{\pm 0.002}$ \\
    DATM       & $0.856_{\pm 0.005}$ & \bestOther{$0.868_{\pm 0.004}$} & $0.875_{\pm 0.003}$ & $0.880_{\pm 0.004}$
               & $0.359_{\pm 0.015}$ & $0.413_{\pm 0.005}$ & $0.419_{\pm 0.002}$ & $0.415_{\pm 0.003}$ \\
    \rowcolor{btmgray}
    BTM (Ours) & \bestOurs{$0.865_{\pm 0.004}$} & $0.867_{\pm 0.009}$ & \bestOurs{$0.886_{\pm 0.005}$} & \bestOurs{$0.891_{\pm 0.001}$}
               & \bestOurs{$0.414_{\pm 0.001}$} & \bestOurs{$0.425_{\pm 0.002}$} & \bestOurs{$0.435_{\pm 0.001}$} & \bestOurs{$0.440_{\pm 0.001}$} \\
    \midrule
    \emph{\textbf{Full Dataset}} & \multicolumn{4}{cB}{$\mathbf{0.901_{\pm 0.001}}$} & \multicolumn{4}{c}{$\mathbf{0.445_{\pm 0.004}}$} \\
    \bottomrule
    \end{tabular}
}
\end{sc}
\par\vspace{3ex}

\textbf{(c) Birmingham dataset} (0.8\% prevalence)\\[0.5ex]
\begin{sc}
\resizebox{0.85\linewidth}{!}{
    \begin{tabular}{lccccBcccc}
    \toprule
    & \multicolumn{4}{cB}{\textbf{AUROC}} & \multicolumn{4}{c}{\textbf{AUPRC}} \\
    \cmidrule(lr{0.5em}){2-5} \cmidrule(lr{0.5em}){6-9}
    \textbf{Method} & \textbf{50} & \textbf{100} & \textbf{200} & \textbf{500} & \textbf{50} & \textbf{100} & \textbf{200} & \textbf{500} \\
    \midrule
    Random     & $0.842_{\pm 0.011}$ & $0.857_{\pm 0.014}$ & $0.865_{\pm 0.009}$ & $0.890_{\pm 0.012}$
               & $0.080_{\pm 0.018}$ & $0.073_{\pm 0.009}$ & $0.119_{\pm 0.016}$ & $0.151_{\pm 0.009}$ \\
               
    M3D         & $0.857_{\pm 0.005}$ & $0.863_{\pm 0.003}$ & $0.872_{\pm 0.004}$ & $0.890_{\pm 0.002}$
                & $0.107_{\pm 0.005}$ & $0.070_{\pm 0.006}$ & $0.107_{\pm 0.008}$ & $0.153_{\pm 0.006}$ \\

    DC         & $0.858_{\pm 0.003}$ & $0.867_{\pm 0.003}$ & $0.878_{\pm 0.002}$ & $0.884_{\pm 0.005}$
               & $0.157_{\pm 0.013}$ & $0.159_{\pm 0.006}$ & $0.162_{\pm 0.010}$ & $0.173_{\pm 0.014}$ \\
    MTT        & $0.828_{\pm 0.006}$ & $0.847_{\pm 0.013}$ & $0.851_{\pm 0.013}$ & $0.884_{\pm 0.007}$
               & $0.189_{\pm 0.027}$ & $0.192_{\pm 0.021}$ & $0.212_{\pm 0.016}$ & $0.234_{\pm 0.017}$ \\
    FTD        & $0.829_{\pm 0.005}$ & $0.839_{\pm 0.010}$ & $0.847_{\pm 0.010}$ & \bestOther{$0.891_{\pm 0.004}$}
               & $0.193_{\pm 0.008}$ & $0.218_{\pm 0.013}$ & $0.216_{\pm 0.018}$ & $0.229_{\pm 0.014}$ \\
    MCT        & $0.860_{\pm 0.001}$ & $0.862_{\pm 0.007}$ & $0.877_{\pm 0.012}$ & $0.886_{\pm 0.001}$
               & $0.233_{\pm 0.002}$ & $0.220_{\pm 0.023}$ & $0.234_{\pm 0.003}$ & $0.247_{\pm 0.006}$ \\
    DATM       & $0.824_{\pm 0.001}$ & $0.832_{\pm 0.022}$ & $0.851_{\pm 0.016}$ & $0.870_{\pm 0.005}$
               & $0.178_{\pm 0.004}$ & $0.229_{\pm 0.005}$ & $0.235_{\pm 0.002}$ & $0.240_{\pm 0.006}$ \\
    \rowcolor{btmgray}
    BTM (Ours) & \bestOurs{$0.863_{\pm 0.004}$} & \bestOurs{$0.872_{\pm 0.004}$} & \bestOurs{$0.883_{\pm 0.004}$} & $0.890_{\pm 0.003}$
               & \bestOurs{$0.269_{\pm 0.007}$} & \bestOurs{$0.270_{\pm 0.002}$} & \bestOurs{$0.274_{\pm 0.003}$} & \bestOurs{$0.284_{\pm 0.002}$} \\
    \midrule
    \emph{\textbf{Full Dataset}} & \multicolumn{4}{cB}{$\mathbf{0.898_{\pm 0.002}}$} & \multicolumn{4}{c}{$\mathbf{0.293_{\pm 0.007}}$} \\
    \bottomrule
    \end{tabular}
}
\end{sc}

\end{table*}

\subsubsection{NHS emergency department cohorts.}

Table~\ref{tab:curial_results} reports results on the three NHS cohorts across synthetic budgets ranging from 50 to 500 \revadd{\emph{ipc}} for the task of COVID-19 prediction. BTM performs strongly across all settings and is particularly effective in the clinically relevant low-prevalence regime. The gains are most consistent in AUPRC, suggesting that the surrogate supervision is especially beneficial when positive-class signal is rare and difficult to preserve under severe compression.

On the Portsmouth dataset (5.3\% prevalence), BTM achieves the best AUROC at \emph{ipc}=50, 200, and 500, and the best AUPRC from \emph{ipc}=100 onward. In AUPRC, this corresponds to improvements of 2.1\%, 5.9\%, and 3.2\% over the strongest baseline at \emph{ipc}=100, 200, and 500, respectively. At \emph{ipc}=500, BTM reaches \(0.609\), which is 99.8\% of full-dataset performance (\(0.610\)). At the smallest budget, DATM is marginally stronger in AUPRC, but BTM overtakes it as the synthetic budget increases.

On the Oxford dataset (1.7\% prevalence), BTM again achieves the best AUROC at \emph{ipc}=50, 200, and 500, and attains the strongest AUPRC at every budget. The relative AUPRC gains over the strongest baseline are 11.9\%, 2.9\%, 3.8\%, and 6.0\% at \emph{ipc}=50, 100, 200, and 500, respectively. Performance improves steadily from \(0.414\) at \emph{ipc}=50 to \(0.440\) at \emph{ipc}=500, corresponding to 98.9\% of the full-dataset AUPRC (\(0.445\)).

The clearest gains appear on the Birmingham dataset (0.8\% prevalence), the most imbalanced cohort.
Here, BTM delivers the best AUPRC at every budget and the best AUROC up to \emph{ipc}=200, while remaining competitive at \emph{ipc}=500. The AUPRC gains over the strongest baseline are 15.5\%, 17.9\%, 16.6\%, and 15.0\% across the four budgets, highlighting the benefit of BTM in the 
\revadd{lowest-prevalence setting.}
At \emph{ipc}=500, BTM reaches \(0.284\), or 96.9\% of full-dataset AUPRC (\(0.293\)).

This pattern is consistent with the geometric analysis in Section~\ref{sec:geometric_tm}. As prevalence decreases, the supervision signal becomes increasingly diffuse and weakly aligned, making trajectory matching more sensitive to high-rank variability. BTM mitigates this effect by replacing raw SGD supervision with smoother, low-rank surrogate trajectories that retain task-relevant structure while suppressing stepwise noise.

Among the trajectory-matching baselines, MCT is generally the most competitive in AUROC, suggesting that reducing supervision rank is already beneficial. Its gains in AUPRC, however, are less consistent, particularly on the lower-prevalence cohorts. This aligns with the methodological difference between the two approaches. MCT constructs convexified trajectory surrogates from stored checkpoints, whereas BTM learns Bézier control points that reduce average path loss and encourage a smoother descent profile. DATM is also competitive on the Portsmouth and Oxford datasets, especially at smaller budgets, but its advantage weakens as prevalence decreases.
The non-trajectory-based baselines, M3D and DC, are also competitive in AUROC in several settings, but their AUPRC performance is less consistent, particularly on the lower-prevalence NHS cohorts. This pattern underscores the importance of evaluating precision–recall performance alongside AUROC when the positive class is rare.
Overall, these results indicate that learned low-rank surrogate supervision is particularly useful low-data, low-prevalence regimes most relevant to clinical practice.

\subsubsection{Public ICU benchmarks.}
Table~\ref{tab:icu_mort} reports results on the eICU and MIMIC-III datasets across synthetic budgets ranging from 50 to 500 \revadd{\emph{ipc}} for the task of in-hospital mortality prediction. As with the NHS cohorts, BTM performs strongly across compression levels, with the clearest and most consistent gains appearing in AUPRC. This again suggests that structured low-rank surrogate supervision is particularly effective when the clinically relevant signal is sparse and must be retained under aggressive data compression.

\begin{table*}[t]
\centering
\caption{
Test performance on the eICU and MIMIC-III in-hospital mortality tasks across synthetic budgets. Results are reported as mean $\pm$ standard deviation over 10 random network initialisations. Best results at each \emph{ipc} are highlighted in \textcolor{blue}{blue} for BTM and \textcolor{red}{red} for the strongest baseline.
}
\label{tab:icu_mort}

\textbf{(a) eICU dataset}\\[0.5ex]
\begin{sc}
\resizebox{0.9\linewidth}{!}{
    \begin{tabular}{lccccBcccc}
    \toprule
    & \multicolumn{4}{cB}{\textbf{AUROC}} & \multicolumn{4}{c}{\textbf{AUPRC}} \\
    \cmidrule(lr{0.5em}){2-5} \cmidrule(lr{0.5em}){6-9}
    \textbf{Method} & \textbf{50} & \textbf{100} & \textbf{200} & \textbf{500} & \textbf{50} & \textbf{100} & \textbf{200} & \textbf{500} \\
    \midrule
    Random     & $0.736_{\pm 0.018}$ & $0.774_{\pm 0.016}$ & $0.802_{\pm 0.005}$ & $0.833_{\pm 0.011}$
               & $0.281_{\pm 0.014}$ & $0.328_{\pm 0.021}$ & $0.360_{\pm 0.014}$ & $0.391_{\pm 0.023}$ \\
    M3D        & $0.779_{\pm 0.002}$ & $0.824_{\pm 0.002}$ & $0.846_{\pm 0.001}$ & $0.852_{\pm 0.001}$
               & $0.334_{\pm 0.003}$ & $0.400_{\pm 0.003}$ & $0.447_{\pm 0.002}$ & $0.435_{\pm 0.005}$ \\
    DC
    & $0.785_{\pm 0.001}$ & $0.780_{\pm 0.002}$ & $0.823_{\pm 0.001}$ & $0.851_{\pm 0.001}$
    & $0.351_{\pm 0.002}$ & $0.347_{\pm 0.002}$ & $0.402_{\pm 0.003}$ & $0.455_{\pm 0.002}$ \\
    MTT        & $0.754_{\pm 0.003}$ & $0.800_{\pm 0.001}$ & $0.829_{\pm 0.004}$ & $0.849_{\pm 0.001}$
               & $0.329_{\pm 0.003}$ & $0.390_{\pm 0.004}$ & $0.398_{\pm 0.007}$ & $0.454_{\pm 0.002}$ \\
    FTD        & $0.798_{\pm 0.002}$ & $0.803_{\pm 0.002}$ & $0.824_{\pm 0.001}$ & $0.847_{\pm 0.004}$
               & $0.374_{\pm 0.002}$ & $0.389_{\pm 0.001}$ & $0.430_{\pm 0.002}$ & $0.442_{\pm 0.004}$ \\
    MCT        & $0.773_{\pm 0.004}$ & $0.828_{\pm 0.001}$ & $0.837_{\pm 0.002}$ & $0.847_{\pm 0.001}$
               & $0.397_{\pm 0.005}$ & $0.427_{\pm 0.001}$ & $0.448_{\pm 0.001}$ & $0.464_{\pm 0.001}$ \\
    DATM       & \bestOther{$0.854_{\pm 0.001}$} & $0.852_{\pm 0.001}$ & $0.846_{\pm 0.001}$ & $0.856_{\pm 0.001}$
               & $0.462_{\pm 0.001}$ & $0.451_{\pm 0.003}$ & $0.464_{\pm 0.001}$ & $0.482_{\pm 0.001}$ \\
    \rowcolor{btmgray}
    BTM (Ours) & \bestOurs{$0.854_{\pm 0.002}$} & \bestOurs{$0.859_{\pm 0.001}$} & \bestOurs{$0.861_{\pm 0.001}$} & \bestOurs{$0.874_{\pm 0.002}$}
               & \bestOurs{$0.479_{\pm 0.002}$} & \bestOurs{$0.486_{\pm 0.001}$} & \bestOurs{$0.476_{\pm 0.001}$} & \bestOurs{$0.506_{\pm 0.001}$} \\
    \midrule
    \emph{\textbf{Full Dataset}} & \multicolumn{4}{cB}{$\mathbf{0.879_{\pm 0.002}}$} & \multicolumn{4}{c}{$\mathbf{0.515_{\pm 0.003}}$} \\
    \bottomrule
    \end{tabular}
}
\end{sc}
\par\vspace{3ex}

\textbf{(b) MIMIC-III dataset}\\[0.5ex]
\begin{sc}
\resizebox{0.9\linewidth}{!}{
    \begin{tabular}{lccccBcccc}
    \toprule
    & \multicolumn{4}{cB}{\textbf{AUROC}} & \multicolumn{4}{c}{\textbf{AUPRC}} \\
    \cmidrule(lr{0.5em}){2-5} \cmidrule(lr{0.5em}){6-9}
    \textbf{Method} & \textbf{50} & \textbf{100} & \textbf{200} & \textbf{500} & \textbf{50} & \textbf{100} & \textbf{200} & \textbf{500} \\
    \midrule
    Random     & $0.740_{\pm 0.022}$ & $0.763_{\pm 0.013}$ & $0.804_{\pm 0.008}$ & $0.840_{\pm 0.005}$
               & $0.307_{\pm 0.022}$ & $0.338_{\pm 0.022}$ & $0.368_{\pm 0.017}$ & $0.409_{\pm 0.013}$ \\
    M3D        & $0.771_{\pm 0.006}$ & $0.774_{\pm 0.004}$ & $0.788_{\pm 0.010}$ & $0.799_{\pm 0.008}$ 
               & $0.390_{\pm 0.007}$ & $0.401_{\pm 0.006}$ & $0.413_{\pm 0.007}$ & $0.403_{\pm 0.017}$ \\
    DC      & $0.789_{\pm 0.007}$ & $0.791_{\pm 0.006}$ & $0.801_{\pm 0.008}$ & $0.807_{\pm 0.005}$ 
               & $0.386_{\pm 0.007}$ & $0.398_{\pm 0.009}$ & $0.414_{\pm 0.016}$ & $0.419_{\pm 0.008}$ \\
    MTT        & $0.800_{\pm 0.008}$ & $0.821_{\pm 0.003}$ & $0.834_{\pm 0.002}$ & $0.838_{\pm 0.002}$
               & $0.421_{\pm 0.021}$ & $0.449_{\pm 0.013}$ & $0.493_{\pm 0.006}$ & $0.492_{\pm 0.007}$ \\
    FTD        & $0.770_{\pm 0.003}$ & $0.779_{\pm 0.005}$ & $0.811_{\pm 0.005}$ & $0.828_{\pm 0.003}$
               & $0.361_{\pm 0.009}$ & $0.384_{\pm 0.012}$ & $0.443_{\pm 0.011}$ & $0.473_{\pm 0.009}$ \\
    MCT        & $0.821_{\pm 0.007}$ & $0.831_{\pm 0.002}$ & $0.835_{\pm 0.002}$ & $0.840_{\pm 0.002}$
               & $0.453_{\pm 0.013}$ & $0.481_{\pm 0.009}$ & $0.492_{\pm 0.007}$ & $0.493_{\pm 0.005}$ \\
    DATM       & $0.824_{\pm 0.003}$ & $0.834_{\pm 0.002}$ & $0.838_{\pm 0.002}$ & \bestOther{$0.842_{\pm 0.002}$}
               & $0.461_{\pm 0.009}$ & $0.484_{\pm 0.008}$ & \bestOther{$0.500_{\pm 0.007}$} & $0.496_{\pm 0.006}$ \\
    \rowcolor{btmgray}
    BTM (Ours) & \bestOurs{$0.828_{\pm 0.005}$} & \bestOurs{$0.835_{\pm 0.003}$} & \bestOurs{$0.839_{\pm 0.002}$} & $0.840_{\pm 0.002}$
               & \bestOurs{$0.471_{\pm 0.008}$} & \bestOurs{$0.488_{\pm 0.007}$} & $0.499_{\pm 0.006}$ & \bestOurs{$0.503_{\pm 0.003}$} \\
    \midrule
    \emph{\textbf{Full Dataset}} & \multicolumn{4}{cB}{$\mathbf{0.837_{\pm 0.003}}$} & \multicolumn{4}{c}{$\mathbf{0.499_{\pm 0.008}}$} \\
    \bottomrule
    \end{tabular}
}
\end{sc}
\end{table*}

On the eICU dataset, BTM attains the best or tied-best AUROC at every budget and the strongest AUPRC throughout. In AUPRC, the relative gains over the strongest baseline are 3.7\%, 7.8\%, 2.6\%, and 5.0\% at \emph{ipc}=50, 100, 200, and 500, respectively. At the largest budget, BTM reaches \(0.506\), corresponding to 98.3\% of full-dataset AUPRC (\(0.515\)). 

On the MIMIC-III dataset, BTM achieves the best AUROC at \emph{ipc}=50, 100, and 200, while remaining competitive at \emph{ipc}=500. In AUPRC, BTM is best at \emph{ipc}=50, 100, and 500, with relative gains of 2.2\%, 0.8\%, and 1.4\% over the strongest baseline at those budgets. At \emph{ipc}=200, DATM is marginally stronger by \(0.001\) AUPRC, indicating that both methods are highly competitive in this regime. At \emph{ipc}=500, BTM reaches \(0.503\), slightly exceeding the full-dataset AUPRC of \(0.499\) within experimental variability.

Among the trajectory-matching baselines, DATM and MCT are generally the strongest competitors on these public ICU benchmarks. However, BTM remains the most consistent overall in AUPRC, particularly on eICU, and is either best or near-best on MIMIC-III across all budgets. FTD appears less effective on MIMIC-III, which may reflect the structured time-series setting with a TCN, where encouraging uniformly flat trajectories can suppress temporally localised signals. DC and M3D are less competitive on the ICU benchmarks than on the NHS cohorts. Overall, the eICU and MIMIC-III results reinforce the same conclusion as the NHS cohorts that replacing raw SGD supervision with learned low-rank surrogate trajectories improves the utility of condensed datasets.

\paragraph{MIMIC-III phenotyping.}

Table~\ref{tab:mimic_pheno} reports results on the MIMIC-III dataset for 25-label multi-label phenotyping across synthetic budgets ranging from 50 to 500 \revadd{\emph{ipc}}, which we map to a per-label budget as described in Section~\ref{sec:exp_setup}. As in the binary prediction tasks, BTM performs strongly across all compression levels, with the clearest and most consistent gains appearing in macro AUPRC. 

In terms of macro AUROC, BTM achieves the best performance at \emph{ipc}=50, 100, and 200, with values of \(0.690\), \(0.713\), and \(0.727\), respectively. At \emph{ipc}=500, DATM is marginally stronger (\(0.730\) versus \(0.729\)), indicating that the two methods are highly competitive at larger synthetic budgets. Overall, the AUROC gains are smaller than in the binary tasks, which is consistent with the greater difficulty of matching supervision in the multi-label setting.

The improvements are more pronounced in macro AUPRC. BTM attains the best performance at every budget, exceeding the strongest baseline by 3.7\%, 3.7\%, 3.2\%, and 2.7\% at \emph{ipc}=50, 100, 200, and 500, respectively. Macro AUPRC increases steadily from \(0.339\) at \emph{ipc}=50 to \(0.389\) at \emph{ipc}=500, reaching 98.5\% of full-dataset performance (\(0.395\)). This pattern suggests that BTM is particularly effective at retaining weaker or less consistently reinforced label signals, which are especially important for precision--recall performance in multi-label clinical prediction.

\begin{table*}[!t]
\centering
\caption{
    Test performance on MIMIC-III 25-label multi-label phenotyping across synthetic budgets. Results are reported as macro AUROC and macro AUPRC, averaged across labels, with mean $\pm$ standard deviation over 10 random network initialisations. Best results at each \emph{ipc} are highlighted in \textcolor{blue}{blue} for BTM and \textcolor{red}{red} for the strongest baseline.
    }
\label{tab:mimic_pheno}

\textbf{MIMIC-III dataset (Phenotyping)}\\[0.5ex]
\begin{sc}
\resizebox{0.9\linewidth}{!}{
    \begin{tabular}{lccccBcccc}
    \toprule
    & \multicolumn{4}{cB}{\textbf{Macro AUROC}} & \multicolumn{4}{c}{\textbf{Macro AUPRC}} \\
    \cmidrule(lr{0.5em}){2-5} \cmidrule(lr{0.5em}){6-9}
    \textbf{Method} & \textbf{50} & \textbf{100} & \textbf{200} & \textbf{500} & \textbf{50} & \textbf{100} & \textbf{200} & \textbf{500} \\
    \midrule
    Random     
    & $0.602_{\pm 0.005}$ & $0.640_{\pm 0.005}$ & $0.661_{\pm 0.001}$ & $0.687_{\pm 0.002}$
    & $0.261_{\pm 0.005}$ & $0.290_{\pm 0.004}$ & $0.312_{\pm 0.002}$ & $0.340_{\pm 0.002}$ \\

    M3D       
    & $0.636_{\pm 0.003}$ & $0.672_{\pm 0.002}$ & $0.679_{\pm 0.002}$ & $0.700_{\pm 0.002}$ 
    & $0.283_{\pm 0.002}$ & $0.309_{\pm 0.003}$ & $0.325_{\pm 0.002}$ & $0.354_{\pm 0.001}$ \\
    
    DC        
    & $0.636_{\pm 0.002}$ & $0.659_{\pm 0.002}$ & $0.681_{\pm 0.002}$ & $0.702_{\pm 0.002}$ 
    & $0.291_{\pm 0.001}$ & $0.313_{\pm 0.003}$ & $0.321_{\pm 0.002}$ & $0.342_{\pm 0.003}$ \\

    MTT        
    & $0.668_{\pm 0.004}$ & $0.689_{\pm 0.004}$ & $0.705_{\pm 0.003}$ & $0.717_{\pm 0.003}$
    & $0.310_{\pm 0.008}$ & $0.336_{\pm 0.006}$ & $0.358_{\pm 0.005}$ & $0.372_{\pm 0.004}$ \\

    FTD        
    & $0.659_{\pm 0.005}$ & $0.681_{\pm 0.004}$ & $0.699_{\pm 0.004}$ & $0.714_{\pm 0.003}$
    & $0.299_{\pm 0.008}$ & $0.325_{\pm 0.007}$ & $0.341_{\pm 0.006}$ & $0.366_{\pm 0.005}$ \\

    MCT        
    & $0.675_{\pm 0.004}$ & $0.698_{\pm 0.004}$ & $0.713_{\pm 0.003}$ & $0.714_{\pm 0.002}$
    & $0.321_{\pm 0.007}$ & $0.342_{\pm 0.006}$ & $0.361_{\pm 0.005}$ & $0.369_{\pm 0.004}$ \\

    DATM       
    & $0.679_{\pm 0.004}$ & $0.700_{\pm 0.003}$ & $0.715_{\pm 0.003}$ & \bestOther{$0.730_{\pm 0.002}$}
    & $0.327_{\pm 0.007}$ & $0.352_{\pm 0.006}$ & $0.371_{\pm 0.004}$ & $0.379_{\pm 0.003}$ \\

    \rowcolor{btmgray}
    BTM (Ours) 
    & \bestOurs{$0.690_{\pm 0.003}$} & \bestOurs{$0.713_{\pm 0.003}$} & \bestOurs{$0.727_{\pm 0.002}$} & $0.729_{\pm 0.002}$
    & \bestOurs{$0.339_{\pm 0.006}$} & \bestOurs{$0.365_{\pm 0.005}$} & \bestOurs{$0.383_{\pm 0.004}$} & \bestOurs{$0.389_{\pm 0.003}$} \\

    \midrule
    \emph{\textbf{Full Dataset}} 
    & \multicolumn{4}{cB}{$\mathbf{0.738_{\pm 0.001}}$} 
    & \multicolumn{4}{c}{$\mathbf{0.395_{\pm 0.002}}$} \\
    \bottomrule
    \end{tabular}
}
\end{sc}
\end{table*}

Among the baselines, DATM is the strongest overall competitor. It remains consistently competitive across budgets and achieves the best macro AUROC at \emph{ipc}=500. MCT is also competitive in macro AUROC, but its macro AUPRC is weaker, while MTT remains below DATM and BTM on both metrics. FTD is again less effective, particularly in macro AUPRC, suggesting that enforcing flatter surrogate trajectories alone is insufficient when supervision must preserve diverse, label-specific structure. As in the IHM task, the non-trajectory-based baselines, DC and M3D, are less competitive overall, generally performing worse than the trajectory-matching baselines in both macro AUROC and macro AUPRC.

Overall, these results reinforce the same conclusion as in the binary tasks that replacing raw SGD supervision with structured, low-rank surrogate trajectories improves the utility of condensed datasets. In the multi-label setting, where supervision is inherently more diffuse because of label overlap, these benefits are most clearly reflected in macro AUPRC.

\subsection{Cross-Architecture Generalisation}
\label{subsec:cross-arch}

Cross-architecture evaluation tests whether the utility of a condensed dataset depends strongly on the architecture used during condensation. 
\revadd{Figure~\ref{fig:crossarch} reports AUPRC at \emph{ipc}=500. For each dataset, synthetic data are condensed using a single source architecture, highlighted by the shaded region, and then evaluated on unseen target architectures. We compare BTM only against DATM here because DATM was the strongest-performing baseline on these datasets in the main experiments.}

\begin{figure}[t]
    \centering
    \includegraphics[width=\linewidth]{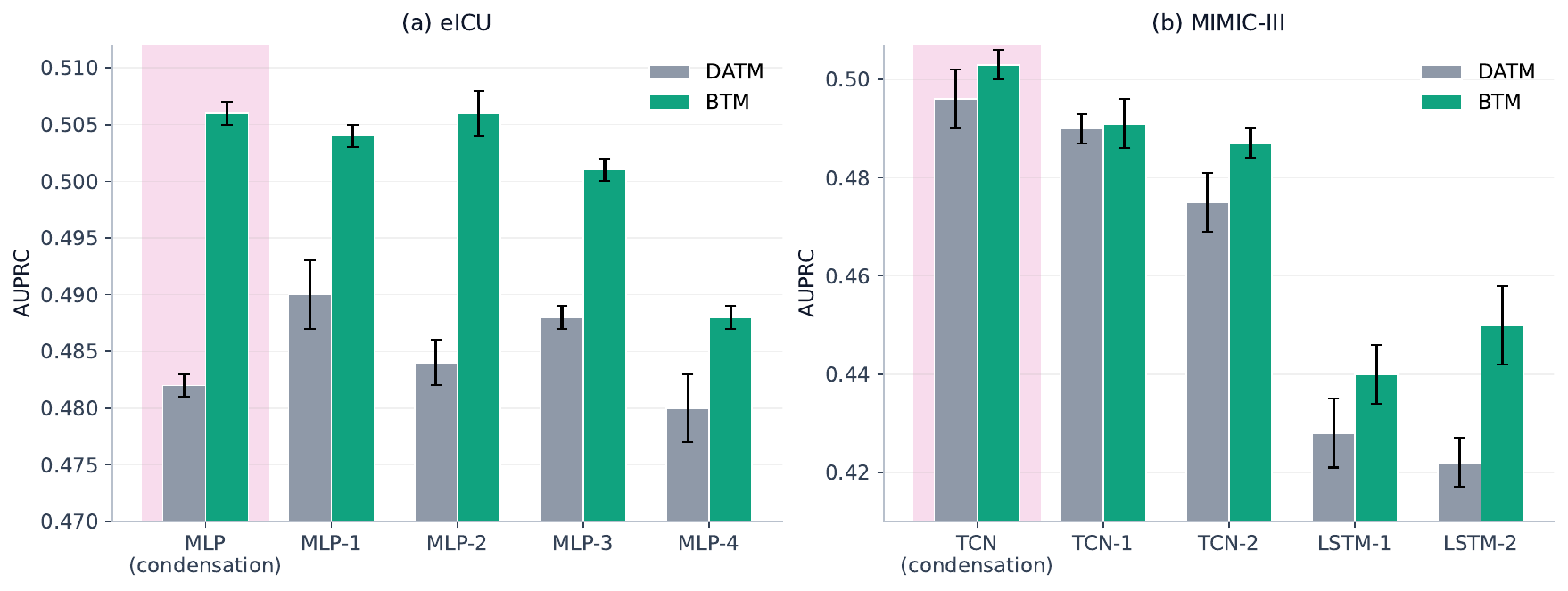}
    \caption{
    \textbf{Cross-architecture generalisation at \emph{ipc}=500.} 
    \revadd{Shaded regions indicate the source architecture used for condensation. BTM consistently outperforms DATM across target architectures, with the largest gains under larger architecture shifts.}
    }
    \label{fig:crossarch}
\end{figure}

On the eICU dataset, synthetic data condensed with an MLP transfer robustly across MLP variants of different depths and widths. BTM outperforms DATM on every target architecture, with absolute AUPRC gains ranging from 0.008 to 0.024. Its performance also remains stable across the MLP family, varying only from 0.488 to 0.506. This suggests that the supervision induced by BTM is not tightly coupled to a specific parametrisation of the MLP, but instead captures task-relevant signal that generalises across closely related architectures.

The transfer setting is more demanding on MIMIC-III, where condensation is performed with a TCN and evaluation includes both TCN and LSTM targets. Here, architecture shift changes the temporal inductive bias of the model, making transfer substantially harder. Even in this setting, BTM remains consistently stronger than DATM across all target architectures. The gains are modest within the TCN family, but larger under transfer to LSTMs, where BTM improves AUPRC by 0.012 and 0.028 on LSTM-1 and LSTM-2, respectively. Relative to its source-architecture performance, BTM also exhibits slightly smaller degradation than DATM under TCN-to-LSTM transfer.

Overall, these results suggest that structuring the supervision signal through learned low-rank surrogate trajectories improves robustness to architecture shift. This effect is most evident when transferring across model classes with different inductive biases, where preserving task-relevant structure without overfitting to architecture-specific optimisation dynamics becomes especially important.

\subsection{Storage Efficiency}
Figure~\ref{fig:storage_cost} shows that BTM substantially reduces trajectory storage across all clinical datasets. Relative to full SGD trajectories, Bézier surrogates yield approximately $33\times$ lower storage on Oxford (NHS) and eICU, and $20\times$ lower storage on MIMIC-III. These reductions lower memory requirements and make it easier to use more expert trajectories in resource-constrained clinical settings.

\begin{figure}[t]
    \centering
    \includegraphics[scale=0.5]{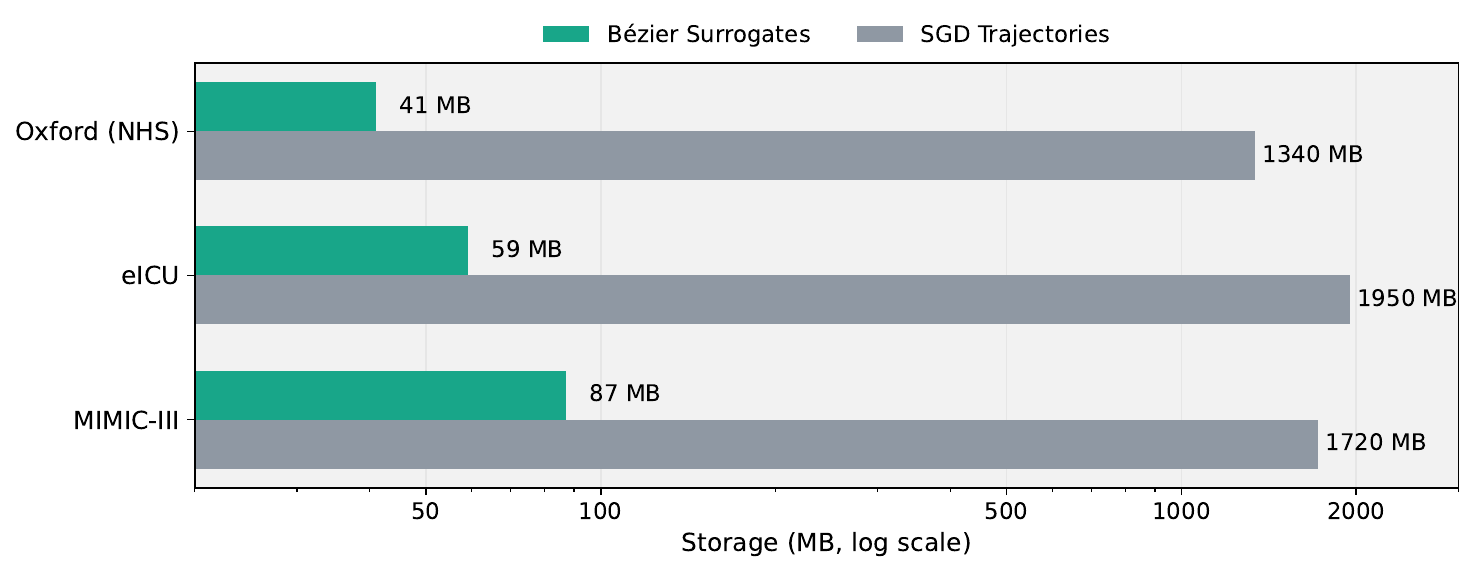}
    \caption{
    \revadd{\textbf{Teacher-trajectory checkpoint storage for standard SGD trajectories and Bézier surrogates.}
    Standard trajectory matching stores many teacher checkpoints per trajectory, whereas each quadratic Bézier surrogate is represented by only an initial point, a learned control point, and an endpoint.
    }
    }
    \label{fig:storage_cost}
\end{figure}

\begin{figure}[t]
    \centering
    \includegraphics[scale=0.5]{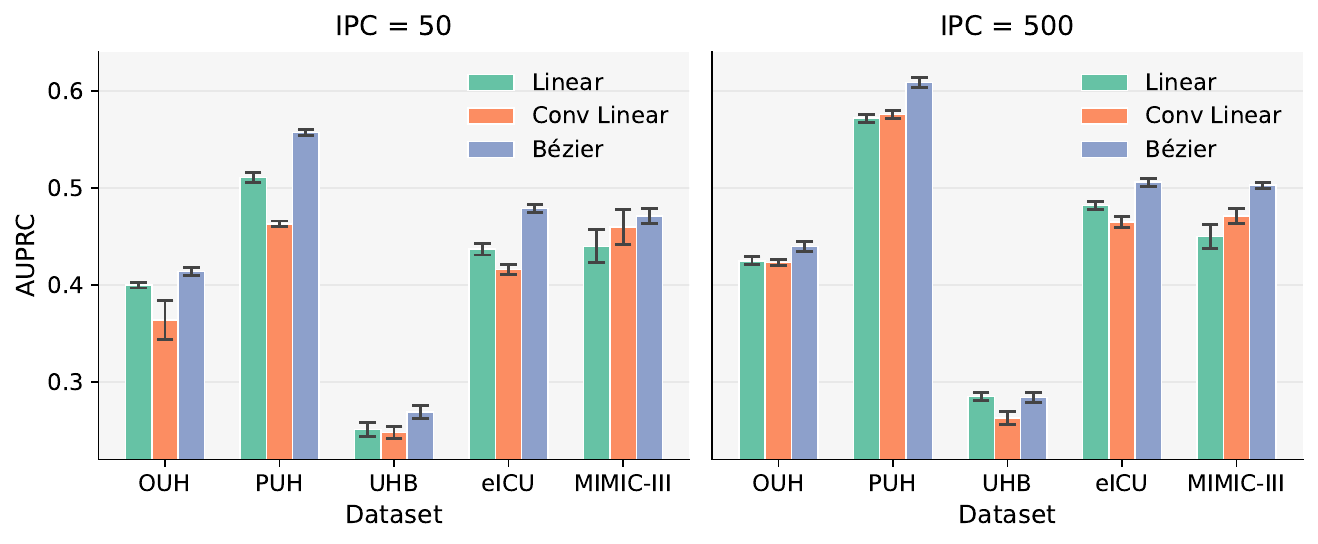}
    \caption{\textbf{Surrogate path complexity ablation.} AUPRC across datasets at \emph{ipc}=50 and \emph{ipc}=500 for three surrogate trajectory parameterisations: linear interpolation, convexified linear interpolation, and quadratic Bézier curves. OUH, PUH, and UHB denote Oxford, Portsmouth, and Birmingham NHS cohorts, respectively. Error bars denote standard deviation across runs. Bézier trajectories achieve the strongest overall performance, outperforming the linear variants in 9 of 10 dataset--budget settings.}
    \label{fig:path_ablation}
\end{figure}

\begin{figure*}[t]
\centering

\begin{subfigure}{0.325\textwidth}
    \centering
    \includegraphics[width=\linewidth]{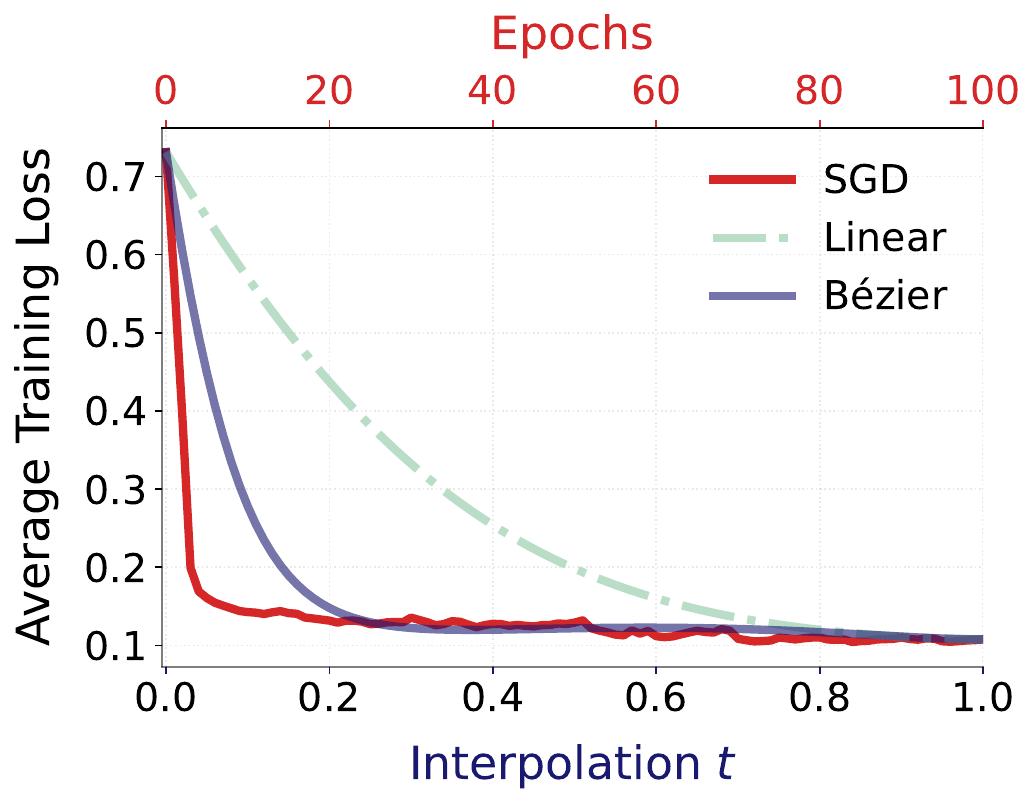}
    \caption{Portsmouth (NHS)}
\end{subfigure}
\hfill
\begin{subfigure}{0.325\textwidth}
    \centering
    \includegraphics[width=\linewidth]{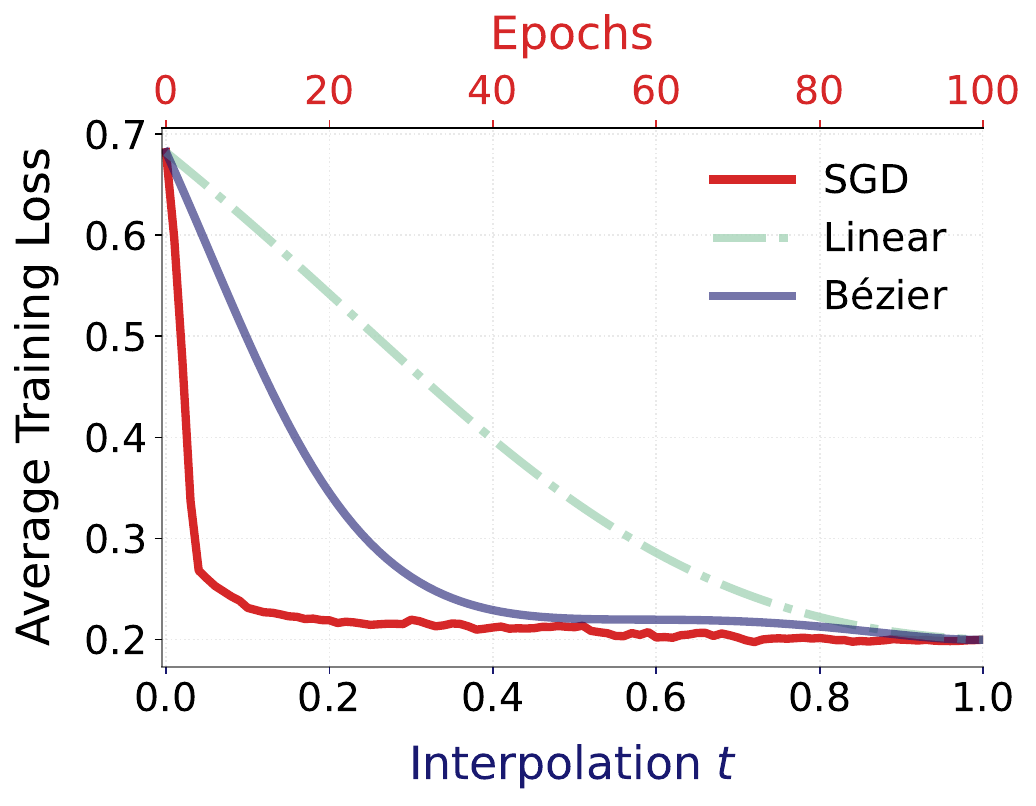}
    \caption{eICU}
\end{subfigure}
\hfill
\begin{subfigure}{0.325\textwidth}
    \centering
    \includegraphics[width=\linewidth]{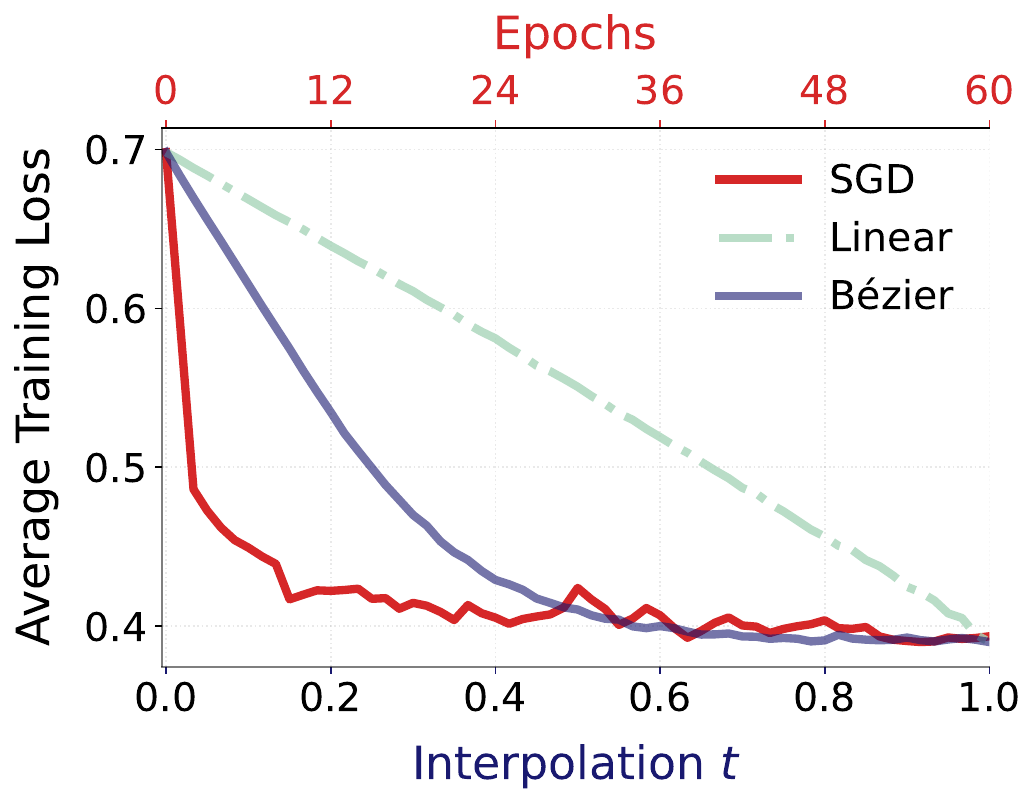}
    \caption{MIMIC-III (Phenotyping)}
\end{subfigure}

\caption{
    \revadd{\textbf{Training-loss profiles along different surrogate-supervision paths.} Average training loss is shown for linear and quadratic Bézier trajectories as a function of interpolation parameter $(t)$, and for SGD as a function of training epoch, across datasets. 
    }
    }
\label{fig:traj_loss_profile}
\end{figure*}

\subsection{Surrogate Path Complexity Ablation}
\label{subsec:path_ablation}

Figure~\ref{fig:path_ablation} evaluates how surrogate trajectory parameterisation affects condensation quality. We compare three increasingly structured path families: simple linear interpolation, a convexified linear trajectory that preserves temporal structure from SGD through fixed convex weights, and the proposed quadratic Bézier curve with a learnable control point. Results are reported in AUPRC across five clinical datasets at both low (\emph{ipc}=50) and high (\emph{ipc}=500) synthetic budgets.

A consistent trend emerges across datasets and budgets. Simple linear interpolation provides a strong baseline, indicating that reducing trajectory complexity is beneficial relative to storing and following full stochastic optimisation paths. Introducing a learnable control point yields further gains in almost every case: Bézier trajectories achieve the highest AUPRC in 9 of 10 dataset--budget settings and remain within 0.001 of the best result in the remaining case.
\revadd{Improvements are especially pronounced on the Portsmouth, eICU and MIMIC-III datasets, and persist across both low and high budgets.}

By contrast, the convexified trajectory often underperforms simple linear interpolation, despite preserving more of the temporal structure of SGD. This suggests that retaining trajectory structure alone is insufficient; the path must be explicitly optimised to provide a compact and informative supervision signal.

Figure~\ref{fig:traj_loss_profile} provides a functional view of this effect by examining the loss landscape along each trajectory. While linear interpolation smooths the high-variance behaviour of SGD, it can still traverse higher-loss regions between endpoints. In contrast, the Bézier trajectory is optimised to minimise average loss along the path and remains consistently in lower-loss regions, yielding a supervision signal that is both structured and task-aligned rather than merely smoothed.

Together, these results support the central design choice of BTM: a low-complexity surrogate path is most effective when it is both structured and task-optimised, yielding a stronger and more informative supervision signal for dataset condensation.

\section{Conclusion}
We presented a geometric view of trajectory matching showing that, for a fixed synthetic dataset, the trajectory-matching residual is lower bounded by the component of the supervision signal lying outside the student’s reachable gradient span. This yields a conditional rank-based representability bottleneck, where, when the supervision signal is not representable within that span, optimisation alone cannot eliminate the mismatch. This perspective clarifies a core limitation of existing trajectory-based dataset condensation methods and motivates the design of supervision signals that are better aligned with the student’s optimisation geometry.

Motivated by this insight, we introduced BTM, which replaces long stochastic optimisation trajectories with structured, low-rank surrogates. By learning quadratic Bézier paths that reduce average path loss and encourage a smooth descent profile from initialisation to solution, BTM provides a more structured supervision signal while preserving functionally meaningful training dynamics. Across diverse clinical datasets spanning tabular and time-series modalities, BTM consistently improves predictive performance, with the strongest gains in low-prevalence and low-data regimes. It also remains robust under architecture shift and reduces trajectory storage requirements by up to $33\times$, improving the practicality of trajectory-based condensation in resource-constrained clinical settings.

More broadly, our results suggest that effective dataset condensation depends not only on how much supervision is provided, but also on how that supervision is structured. In clinical contexts, this is particularly important because compact synthetic datasets may help support more efficient reuse of information derived from governed datasets while reducing storage and training costs. BTM currently uses a fixed quadratic parameterisation, and extending it to richer adaptive trajectory families is a natural next step. Incorporating differential privacy is another important direction for enabling safe and scalable clinical deployment.

\subsubsection*{Broader Impact Statement}

This work has potential positive impact in machine learning settings where access to large real-world datasets is constrained by governance, storage, or computational cost. By improving dataset condensation, BTM may reduce the resources required to store, transfer, and repeatedly train on large datasets, thereby supporting more reproducible experimentation, lowering barriers for smaller research groups, and enabling more sustainable model development. In clinical machine learning, these benefits are especially relevant because access to \revadd{regulated} health data is often restricted, so compact synthetic datasets may help support broader methodological research on information derived from such data in resource-constrained academic and clinical environments.

These benefits should be balanced against important risks. Condensed or synthetic datasets are not inherently safe to share, and BTM does not provide a formal privacy guarantee or replace privacy-preserving methods, data governance procedures, or access controls. Such datasets may also retain or amplify demographic imbalance, site-specific artefacts, and historical bias, which could lead to misleading conclusions or uneven downstream performance across patient subgroups.

\revadd{A further limitation specific to structured low-rank supervision is that low spectral energy does not necessarily imply low clinical relevance. Weakly reinforced directions may encode rare-condition or subgroup-specific signals, so compression could suppress information important for underrepresented patient groups. BTM does not guarantee preservation of such signals, and aggregate performance should therefore be complemented by subgroup- and rare-label evaluation where clinically appropriate.}

We therefore view BTM as a method for improving the utility of trajectory-based dataset condensation, rather than as a standalone solution for safe data sharing or clinical deployment. Any practical use should be accompanied by explicit privacy assessment, subgroup-level fairness evaluation, and task-specific external validation. An important direction for future work is to combine structured trajectory-based condensation with formal privacy guarantees and stronger safeguards against bias, misuse, and over-interpretation of condensed data.

\subsection*{Acknowledgments}
Pafue Christy Nganjimi was supported by the EPSRC Centre for Doctoral Training in Health Data Science (EP/S02428X/1). David Clifton was supported by the Pandemic Sciences Institute at the University of Oxford; the National Institute for Health Research (NIHR) Oxford Biomedical Research Center (BRC); an NIHR Research Professorship; a Royal Academy of Engineering Research Chair; the Wellcome Trust; the UKRI; and the InnoHK Hong Kong Center for Center for Cerebrocardiovascular Engineering (COCHE). Anshul Thakur was supported by the NVIDIA Academic Grant and the Lambda Research Grant.

\bibliography{tmlr}
\bibliographystyle{tmlr}

\newpage
\appendix

\setcounter{section}{0}
\setcounter{theorem}{0}
\setcounter{equation}{0}
\setcounter{lemma}{0}
\counterwithin{figure}{section}
\counterwithin{table}{section}
\counterwithin{algorithm}{section}

\section{PIPELINE OF PROPOSED METHOD}
\label{sec:btm-algo}

\begin{algorithm}[H]
\caption{\revadd{Bézier Trajectory Matching optimisation procedure.}}
\label{alg:condensation}
\begin{algorithmic}[1]
\REQUIRE Collection of $M$ optimised Bézier surrogates $\{\Phi_{\boldsymbol{\phi}^\star}^{(m)}\}_{m=1}^M$, synthetic dataset \(\tilde{\mathcal{D}} = \{(\tilde{\mathbf{x}}_i, \tilde{\mathbf{y}}_i)\}_{i=1}^{|\tilde{\mathcal{D}}|}\), number of classes $C$, input dimension $d$, student learning rate $\eta_s$, meta learning rate $\eta_{\mathbf{x}}$, number of student steps $N$, batch size $b$, maximum iterations $T_{\max}$
\ENSURE Optimised synthetic dataset $\tilde{\mathcal{D}}^\star$
\STATE Stack synthetic inputs into matrix $\tilde{\mathbf{X}} \in \mathbb{R}^{|\tilde{\mathcal{D}}| \times d}$
\STATE Stack synthetic labels into tensor $\tilde{\mathbf{Y}} \in \mathbb{R}^{|\tilde{\mathcal{D}}| \times C}$
\FOR{$i = 1$ to $T_{\max}$}
    \STATE Sample surrogate index $m \sim \mathcal{U}\{1, \ldots, M\}$
    \STATE Sample $t_s, t_e \sim \mathcal{U}(0,1)$ such that $t_s < t_e$
    \STATE $\boldsymbol{\theta}_s \leftarrow \Phi_{\boldsymbol{\phi}^\star}^{(m)}(t_s)$ \COMMENT{Start position on surrogate}
    \STATE $\boldsymbol{\theta}_e \leftarrow \Phi_{\boldsymbol{\phi}^\star}^{(m)}(t_e)$ \COMMENT{Target position on surrogate}
    \STATE $\hat{\boldsymbol{\theta}}_0 \leftarrow \boldsymbol{\theta}_s$ \COMMENT{Initialise student model}

    \FOR{$n = 0$ to $N-1$}
        \STATE Sample mini-batch $B_n \subset \tilde{\mathcal{D}}$, $|B_n| = b$
        \STATE $\hat{\boldsymbol{\theta}}_{n+1} \leftarrow \hat{\boldsymbol{\theta}}_{n} - \eta_s \nabla_{\hat{\boldsymbol{\theta}}_{n}} \!\left[ \frac{1}{b} \sum_{\substack{(\tilde{\mathbf{x}}, \tilde{\mathbf{y}}) \\ \in B_n}} \ell(\hat{\boldsymbol{\theta}}_{n}; \tilde{\mathbf{x}}, \tilde{\mathbf{y}})\right]$
    \ENDFOR

    \STATE $\mathcal{L}_{\mathrm{BTM}} \leftarrow \dfrac{\|\hat{\boldsymbol{\theta}}_{N} - \boldsymbol{\theta}_e\|_2^2}{\|\boldsymbol{\theta}_s - \boldsymbol{\theta}_e\|_2^2}$ \COMMENT{Normalised matching loss}

    \STATE $\boldsymbol{g}_L \leftarrow \nabla_{\hat{\boldsymbol{\theta}}_{N}} \mathcal{L}_{\text{BTM}} =
    \dfrac{2(\hat{\boldsymbol{\theta}}_{N} - \boldsymbol{\theta}_e)}{\|\boldsymbol{\theta}_s - \boldsymbol{\theta}_e\|_2^2}$ \COMMENT{Gradient signal}

    \STATE Compute \(\nabla_{\tilde{\mathbf{X}}}\mathcal{L}_{\mathrm{BTM}} \;\approx\; -\eta_s \sum_{n=0}^{N-1} \frac{1}{|B_n|} \sum_{(\tilde{\mathbf{x}}, \tilde{\mathbf{y}}) \in B_n} \nabla_{\tilde{\mathbf{X}}} \Big\langle \nabla_{\boldsymbol{\theta}} \ell(\hat{\boldsymbol{\theta}}_{n}; \tilde{\mathbf{x}}, \tilde{\mathbf{y}}), \boldsymbol{g}_L \Big\rangle, \)

    \STATE $\tilde{\mathbf{X}} \leftarrow \tilde{\mathbf{X}} - \eta_{\mathbf{x}} \nabla_{\tilde{\mathbf{X}}}\mathcal{L}_{\mathrm{BTM}}$ \COMMENT{Update synthetic inputs}
\ENDFOR
\STATE \textbf{return} $\tilde{\mathcal{D}}^\star = \tilde{\mathcal{D}}$
\end{algorithmic}
\end{algorithm}

\section{Proofs}
\label{app:proofs}

This section gathers the proofs of the main theoretical results. We first prove the geometric lower bounds underlying trajectory matching, namely Theorem~\ref{thm:projection-lb} and Theorem~\ref{thm:rank-bottleneck}, which formalise the restriction imposed by the student’s reachable span. We then prove Theorem~\ref{thm:bezier-span-method}, which establishes the low-dimensional structure of the displacements induced by quadratic Bézier surrogate trajectories. Finally, we prove Theorem~\ref{thm:bez-sur} and Corollary~\ref{cor:bez_supervision}. The corollary is a conditional statement rather than an algorithmic guarantee; it isolates a sufficient condition under which a synthetic trajectory would inherit the surrogate’s prediction-space behaviour.

Throughout, we retain the notation introduced in the main text. All norms on parameter vectors are Euclidean unless stated otherwise, and \(\|\cdot\|_F\) denotes the Frobenius norm for matrices.

\subsection{Proof of Theorem~\ref{thm:projection-lb}}
\label{app:proof-projection-lb}

\begin{proof}
From the student update in \eqref{eq:student-update},
\[
\hat{\boldsymbol{\theta}}_{n+1}
-
\hat{\boldsymbol{\theta}}_n
=
-\eta_s \mathbf{g}_n.
\]
Summing over \(n=0,\dots,N-1\) and using
\(\hat{\boldsymbol{\theta}}_0=\boldsymbol{\theta}_{s_k}\) gives
\[
\boldsymbol{\delta}_k
=
\hat{\boldsymbol{\theta}}_N-\boldsymbol{\theta}_{s_k}
=
-\eta_s\sum_{n=0}^{N-1}\mathbf{g}_n.
\]
Hence, by the definition of \(\mathcal{G}_k\) in \eqref{eq:grad_span},
\[
\boldsymbol{\delta}_k\in\mathcal{G}_k \qquad \text{for all } k=1,\dots,K.
\]

Using the trajectory matching objective in \eqref{eq:tm_loss}, we therefore have
\[
\mathcal{L}_{\mathrm{TM}}(s_k,e_k;\tilde{\mathcal{D}})
=
\left\|
\boldsymbol{\delta}_k-\boldsymbol{\Delta}_k
\right\|_2^2
\ge
\min_{\mathbf{v}\in\mathcal{G}_k}
\left\|
\mathbf{v}-\boldsymbol{\Delta}_k
\right\|_2^2.
\]
By the standard orthogonal projection theorem~\citep{golub2013matrix}, the minimiser of the right-hand side is
\(P_{\mathcal{G}_k}\boldsymbol{\Delta}_k\), and the corresponding minimum is
\[
\left\|
\boldsymbol{\Delta}_k
-
P_{\mathcal{G}_k}\boldsymbol{\Delta}_k
\right\|_2^2
=
\left\|
(I-P_{\mathcal{G}_k})\boldsymbol{\Delta}_k
\right\|_2^2.
\]
Consequently,
\[
\mathcal{L}_{\mathrm{TM}}(s_k,e_k;\tilde{\mathcal{D}})
\ge
\left\|
(I-P_{\mathcal{G}_k})\boldsymbol{\Delta}_k
\right\|_2^2,
\]
as required.
\end{proof}

\subsection{Proof of Theorem~\ref{thm:rank-bottleneck}}
\label{app:proof-rank-bottleneck}

\begin{proof}
\revadd{Theorem~\ref{thm:projection-lb} establishes that
\(\boldsymbol{\delta}_k\in\mathcal{G}_k\). By the definition of the global
reachable span in \eqref{eq:global_span},
\(\mathcal{G}_k\subseteq\mathcal{G}\), and hence
\begin{equation*}
    \boldsymbol{\delta}_k\in\mathcal{G},
    \qquad \text{for all } k=1,\dots,K.  
\end{equation*}

Let \(P_{\mathcal{G}}\) denote the orthogonal projector onto
\(\mathcal{G}\). By the orthogonal projection theorem~\citep{golub2013matrix}, for every \(k\),
\begin{equation*}
\mathcal{L}_{\mathrm{TM}}^{(k)}
:=
\mathcal{L}_{\mathrm{TM}}(s_k,e_k;\tilde{\mathcal{D}})
\ge
\|(I-P_{\mathcal{G}})\boldsymbol{\Delta}_k\|_2^2.
\end{equation*}

Summing over \(k\) and using
\(\mathbf{A}=[\boldsymbol{\Delta}_1,\dots,\boldsymbol{\Delta}_K]\)
gives
\begin{equation*}
\frac{1}{K}\sum_{k=1}^{K}\mathcal{L}_{\mathrm{TM}}^{(k)}
\ge
\frac{1}{K}|(I-P_{\mathcal{G}})\mathbf{A}|_F^2.
\end{equation*}

By the Eckart--Young--Mirsky theorem~\citep{eckart1936approximation}, minimising this residual over all \(r\)-dimensional subspaces is achieved by taking \(\mathcal{S}\) to be the span of the top \(r\) left singular vectors of \(\mathbf{A}\). Consequently,
\begin{equation*}
\min_{\dim(\mathcal{S})=r}
\frac{1}{K}|(I-P_{\mathcal{S}})\mathbf{A}|_F^2
=
\frac{1}{K}\sum_{j=r+1}^{r_{\mathbf{A}}}
\sigma_j(\mathbf{A})^2.
\end{equation*}

This proves the theorem.
}
\end{proof}

\subsection{Proof of Theorem~\ref{thm:bezier-span-method}}
\label{app:proof-bezier-span}

\begin{proof}
The quadratic Bézier trajectory can be written as
\begin{equation*}
\Phi(t)
=
\boldsymbol{\theta}_0
+
2t(\boldsymbol{\phi}-\boldsymbol{\theta}_0)
+
t^2(\boldsymbol{\theta}_0-2\boldsymbol{\phi}+\boldsymbol{\theta}_T).
\end{equation*}

Therefore, for \(0 \le t_s < t_e \le 1\),
\begin{align*}
\Phi(t_e)-\Phi(t_s)
&=
(t_e-t_s)
\Big[
2(\boldsymbol{\phi}-\boldsymbol{\theta}_0)
+
(t_e+t_s)
(\boldsymbol{\theta}_0-2\boldsymbol{\phi}+\boldsymbol{\theta}_T)
\Big] \\
&=
(t_e-t_s)(\boldsymbol{\theta}_T-\boldsymbol{\theta}_0)
+
(t_e-t_s)(1-t_s-t_e)
\bigl(2\boldsymbol{\phi}-\boldsymbol{\theta}_0-\boldsymbol{\theta}_T\bigr).
\end{align*}

Thus every segment displacement lies in
\[
\mathrm{span}\Bigl\{
\boldsymbol{\theta}_T-\boldsymbol{\theta}_0,\;
2\boldsymbol{\phi}-\boldsymbol{\theta}_0-\boldsymbol{\theta}_T
\Bigr\},
\]
which has dimension at most \(2\). Hence \(\dim(V_{\Phi})\le2\).

Any displacement matrix whose columns are segment displacements of \(\Phi\) therefore has column space contained in \(V_{\Phi}\), and consequently has rank at most \(2\). This proves the theorem.
\end{proof}

\subsection{Proof of Theorem~\ref{thm:bez-sur}}
\label{app:proof-bez-sur}

\revadd{
\begin{proof}
For the quadratic Bézier surrogate,
\begin{equation*}
\Phi''(t)
=
2(\boldsymbol{\theta}_0-2\boldsymbol{\phi}^\star+\boldsymbol{\theta}_T),
\end{equation*}
and hence
\[
\sup_{t\in[0,1]}\|\Phi''(t)\|_2=\kappa.
\]
Direct subtraction of the endpoint chord
\(c(t)=(1-t)\boldsymbol{\theta}_0+t\boldsymbol{\theta}_T\) gives
\begin{equation*}
\Phi(t)-c(t)
=
t(1-t)
\bigl(2\boldsymbol{\phi}^\star-\boldsymbol{\theta}_0-\boldsymbol{\theta}_T\bigr).
\end{equation*}
Since \(t(1-t)\le\tfrac14\) on \([0,1]\), with equality at
\(t=\tfrac12\),
\[
\sup_{t\in[0,1]}\|\Phi(t)-c(t)\|_2
=
\frac{\kappa}{8}.
\]

For any \(\boldsymbol{x}\in\mathcal{X}\) and \(t\in[0,1]\), the
\(L_f\)-Lipschitz assumption and the triangle inequality give
\begin{align*}
\|f_{\Phi(t)}(\boldsymbol{x})-f_{\gamma(t)}(\boldsymbol{x})\|_2
&\le
L_f\|\Phi(t)-\gamma(t)\|_2 \\
&\le
L_f\Bigl(
\|\Phi(t)-c(t)\|_2+\|\gamma(t)-c(t)\|_2
\Bigr) \\
&\le
L_f\left(\frac{\kappa}{8}+D\right).
\end{align*}

Taking the supremum over \(\boldsymbol{x}\in\mathcal{X}\) and
\(t\in[0,1]\) proves part~\textbf{(ii)}, while the preceding identities
establish part~\textbf{(i)}.

\end{proof}
}


\subsubsection{Teacher-path curvature intuition.}
For a generic teacher update sequence of the form
\begin{equation*}
\boldsymbol{\theta}_{\tau+1} = \boldsymbol{\theta}_\tau - \eta \mathbf{u}_\tau,
\end{equation*}
define the discrete second-order difference
\begin{equation*}
\mathbf{h}_\tau
:=
\boldsymbol{\theta}_{\tau+1}
-2\boldsymbol{\theta}_\tau
+\boldsymbol{\theta}_{\tau-1}.
\end{equation*}
Substituting the update at times \(\tau\) and \(\tau-1\) gives
\begin{equation*}
\mathbf{h}_\tau = -\eta(\mathbf{u}_\tau-\mathbf{u}_{\tau-1}).
\end{equation*}
If \(\mathbf{u}_\tau=\mathbf{d}_\tau+\boldsymbol{\xi}_\tau\) is decomposed into systematic drift and stochastic variability, then
\begin{equation*}
\mathbf{h}_\tau
=
-\eta\big[(\mathbf{d}_\tau-\mathbf{d}_{\tau-1})
+
(\boldsymbol{\xi}_\tau-\boldsymbol{\xi}_{\tau-1})\big].
\end{equation*}
Thus, discrete teacher curvature mixes optimisation drift with stochastic variability, whereas the quadratic Bézier surrogate has globally smooth curvature.

\subsection{Proof of Corollary~\ref{cor:bez_supervision}}
\label{app:cor}

\begin{proof}
Fix any \(\boldsymbol{x}\in\mathcal{X}\) and \(t\in[0,1]\). By the triangle inequality,
\begin{equation*}
\|f_{\gamma_{\mathrm{stu}}(t)}(\boldsymbol{x}) - f_{\gamma(t)}(\boldsymbol{x})\|_2
\le
\|f_{\gamma_{\mathrm{stu}}(t)}(\boldsymbol{x}) - f_{\Phi(t)}(\boldsymbol{x})\|_2
+
\|f_{\Phi(t)}(\boldsymbol{x}) - f_{\gamma(t)}(\boldsymbol{x})\|_2.
\end{equation*}
Using again the \(L_f\)-Lipschitz continuity of the model map,
\begin{equation*}
\|f_{\gamma_{\mathrm{stu}}(t)}(\boldsymbol{x}) - f_{\Phi(t)}(\boldsymbol{x})\|_2
\le
L_f \|\gamma_{\mathrm{stu}}(t)-\Phi(t)\|_2
\le
L_f \varepsilon_{\mathrm{syn}}.
\end{equation*}
Moreover, Theorem~\ref{thm:bez-sur}(ii) gives
\begin{equation*}
\|f_{\Phi(t)}(\boldsymbol{x}) - f_{\gamma(t)}(\boldsymbol{x})\|_2
\le
L_f\left(\frac{\kappa}{8}+D\right).
\end{equation*}
Combining the two bounds yields
\begin{equation*}
\|f_{\gamma_{\mathrm{stu}}(t)}(\boldsymbol{x}) - f_{\gamma(t)}(\boldsymbol{x})\|_2
\le
L_f\left(\varepsilon_{\mathrm{syn}}+\frac{\kappa}{8}+D\right).
\end{equation*}
Taking the supremum over \(\boldsymbol{x}\in\mathcal{X}\) and \(t\in[0,1]\) proves the corollary.
\end{proof}

\section{DATASET DETAILS}
\label{sec:dataset_details}


We provide implementation-specific preprocessing and representation details for each dataset. Shared experimental settings (splits, normalisation, and evaluation protocol) are described in Section~\ref{sec:exp_setup}.

\subsection{NHS Emergency Department Datasets (Oxford, Portsmouth, Birmingham)}

For the NHS cohorts, each patient is represented by a fixed 27-dimensional feature vector derived from admission-time data. Missing values are imputed using feature-wise medians computed on the training set.

All features correspond to measurements available at presentation, and no temporal aggregation is performed. Each dataset is treated independently, with no cross-site data mixing.

\begin{table}[ht]
\centering
\caption{Dataset characteristics for NHS emergency department datasets.}
\label{tab:nhs_examples}
\begin{tabular}{lccc}
\toprule
                        & Oxford  & Portsmouth & Birmingham \\ 
\midrule
\# Examples             & 161,955 & 38,717     & 95,236     \\
\# Positive Examples    & 2,791   & 2,005      & 790        \\ 
\# Features             & 27      & 27         & 27         \\ 
Prevalence (\%)         & 1.7     & 5.3        & 0.8        \\
\bottomrule    
\end{tabular}
\end{table}

\begin{table}[ht]
\centering
\caption{Clinical predictors used for COVID-19 diagnosis.}
\label{tab:nhs_features}
\scalebox{0.9}{
\begin{tabular}{p{0.43\textwidth}p{0.5\textwidth}}\toprule
\multicolumn{1}{c}{\textbf{Category}} & \multicolumn{1}{c}{\textbf{Features}} \\ \midrule
Vital Signs & Heart rate, respiratory rate, oxygen saturation, systolic blood pressure, diastolic blood pressure, temperature \\ 
\cmidrule{2-2}
Blood Tests & Haemoglobin, haematocrit, mean cell volume, white cell count, neutrophil count, lymphocyte count, monocyte count, eosinophil count, basophil count, platelets \\ 
\cmidrule{2-2}
Liver Function Tests \& C-reactive protein & Albumin, alkaline phosphatase, alanine aminotransferase, bilirubin, C-reactive protein \\ 
\cmidrule{2-2}
Urea \& Electrolytes & Sodium, potassium, creatinine, urea, estimated glomerular filtration rate \\ 
\bottomrule
\end{tabular}
}
\end{table}

Dataset-level statistics are summarised in Table~\ref{tab:nhs_examples}. Each dataset comprises routinely collected admission-time features, including demographics, vital signs, and laboratory blood tests, resulting in a total of 27 features per patient. The clinical predictors are grouped by category in Table~\ref{tab:nhs_features}. Missing values are handled via median imputation.

Access to the datasets is governed by NHS and Health Research Authority (HRA) approval (IRAS ID: 281832). Oxford data are available via the Infections in Oxfordshire Research Database\footnote{\url{https://oxfordbrc.nihr.ac.uk/research-themes/modernising-medical-microbiology-and-big-infection-diagnostics/infections-in-oxfordshire-research-database-iord/}}, while Portsmouth and Birmingham data are accessible upon reasonable request to the respective trusts.

\subsection{eICU}




We use the pre-processed version of the eICU Collaborative Research Database available at \url{https://physionet.org/content/mimic-eicu-fiddle-feature/1.0.0/}. This representation maps each ICU stay to a fixed 402-dimensional feature vector.

Dataset characteristics are reported in Table~\ref{tab:eicu_examples} and variables used to construct the feature representation are summarised in Table~\ref{tab:eicu_features}.

Continuous variables are discretised into quantile-based bins (typically five bins), and summary statistics (minimum, maximum, and mean) are computed where applicable. Categorical variables are one-hot encoded. This results in a fully tabular representation with no temporal dimension.

\begin{table}[ht]
\centering
\caption{Dataset characteristics for eICU.}
\label{tab:eicu_examples}
\begin{tabular}{lc}
\toprule
                        & eICU    \\ 
\midrule
\# Examples             & 49,305  \\
\# Positive Examples    & 4,501   \\ 
\# Features             & 402     \\ 
Prevalence (\%)         & 9.1     \\
\bottomrule    
\end{tabular}
\end{table}

\begin{table}[ht]
\centering
\caption{Source physiological and demographic variables used to construct the 402-dimensional feature vector for eICU IHM prediction.}
\label{tab:eicu_features}
\scalebox{0.85}{
\begin{tabular}{p{0.43\textwidth}p{0.5\textwidth}}\toprule
\multicolumn{1}{c}{\textbf{Category}} & \multicolumn{1}{c}{\textbf{Variables}} \\ \midrule
Demographics & Age, height, weight, gender \\
\cmidrule{2-2}
Admission Information & Hospital admit source, hospital admit offset, unit admit source, unit stay type, unit type, Apache admission diagnosis, airway type \\
\cmidrule{2-2}
Vital Signs & Heart rate, respiratory rate, oxygen saturation, temperature (Celsius and Fahrenheit), temperature location \\
\cmidrule{2-2}
Blood Pressure & Non-invasive systolic/diastolic/mean blood pressure, invasive systolic/diastolic/mean blood pressure, central venous pressure \\
\cmidrule{2-2}
Oxygen Support & O2 administration device, O2 level percentage \\
\cmidrule{2-2}
Laboratory Measurements & Glucose \\
\bottomrule
\end{tabular}
}
\end{table}

\subsection{MIMIC-III}

We process MIMIC-III using publicly available benchmarking code to obtain multivariate time-series representations. Only patients with at least 48 hourly observations are retained, and the first 48 hours of each stay are used for all tasks.

After removing duplicate features, each time step is represented by a 60-dimensional feature vector, yielding an input tensor of shape $48 \times 60$ per patient. Binary mask features are included to indicate the presence or absence of each measurement at each time step.

Categorical clinical variables are encoded as follows: Glasgow Coma Scale components are one-hot encoded (eye opening: 5 categories; motor response: 6 categories; verbal response: 5 categories; total score: 11 categories), and capillary refill rate is encoded as a binary variable.

For in-hospital mortality prediction, the training set is class-balanced to mitigate label imbalance. The phenotyping task follows the standard 25-label multi-label benchmark setting.

\begin{table}[!htbp]
\centering
\caption{Dataset characteristics for MIMIC-III.}
\label{tab:mimic_examples}
\begin{tabular}{lcc}
\toprule
                                & MIMIC-III (IHM) & MIMIC-III (Phenotyping)\\ 
\midrule
\# Examples                     & 21,156          & 21,728 \\
\# Time-steps                   & 48              & 48 \\
\# Features per time-step       & 60              & 60\\ 
\bottomrule    
\end{tabular}
\end{table}

\vspace{1cm}
\noindent \textsc{\textbf{List of features in MIMIC-III dataset}}

\begin{multicols}{3}
\nolinenumbers
\begin{enumerate}
\itemsep -0.5em
    \item Capillary refill rate-0.0
    \item Capillary refill rate-1.0
    \item Diastolic blood pressure
    \item Fraction inspired oxygen
    \item Glasgow coma scale eye opening-2 To Pain
    \item Glasgow coma scale eye opening-3 To speech
    \item Glasgow coma scale eye opening-1 No Response
    \item Glasgow coma scale eye opening-4 Spontaneously
    \item Glasgow coma scale eye opening-0 None
    \item Glasgow coma scale motor response-1 No Movement
    \item Glasgow coma scale motor response-3 Abnormal flexion
    \item Glasgow coma scale motor response-2 Abnormal extension
    \item Glasgow coma scale motor response-4 Flex-withdraws
    \item Glasgow coma scale motor response-5 Localizes Pain
    \item Glasgow coma scale motor response-6 Obeys Commands 
    \item Glasgow coma scale total-11
 
    \item Glasgow coma scale total-10
    \item Glasgow coma scale total-13
    \item Glasgow coma scale total-12
    \item Glasgow coma scale total-15
    \item Glasgow coma scale total-14
    \item Glasgow coma scale total-3
    \item Glasgow coma scale total-5
    \item Glasgow coma scale total-4
    \item Glasgow coma scale total-7
    \item Glasgow coma scale total-6
    \item Glasgow coma scale total-9
    \item Glasgow coma scale total-8
 
    \item Glasgow coma scale verbal response-1 No Response
    \item Glasgow coma scale verbal response-4 Confused
    \item Glasgow coma scale verbal response-2 Incomprehensible sounds
    \item Glasgow coma scale verbal response-3 Inappropriate Words
    \item Glasgow coma scale verbal response-5 Oriented
 
    \item Glucose
    \item Heart Rate
    \item Height
    \item Mean blood pressure
    \item Oxygen saturation
    \item Respiratory rate
    \item Systolic blood pressure
    \item Temperature
    \item Weight
    \item pH
    \item mask-Capillary refill rate
    \item mask-Diastolic blood pressure
    \item mask-Fraction inspired oxygen
    \item mask-Glasgow coma scale eye opening
    \item mask-Glasgow coma scale motor response
    \item mask-Glasgow coma scale total
    \item mask-Glasgow coma scale verbal response
    \item mask-Glucose
    \item mask-Heart Rate
    \item mask-Height
    \item mask-Mean blood pressure
    \item mask-Oxygen saturation
    \item mask-Respiratory rate
    \item mask-Systolic blood pressure
    \item mask-Temperature
    \item mask-Weight
    \item mask-pH

\end{enumerate}
\end{multicols}

\noindent \textsc{\textbf{List of 25 patient disorders involved in Phenotyping in MIMIC-III dataset}}

\begin{multicols}{3}
\begin{enumerate}
\nolinenumbers
\itemsep -0.5em
    \item Acute and unspecified renal failure
    \item Acute cerebrovascular disease
    \item Acute myocardial infarction
    \item Cardiac dysrhythmias 
    \item Chronic kidney disease
    \item Chronic obstructive pulmonary disease
    \item Complications of surgical/medical care 
    \item Conduction disorders 
    \item Congestive heart failure; non hypertensive
    
    \item Coronary atherosclerosis and related
    \item Diabetes mellitus with complications
    \item Diabetes mellitus without complication 
    \item Disorders of lipid metabolism
    \item Essential hypertension
    \item Fluid and electrolyte disorders
    \item Gastrointestinal haemorrhage
    \item Hypertension with complications
    
    \item Other liver diseases 
    \item Other lower respiratory disease
    \item Other upper respiratory disease
    \item Pleurisy; pneumothorax; pulmonary collapse
    \item Pneumonia
    \item Respiratory failure; insufficiency; arrest
    \item Septicemia (except in labour)
    \item Shock 
\end{enumerate}
\end{multicols}

\section{IMPLEMENTATION DETAILS}
\label{sec:implementation}

\revadd{
This section specifies the architectures, optimisation procedures, and
method-specific settings used in all experiments. We first describe the common
clinical-data protocol, then distinguish teacher-trajectory generation from
synthetic-data condensation, and finally give the implementation details for
each baseline. Unless stated otherwise, all methods use the same data splits,
backbone architecture, label representation, synthetic-data budget, and
downstream evaluation protocol.
}

\subsection{Model Architectures}
\label{subsec:architectures}

\revadd{
\textbf{Tabular datasets.}
For the NHS cohorts and eICU, we use a multi-layer perceptron (MLP)
\citep{rumelhart1986learning}. The base network contains one hidden layer,
ReLU activation, and dropout of 0.25 after the hidden layer. The hidden
dimension, $h$, is $64$ for the NHS cohorts and $256$ for eICU. The final layer produces a single output for binary classification.

For cross-architecture evaluation on eICU, we additionally evaluate wider and
deeper MLP variants while retaining the same activation and dropout choices.
These architectures are summarised in Table~\ref{tab:mlp_variants}.
}

\begin{table}[H]
\centering
\caption{\revadd{
MLP variants used for cross-architecture evaluation on eICU. The base MLP
used during condensation is marked by $\dagger$; the remaining networks are
used only as unseen evaluation architectures.
}}
\vspace{-1mm}
\label{tab:mlp_variants}
\small
\begin{tabular}{l l l}
\toprule
\textbf{Model} & \textbf{Architecture} & \textbf{Key details} \\
\midrule
MLP$^\dagger$ & 1 hidden layer & $h$ units, ReLU, dropout 0.25 \\
MLP-1 & 1 hidden layer & $2h$ units, ReLU, dropout 0.25 \\
MLP-2 & 1 hidden layer & $4h$ units, ReLU, dropout 0.25 \\
MLP-3 & 2 hidden layers & $(h,2h)$, ReLU, dropout 0.25 \\
MLP-4 & 3 hidden layers & $(h,2h,4h)$, ReLU, dropout 0.25 \\
\bottomrule
\end{tabular}
\end{table}

\revadd{
\textbf{Time-series datasets.}
For MIMIC-III, we use a temporal convolutional network (TCN)
\citep{bai2018empirical}. The base TCN contains one residual temporal block
with 64 channels, kernel size 9, dilation 1, BatchNorm, PReLU activations, and
dropout of 0.75. The input is a $48\times60$ multivariate time series.
Temporal features are mean-pooled before the output layer. The output
dimension is one for in-hospital mortality and 25 for phenotyping.

For cross-architecture evaluation on MIMIC-III in-hospital mortality, we
additionally evaluate the TCN and LSTM variants listed in
Table~\ref{tab:arch_variants}.
}

\begin{table}[H]
\centering
\caption{\revadd{
Alternative architectures used for cross-architecture evaluation on
MIMIC-III. The base TCN used during condensation is marked by $\dagger$.
}}
\vspace{-1mm}
\label{tab:arch_variants}
\small
\begin{tabular}{l l l}
\toprule
\textbf{Model} & \textbf{Architecture} & \textbf{Key details} \\
\midrule
TCN$^\dagger$ & Single-scale TCN
& 1 block, kernel 9, 64 channels, dropout 0.75 \\
TCN-1 & Multi-scale TCN
& 1 block, kernels [3,5,7], 192 channels, dropout 0.5 \\
TCN-2 & Multi-scale TCN
& 2 blocks, kernels [3,5], 256 channels/layer, dropout 0.5,
exponential dilation \\
LSTM-1 & LSTM
& 1 layer, hidden dimension 128, unidirectional, dropout 0.25 \\
LSTM-2 & LSTM
& 1 layer, hidden dimension 256, unidirectional, dropout 0.25 \\
\bottomrule
\end{tabular}
\end{table}

\subsection{Common Clinical-Data Protocol}
\label{subsec:baseline_adaptation}

\revadd{
The condensation methods considered in this work were predominantly developed
for image classification. We adapt them to clinical tabular and time-series
data while preserving the defining optimisation objective of each method.

All methods use the modality-specific backbone described in
Section~\ref{subsec:architectures}. Image-specific operations, including
cropping, flipping, colour transformations, and differentiable image
augmentation, are disabled. Synthetic tabular examples are optimised directly
in the same normalised feature space as the real data. Synthetic MIMIC-III
examples retain the same temporal length and feature-channel dimensions as
the preprocessed real inputs. Missingness indicators, where present, occupy
the same channels as in the real representation and are treated as learnable
synthetic input features under the condensation objective.

For binary tasks, synthetic examples use fixed hard class labels. For
MIMIC-III phenotyping, they use fixed multi-hot label vectors and the per-label
budgeting scheme described in Section~\ref{sec:exp_setup}. Synthetic-label
optimisation is disabled for every method, including DATM, so that differences
between methods reflect the condensation objective and trajectory supervision
rather than differences in label optimisation.

All teacher trajectories are generated exclusively from the real training
split. Hyperparameters are selected using the validation split, and the
held-out test split is used only for final evaluation. After condensation,
freshly initialised networks are trained from scratch on the synthetic data
and evaluated on the real test set.
}

\subsection{Hyperparameter Selection and Learning-Rate Convention}
\label{subsec:hyperparameter_selection}

\revadd{
All hyperparameter selection is performed using the validation split, and the held-out test set is never used for tuning. 

For teacher-trajectory generation, we evaluate the full grid of learning rates $\{0.01,0.02,0.05,0.1\}$ and momentum values $\{0.0,0.1,0.5,0.9\}$, giving 16 candidate configurations per dataset. Learning rate and momentum are selected using validation AUPRC. The teacher-training epoch budget is selected from preliminary validation curves by choosing a fixed training horizon once performance has plateaued. The resulting settings are fixed for all teacher trajectories and are reported in Table~\ref{tab:traj_settings}.

Downstream network-training settings are selected separately using the same validation metric and learning-rate and momentum grid. The settings used during condensation and final evaluation are reported in Tables~\ref{tab:eval_condensation} and~\ref{tab:eval_final}.

Where additional method-specific hyperparameters are tuned, selection is based on validation AUPRC. For method-specific hyperparameters requiring tuning, each method is allocated the same maximum number of candidate configurations for each dataset and synthetic-data budget. Final settings are reported with the corresponding implementation details; parameters inherited from the original methods or fixed as implementation choices are not additionally tuned.

For trajectory-matching methods, $\eta_{\mathrm{T}}$ denotes the teacher learning rate, $\eta_{\mathrm{s}}$ the student learning rate, and $\eta_{\mathrm{x}}$ the synthetic-input learning rate. We initialise $\eta_{\mathrm{s}}$ to $0.05$ for the tabular datasets, $0.02$ for MIMIC-III in-hospital mortality, and $0.1$ for MIMIC-III phenotyping. During condensation, $\eta_{\mathrm{s}}$ is meta-optimised using SGD with learning rate $10^{-4}$ and momentum $0.5$. Unless otherwise stated, synthetic inputs are optimised using SGD with $\eta_{\mathrm{x}}=100$ and momentum $0.9$ for 20,000 outer iterations.

The synthetic mini-batch size is
\[
b=\min(2\times\textit{ipc},256).
\]
Samples are drawn without replacement where possible and with replacement when the available synthetic set is smaller than the required mini-batch. All synthetic datasets are initialised from randomly sampled real training examples, with a fixed random seed of 42 used for this initial sample selection.
}

\subsection{Teacher-Trajectory Generation}
\label{subsec:teacher_trajectories}

\revadd{
For every dataset, we generate 50 teacher trajectories from independent random
network initialisations. Each teacher is trained only on the real training
split. Unless stated otherwise, the underlying optimiser is SGD with the
dataset-specific learning rate, momentum, and epoch budget reported in
Table~\ref{tab:traj_settings}. We use a mini-batch size of 256 and weight
decay of $5\times10^{-4}$ for teacher training. Model parameters are saved
after every training epoch, so each stored trajectory consists of one
checkpoint per epoch in addition to the initialisation.
}

\begin{table}[t]
\centering
\caption{\revadd{
Base SGD settings used to generate real-data teacher trajectories.}}
\vspace{-1mm}
\label{tab:traj_settings}
\small
\begin{tabular}{l c c c}
\toprule
\textbf{Dataset} & $\boldsymbol{\eta_{\mathrm{T}}}$ &
\textbf{Momentum} & \textbf{Epochs} \\
\midrule
MIMIC-III (IHM) & 0.02 & 0.0 & 60 \\
MIMIC-III (Phenotyping) & 0.05 & 0.9 & 60 \\
eICU & 0.05 & 0.9 & 100 \\
Oxford & 0.02 & 0.9 & 100 \\
Portsmouth & 0.01 & 0.9 & 100 \\
Birmingham & 0.02 & 0.9 & 100 \\
\bottomrule
\end{tabular}
\end{table}

\revadd{
\paragraph{MTT.}
MTT uses the standard SGD trajectories generated using the settings in Table~\ref{tab:traj_settings}. All stored checkpoints are eligible for discrete segment sampling subject to the start and target rules below.
}

\revadd{
\paragraph{FTD.} FTD uses teacher trajectories generated with Surrogate Gap Guided
Sharpness-Aware Minimization (GSAM), with a fixed perturbation radius
$\rho=0.01$ and GSAM coefficient $\alpha=0.3$. The underlying SGD optimiser
uses the dataset-specific learning rate $\eta_{\mathrm{T}}$ and momentum
reported in Table~\ref{tab:traj_settings}, with the learning rate decayed
linearly to zero over the teacher-training run.
}

\revadd{
\paragraph{DATM.}
For the tabular datasets (NHS and eICU), DATM uses GSAM teacher trajectories
generated with the same procedure as FTD. For MIMIC-III IHM and phenotyping,
DATM instead uses the standard SGD trajectories from
Table~\ref{tab:traj_settings}, because GSAM trajectories underperformed the
corresponding SGD trajectories in preliminary validation experiments.
}

\revadd{
\paragraph{MCT.}
MCT is constructed from the corresponding MTT SGD trajectories. We retain the initial and final states together with the checkpoints at one-third and two-thirds of the stored checkpoint sequence, producing a three-segment piecewise-linear path. Within each segment, the trajectory is mapped to the monotone coordinate $\beta^{(t)}$ defined by cumulative normalised training-step length and convexified following the original MCT procedure. The resulting continuous coordinate is normalised to $t\in[0,1]$, from which start and target points are sampled continuously.
}

\revadd{
\paragraph{BTM.}
BTM uses the initial and final states of each standard SGD teacher trajectory
and replaces the intermediate stored SGD path with one quadratic Bézier
surrogate
\[
\Phi(t)
=
(1-t)^2\boldsymbol{\theta}_0
+
2t(1-t)\boldsymbol{\phi}
+
t^2\boldsymbol{\theta}_T.
\]
The control point $\boldsymbol{\phi}$ is learned once for each teacher
trajectory using the procedure in
Section~\ref{subsec:control_opt}. The resulting surrogate is then reused
throughout condensation.
}

\revadd{
For each method, the same set of 50 underlying teacher initialisations is
reused across all synthetic-data budgets.
}

\subsection{Trajectory-Matching Condensation}
\label{subsec:trajectory_condensation}

\revadd{
All trajectory-matching methods follow the same basic student optimisation
procedure. A teacher start state $\boldsymbol{\theta}_s$ and target state
$\boldsymbol{\theta}_e$ are sampled according to the method-specific rule
below. The student is initialised at
$\hat{\boldsymbol{\theta}}_0 := \boldsymbol{\theta}_s$ and trained for $N$
gradient steps on mini-batches from the current synthetic dataset:
\[
\hat{\boldsymbol{\theta}}_{n+1}
=
\hat{\boldsymbol{\theta}}_n
-
\eta_{\mathrm{s}}
\nabla_{\boldsymbol{\theta}}
\mathcal{L}_{\mathrm{task}}
(\hat{\boldsymbol{\theta}}_n;\widetilde{\mathcal{B}}_n).
\]

The trajectory-matching loss is
\[
\mathcal{L}_{\mathrm{TM}}
=
\frac{
\|\hat{\boldsymbol{\theta}}_N-\boldsymbol{\theta}_e\|_2^2
}{
\|\boldsymbol{\theta}_s-\boldsymbol{\theta}_e\|_2^2
}.
\]

}

\begin{table*}[t]
\centering
\caption{\revadd{
Trajectory-matching settings. $M$ denotes checkpoint separation for MTT, FTD, and DATM, while $\Delta t$ denotes continuous segment length for MCT and BTM.}}
\label{tab:tm_condensation_settings}
\small
\begin{tabular}{l c c p{0.48\linewidth}}
\toprule
\textbf{Method} & \textbf{Target separation} & \textbf{Student steps $N$} & \textbf{Start and target sampling} \\
\midrule

MTT
& $M=5$
& $N=60$
& Sample a valid start checkpoint uniformly and use the checkpoint five epochs later as the target. \\

FTD
& $M=5$
& $N=60$
& Apply the same discrete sampling rule to the GSAM trajectory. \\

DATM
& $M=5$
& $N=60$
& Sample the start checkpoint from the current DATM difficulty interval and use the checkpoint five epochs later as the target. \\

MCT
& $\Delta t=0.2$
& $N=30$
& Sample $t_s\sim\mathcal{U}(0,0.8)$ and set $t_e=t_s+0.2$ along the convexified trajectory. \\

BTM
& $\Delta t=0.2$
& $N=30$
& Sample $t_s\sim\mathcal{U}(0,0.8)$, set $t_e=t_s+0.2$, and use
$\boldsymbol{\theta}_s=\Phi(t_s)$ and
$\boldsymbol{\theta}_e=\Phi(t_e)$. \\

\bottomrule
\end{tabular}
\end{table*}

\subsubsection{DATM Difficulty Schedule}
\label{subsec:datm_schedule}

\revadd{
For DATM, we follow the sequential-generation procedure of
\citet{guo2024datm}. Let $T^-$, $T_0$, and $T^+$ denote the lower bound,
initial upper bound, and final upper bound of the eligible start-checkpoint
range. The upper bound is increased by one checkpoint every $I$ outer
iterations until it reaches $T^+$. At each iteration, the start checkpoint is
sampled uniformly from the eligible integer checkpoints, and the target is
five teacher epochs later. We use $I=200$ for $ipc\in\{50,100\}$ and
$I=400$ for $ipc\in\{200,500\}$.

For eICU and the NHS datasets,
\[
(T^-,T_0,T^+)=
\begin{cases}
(0,20,40), & 50\ \text{\emph{ipc}}, \\
(10,20,60), & 100\ \text{\emph{ipc}}, \\
(15,40,90), & 200,500\ \text{\emph{ipc}}.
\end{cases}
\]

For MIMIC-III in-hospital mortality and phenotyping,
\[
(T^-,T_0,T^+)=
\begin{cases}
(0,10,25), & 50\ \text{\emph{ipc}}, \\
(5,15,40), & 100\ \text{\emph{ipc}}, \\
(10,30,45), & 200\ \text{\emph{ipc}}, \\
(15,40,50), & 500\ \text{\emph{ipc}}.
\end{cases}
\]
}

\subsubsection{Path-Complexity Ablations}
\label{subsec:path_complexity_impl}

\revadd{
The linear and convexified-linear variants in the path-complexity ablation
use the same teacher endpoints, synthetic-data optimiser, student optimiser,
condensation horizon, segment length $\Delta t=0.2$, and $N=30$ student
updates as BTM. The variants differ only in the teacher path between the
endpoints.

The linear variant uses ordinary interpolation between $\boldsymbol{\theta}_0$ and $\boldsymbol{\theta}_T$. The convexified-linear variant follows the MCT-style parameterisation, using SGD-derived $\beta$ coefficients to determine progress between the teacher endpoints. The variants therefore differ in the parameterisation of the surrogate teacher path while the remaining condensation procedure is held fixed.

}

\subsection{Distribution- and Gradient-Matching Baselines}
\label{subsec:nontrajectory_baselines}

\revadd{
The M3D \citep{zhang2024m3d} and DC \citep{zhao2020dataset} baselines use
the same clinical representations, fixed labels, synthetic-data budgets,
validation split, and downstream evaluation protocol as the
trajectory-matching methods. Synthetic examples are initialised from randomly
sampled real training examples. For binary tasks, each method applies its
class-conditional objective separately to the positive and negative classes.
For MIMIC-III phenotyping, positive and negative terms are computed separately
for each of the 25 phenotypes and averaged across labels. Image-specific
components are adapted to their tabular or time-series counterparts, while
the remaining procedures follow the original implementations unless stated
otherwise.
}

\subsubsection{Dataset Condensation by Minimising Maximum Mean Discrepancy}
\label{subsec:m3d_impl}

\revadd{
For M3D, we retain the class-conditional Maximum Mean Discrepancy objective
and replace the image encoder with the corresponding modality-specific
backbone from Section~\ref{subsec:architectures}, using its penultimate
representation for distribution matching. Following the original method, the
feature extractor is randomly initialised, remains fixed, and is resampled
every five condensation iterations. Image-specific factorisation and
upsampling are omitted.

We use $b_{\mathrm{real}}=256$, so each class-specific MMD update uses all
corresponding synthetic examples and up to $16b_{\mathrm{real}}=4096$
randomly sampled real examples, or all available examples when fewer are
present. Synthetic inputs are optimised using SGD with learning rate $100$,
momentum $0.9$, and Nesterov acceleration for 20,000 iterations. The
remaining M3D procedure follows \citet{zhang2024m3d}.
}

\subsubsection{Dataset Condensation with Gradient Matching}
\label{subsec:dc_impl}

\revadd{
For DC, we retain the class-conditional gradient-matching objective and the
original layer-wise, output-node-wise cosine gradient distance. Each
gradient-matching update uses up to 256 randomly sampled real examples from
the relevant class together with the corresponding synthetic examples.

After each synthetic-data update, tabular features are constrained to their
observed training-set ranges. For MIMIC-III, this constraint is applied
separately to each time-feature coordinate while preserving the
$48 \times 60$ input structure.

Synthetic inputs are optimised using SGD with learning rate $0.1$, momentum
$0.5$, and Nesterov acceleration for the tabular datasets, and learning rate
$10$, momentum $0.9$, and Nesterov acceleration for MIMIC-III.
Network-training steps use the dataset-specific settings in
Section~\ref{subsec:hyperparameter_selection}.

For all synthetic-data budgets, we use $T=50$ gradient-matching steps and
$\zeta_\theta=10$ network-training steps. A new network is randomly
initialised at the beginning of each condensation iteration. DC is run for
2,000 condensation iterations without additional early stopping. The remaining optimisation procedure follows \citet{zhao2020dataset}.
}

\subsection{BTM Control-Point Optimisation}
\label{subsec:control_opt}

\revadd{
For each SGD teacher trajectory, the Bézier control point
$\boldsymbol{\phi}$ is initialised at the midpoint of the teacher endpoints,
\[
\boldsymbol{\phi}^{(0)}
=
\frac{
\boldsymbol{\theta}_0+\boldsymbol{\theta}_T
}{2},
\]
and optimised once before dataset condensation. At each optimisation step we
sample $N_{\mathrm{MC}}$ interpolation values independently from
$\mathcal{U}(0,1)$, evaluate the task loss at the corresponding points on the
Bézier curve, and update $\boldsymbol{\phi}$ using their average loss.

Because $\boldsymbol{\theta}_0$ is a random initialisation, the objective is
not intended to produce uniformly low loss along the complete curve. Instead,
it reduces unnecessary high-loss excursions relative to the endpoint-defined
path while preserving the exact teacher initialisation and final state.
}

\begin{algorithm}[H]
\caption{Control Point Optimisation}
\label{alg:control_optimization}
\begin{algorithmic}[1]
\REQUIRE Teacher endpoints
$\boldsymbol{\theta}_0,\boldsymbol{\theta}_T$,
training dataset $\mathcal{D}$,
control-point learning rate $\eta_\phi$,
tolerance $\epsilon$,
maximum updates $T_{\max}$,
Monte Carlo samples $N_{\mathrm{MC}}$
\ENSURE Optimised control point $\boldsymbol{\phi}^{\star}$

\STATE
$\boldsymbol{\phi}
\leftarrow
(\boldsymbol{\theta}_0+\boldsymbol{\theta}_T)/2$

\FOR{$j=1$ to $T_{\max}$}

    \STATE Sample
    $\{t_i\}_{i=1}^{N_{\mathrm{MC}}}
    \overset{\mathrm{iid}}{\sim}\mathcal{U}(0,1)$

    \STATE $\mathcal{L}_{\mathrm{avg}}\leftarrow0$

    \FOR{$i=1$ to $N_{\mathrm{MC}}$}

        \STATE
        $\boldsymbol{\theta}_{t_i}
        \leftarrow
        (1-t_i)^2\boldsymbol{\theta}_0
        +2t_i(1-t_i)\boldsymbol{\phi}
        +t_i^2\boldsymbol{\theta}_T$

        \STATE Sample real-data mini-batch
        $\mathcal{B}_i\subset\mathcal{D}$

        \STATE
        $\mathcal{L}_{\mathrm{avg}}
        \leftarrow
        \mathcal{L}_{\mathrm{avg}}
        +
        \frac{1}{N_{\mathrm{MC}}}
        \mathcal{L}_{\mathrm{task}}
        (f_{\boldsymbol{\theta}_{t_i}},\mathcal{B}_i)$

    \ENDFOR

    \STATE
    $\boldsymbol{g}
    \leftarrow
    \nabla_{\boldsymbol{\phi}}
    \mathcal{L}_{\mathrm{avg}}$

    \IF{$\|\boldsymbol{g}\|_2<\epsilon$}
        \STATE \textbf{break}
    \ENDIF

    \STATE
    $\boldsymbol{\phi}
    \leftarrow
    \boldsymbol{\phi}-\eta_\phi\boldsymbol{g}$

\ENDFOR

\STATE \textbf{return} $\boldsymbol{\phi}^{\star}=\boldsymbol{\phi}$
\end{algorithmic}
\end{algorithm}

\revadd{
The task loss is binary cross-entropy with logits for binary classification
and multi-label binary cross-entropy for phenotyping. Control-point
optimisation uses real-data mini-batches of 256 and SGD with learning rate
$\eta_\phi=10^{-2}$, momentum $0.9$, and no weight decay. We use
$\epsilon=10^{-5}$ and a maximum of $T_{\max}=300$ optimisation steps.
The Monte Carlo sample size is $N_{\mathrm{MC}}=5$ throughout. The control point is optimised once for each teacher trajectory, and the resulting surrogate is reused throughout condensation.
}

\subsection{Evaluation Protocol}
\label{subsec:evaluation_protocol}

\revadd{
\textbf{Dataset selection during condensation.}
Every 10 outer condensation iterations, the current synthetic dataset is evaluated by training a freshly initialised network from scratch. Synthetic datasets are selected using validation AUPRC only; the test set is not consulted during model selection. Evaluation networks are trained using SGD with weight decay $5\times10^{-4}$ and batch size $\min(2\times\textit{ipc},256)$. Dataset-specific learning rates, momentum, and epoch budgets are given in Table~\ref{tab:eval_condensation}.
}

\begin{table}[t]
\centering
\caption{Evaluation settings during condensation.}
\label{tab:eval_condensation}
\small
\setlength{\tabcolsep}{3pt}
\begin{tabular}{l c c c}
\toprule
\textbf{Dataset} & \textbf{LR} & \textbf{Momentum} & \textbf{Epochs} \\
\midrule
MIMIC-III (IHM) & 0.02 & 0.9 & 60 \\
MIMIC-III (Phenotyping) & 0.05 & 0.9 & 60 \\
eICU, Oxford, Portsmouth, Birmingham & 0.05 & 0.9 & 50 \\
\bottomrule
\end{tabular}
\end{table}

\revadd{
\textbf{Final evaluation.}
After condensation, each selected synthetic dataset is evaluated using 10 fresh network initialisations with seeds 0 through 9. The same SGD optimiser, weight decay, and batch-size rule are used with the longer schedules in Table~\ref{tab:eval_final}. All reported test metrics are computed at this stage.
}

\begin{table}[t]
\centering
\caption{Final evaluation settings.}
\label{tab:eval_final}
\small
\setlength{\tabcolsep}{3pt}
\begin{tabular}{l c c c}
\toprule
\textbf{Dataset} & \textbf{LR} & \textbf{Momentum} & \textbf{Epochs} \\
\midrule
MIMIC-III (IHM) & 0.02 & 0.9 & 80 \\
MIMIC-III (Phenotyping) & 0.05 & 0.9 & 80 \\
eICU, Oxford, Portsmouth, Birmingham & 0.05 & 0.9 & 100 \\
\bottomrule
\end{tabular}
\end{table}

\revadd{
\textbf{Cross-architecture evaluation.}
For eICU, the condensed datasets are evaluated on the MLP variants in
Table~\ref{tab:mlp_variants}. For MIMIC-III IHM, they are evaluated on the
TCN and LSTM variants in Table~\ref{tab:arch_variants}. Within the MLP and TCN families, the evaluation optimiser follows the corresponding base
architecture. The LSTM models use learning rate $0.01$, momentum $0.9$, and
30 training epochs.

We report results over 10 random seeds, using seeds 0 through 9 for all
reported experiments.
}

\revadd{
\section{Empirical Validation of Theorems}

\subsection{Projection Residuals of Learned Synthetic Datasets}
\label{app:projection_residuals}

We directly evaluate the representability bottleneck using the reachable
gradient spans obtained from the final learned synthetic datasets. For each
sampled teacher segment, we freeze the synthetic dataset, initialise the
student at the corresponding segment start, and construct
$\mathcal{G}_k$ as in \eqref{eq:grad_span}. We then measure the projection
residual from Theorem~\ref{thm:projection-lb}, normalised by
$\|\boldsymbol{\Delta}_k\|_2^2$, together with the correspondingly
normalised realised trajectory-matching error. For each evaluated condition
at \emph{ipc}=500, we sample 500 segments from \revadd{the same }50 trajectories, with 10
segments per trajectory. We use DATM as the comparison method because it is
among the strongest trajectory-based baselines in the main experiments.
}

\begin{figure*}[t]
    \centering

    \begin{tabular}{cc}

        \multicolumn{2}{c}{\textbf{Portsmouth (5.3\%)}} \\[-0.15em]
        \includegraphics[width=0.40\textwidth]
        {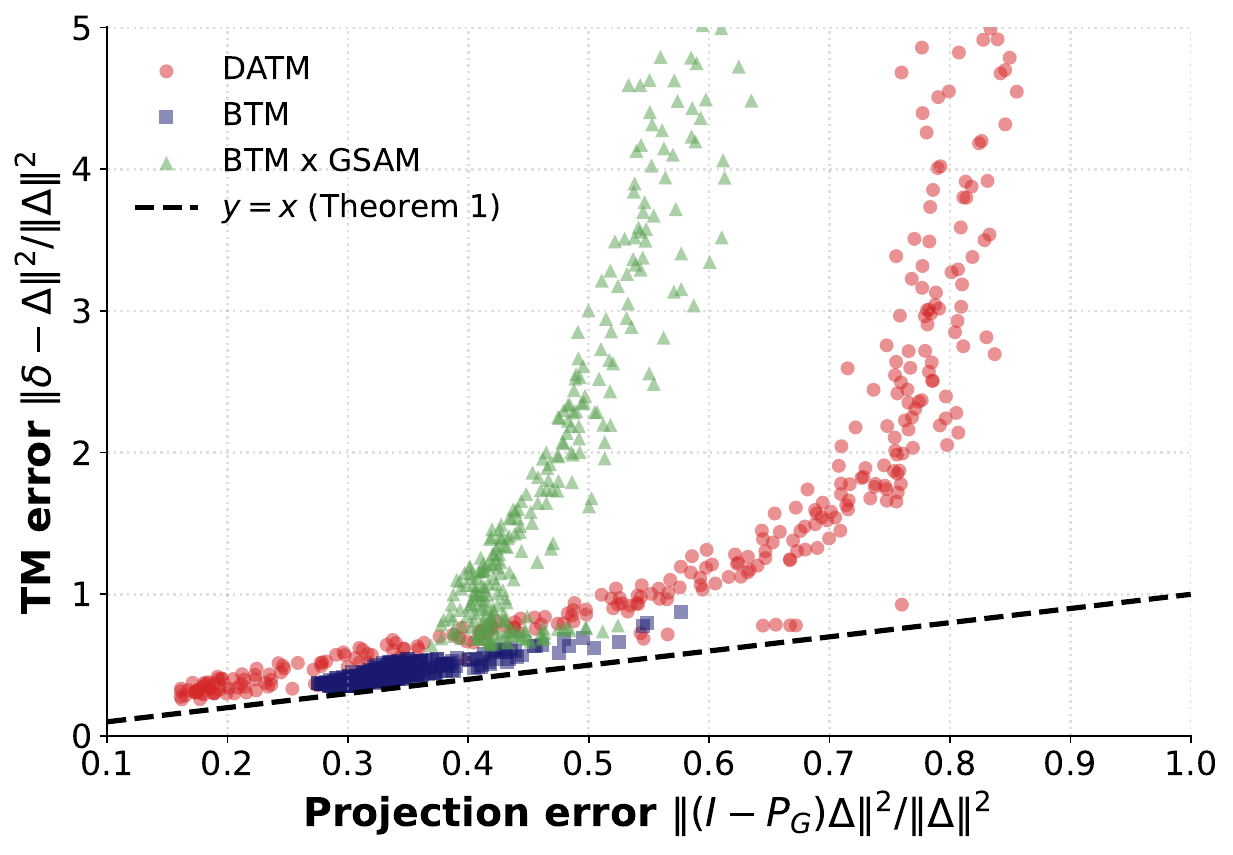}
        &
        \includegraphics[width=0.40\textwidth]
        {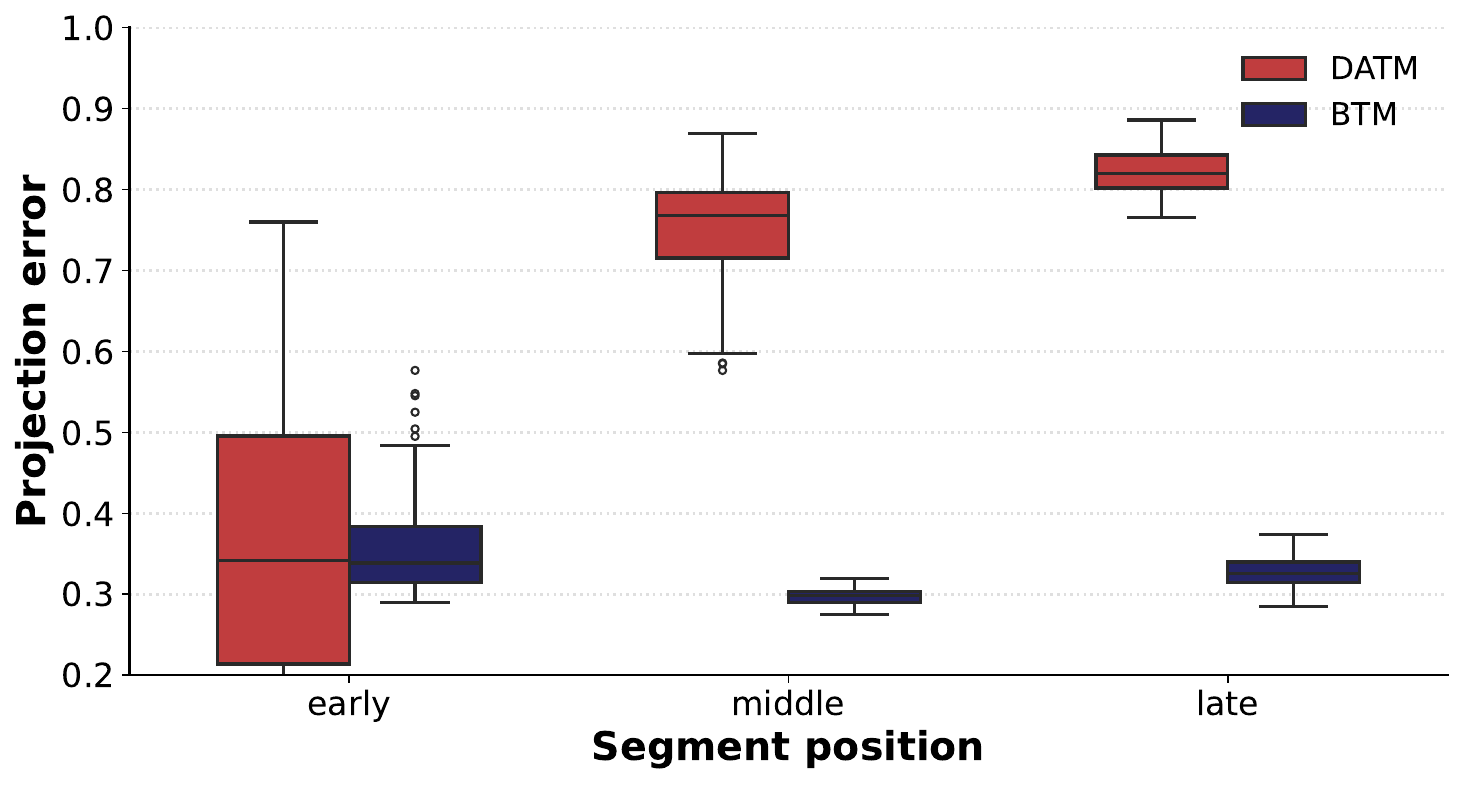}
        \\[0.25em]

        \multicolumn{2}{c}{\textbf{Oxford (1.7\%)}} \\[-0.15em]
        \includegraphics[width=0.40\textwidth]
        {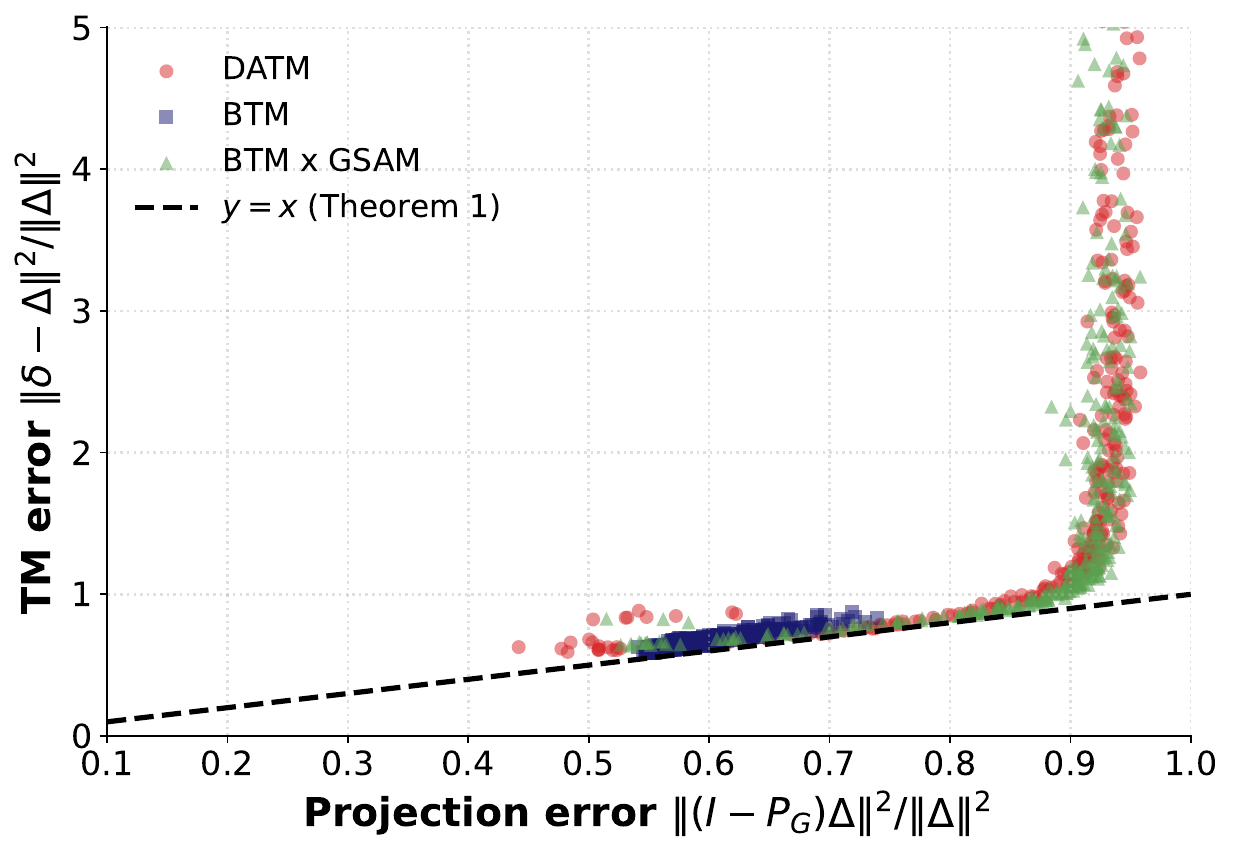}
        &
        \includegraphics[width=0.40\textwidth]
        {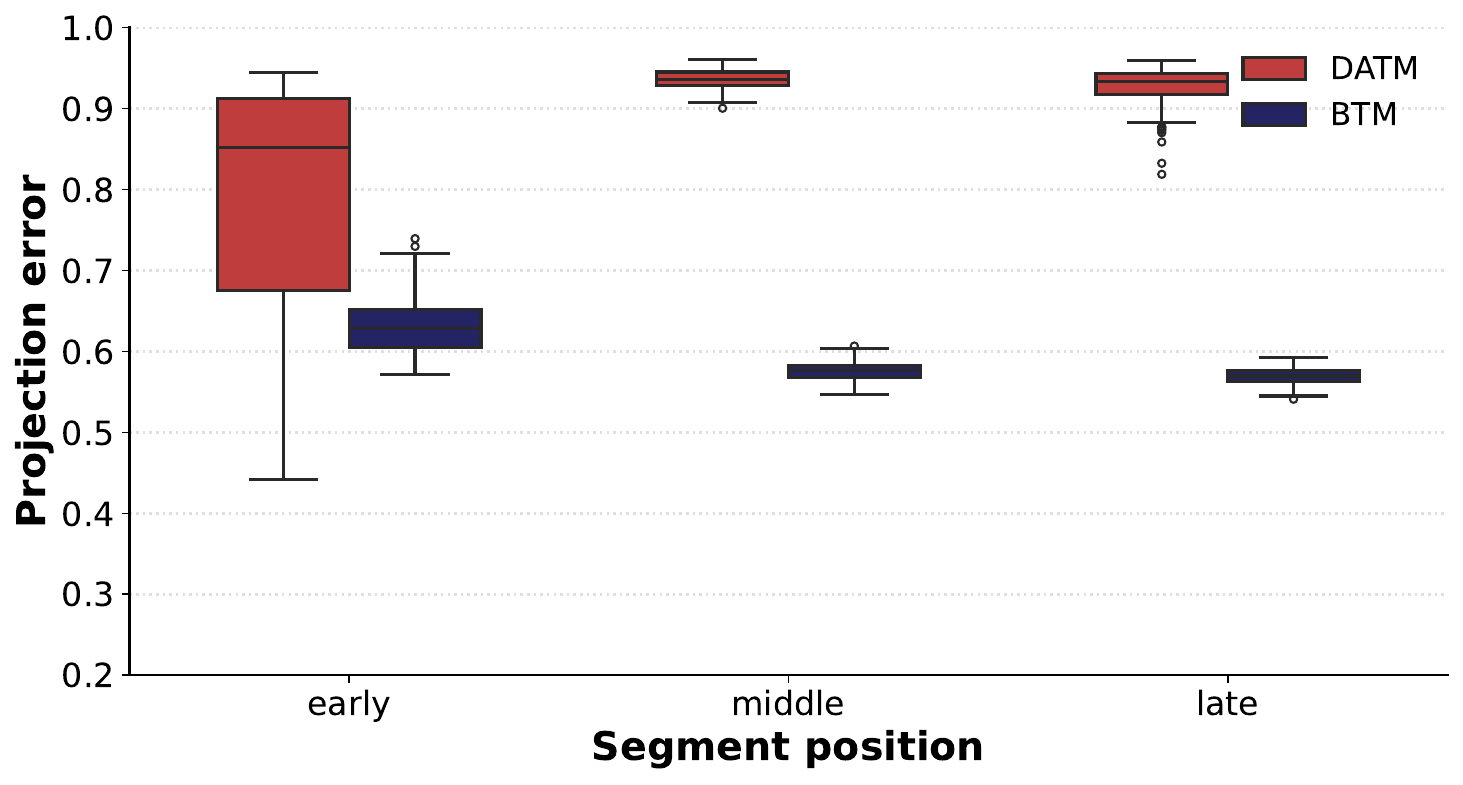}
        \\[0.25em]

        \multicolumn{2}{c}{\textbf{Birmingham (0.8\%)}} \\[-0.15em]
        \includegraphics[width=0.40\textwidth]
        {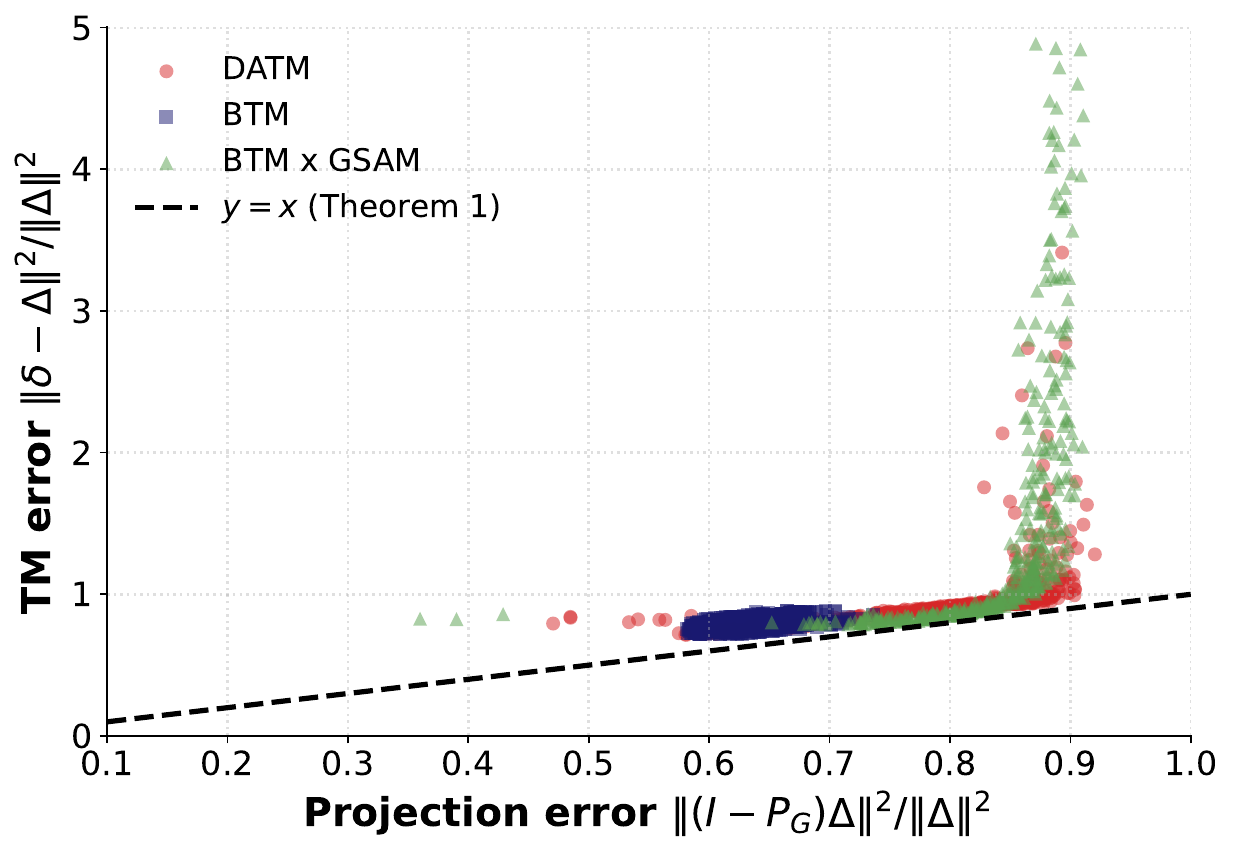}
        &
        \includegraphics[width=0.40\textwidth]
        {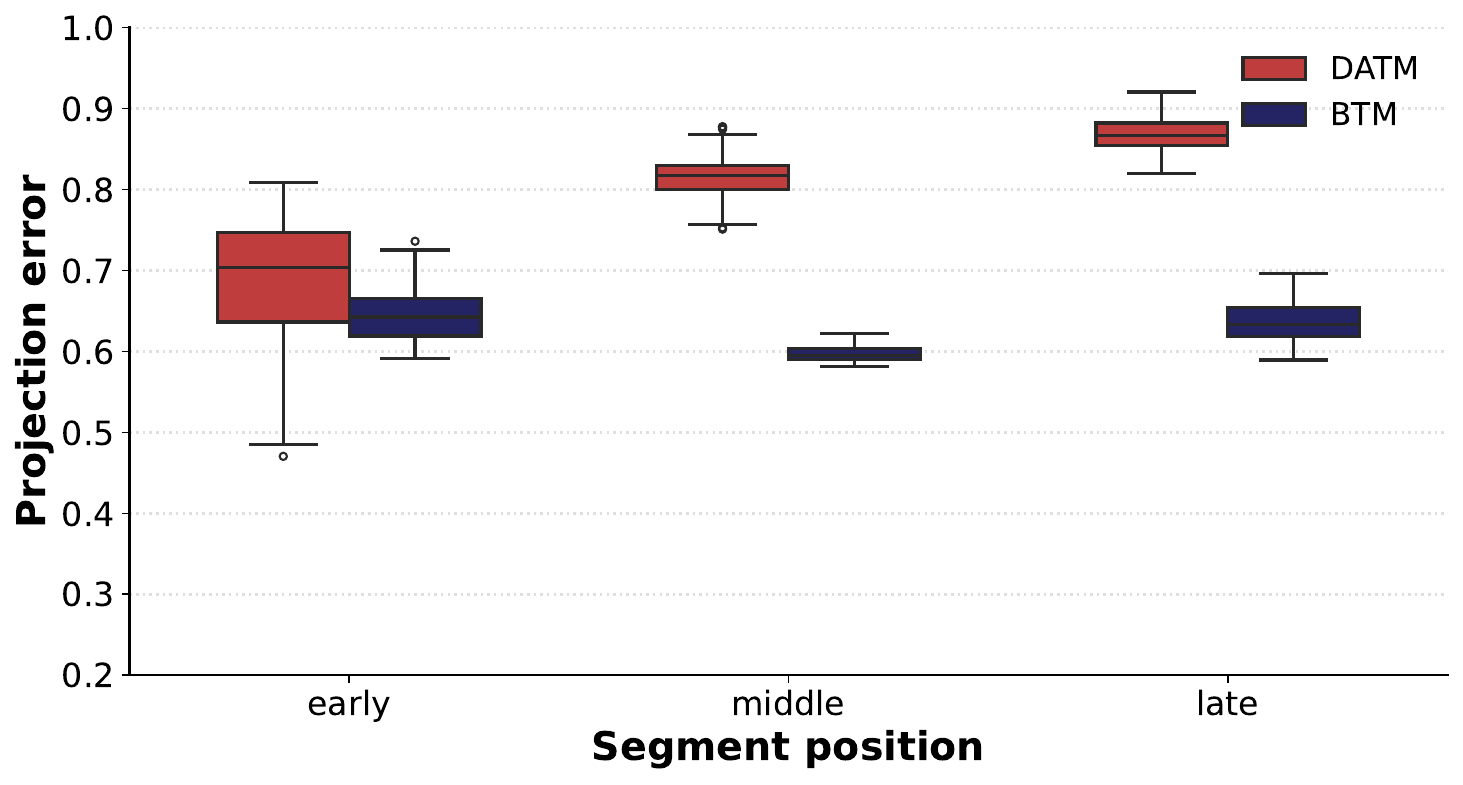}
        \\[0.5em]

        \textbf{Realised error vs.\ projection residual}
        &
        \textbf{Residual by trajectory position}
        \\[0.25em]
    \end{tabular}

    \caption{
        \revadd{
        \textbf{Projection residuals of learned synthetic datasets at
        \emph{ipc}=500.}
        The left panels show normalised realised trajectory-matching error
        against the normalised projection residual from
        Theorem~\ref{thm:projection-lb} for DATM, BTM, and BTM
        $\times$ GSAM. The dashed line denotes the theoretical lower bound
        $y=x$. The vertical axis is truncated at $5$ for readability, while
        Table~\ref{tab:learned_projection_residuals} reports the full
        distributions. The right panels show projection residuals across
        early, middle, and late trajectory segments for DATM and BTM.
        BTM $\times$ GSAM retains the learned BTM synthetic dataset while
        replacing Bézier supervision with raw GSAM trajectory supervision.
        Each condition contains 500 sampled segments per dataset.
        }
    }
    \label{fig:learned_projection_residuals}
\end{figure*}

\begin{table}[t]
    \centering
    \caption{
        \revadd{Projection residuals and realised matching errors of learned synthetic datasets at \emph{ipc}=500, using 500 segments per condition and dataset. Projection residuals are reported as mean $\pm$ standard deviation. Realised errors are reported as median (IQR) because the raw-trajectory conditions exhibit heavy upper tails.}
    }
    
    \small
    \setlength{\tabcolsep}{3pt}
    
    \begin{tabular}{llcc}
        \toprule
        \textbf{Dataset} & \textbf{Condition} & \textbf{Projection residual}
        & \textbf{Realised TM error} \\
        \midrule

        Portsmouth
        & DATM
        & $0.651 \pm 0.223$
        & $2.583\;(0.856\text{--}11.188)$ \\
        & BTM
        & $\mathbf{0.327 \pm 0.043}$
        & $\mathbf{0.403\;(0.377\text{--}0.463)}$ \\
        & BTM $\times$ GSAM
        & $0.539 \pm 0.101$
        & $3.504\;(1.175\text{--}13.529)$ \\
        \addlinespace

        Oxford
        & DATM
        & $0.885 \pm 0.109$
        & $3.672\;(1.282\text{--}16.753)$ \\
        & BTM
        & $\mathbf{0.593 \pm 0.037}$
        & $\mathbf{0.648\;(0.626\text{--}0.691)}$ \\
        & BTM $\times$ GSAM
        & $0.887 \pm 0.087$
        & $2.466\;(1.105\text{--}9.739)$ \\
        \addlinespace

        Birmingham
        & DATM
        & $0.793 \pm 0.086$
        & $0.923\;(0.861\text{--}0.954)$ \\
        & BTM
        & $\mathbf{0.627 \pm 0.032}$
        & $\mathbf{0.782\;(0.753\text{--}0.823)}$ \\
        & BTM $\times$ GSAM
        & $0.841 \pm 0.065$
        & $1.148\;(0.876\text{--}2.807)$ \\
        \bottomrule
    \end{tabular}

    \label{tab:learned_projection_residuals}
\end{table}

\revadd{
Figure~\ref{fig:learned_projection_residuals} directly evaluates
Theorem~\ref{thm:projection-lb} using the gradient spans induced by the
final learned synthetic datasets. All sampled segments lie on or above the
theoretical $y=x$ boundary, as required by the lower bound. The projection
residuals are also substantial rather than numerically negligible.
As shown in Table~\ref{tab:learned_projection_residuals}, BTM reduces the
mean residual from $0.651$ to $0.327$ on Portsmouth, from $0.885$ to
$0.593$ on Oxford, and from $0.793$ to $0.627$ on Birmingham. These results
show that the representability bottleneck identified in
Theorem~\ref{thm:projection-lb} is active for the reachable spans of learned
synthetic datasets in practice.
}

\revadd{
The BTM $\times$ GSAM control tests whether this reduction can be explained
solely by the learned synthetic dataset. Keeping the final BTM synthetic
set fixed while replacing Bézier supervision with raw GSAM supervision
increases the mean projection residual from $0.327$ to $0.539$ on
Portsmouth, from $0.593$ to $0.887$ on Oxford, and from $0.627$ to $0.841$
on Birmingham. The same synthetic set is therefore substantially better
aligned with structured Bézier supervision than with the corresponding raw
trajectory supervision. This supports the interpretation that the
representability bottleneck depends on the structure of the supervision
signal and not only on the synthetic samples.
}

\revadd{
The position-wise results provide a complementary view. BTM remains
comparatively stable across early, middle, and late segments, with median
residuals of $0.338/0.298/0.326$ on Portsmouth,
$0.629/0.576/0.570$ on Oxford, and $0.642/0.595/0.633$ on Birmingham.
DATM shows substantially larger residuals beyond the early trajectory
region, with middle and late values of $0.768/0.820$, $0.936/0.933$, and
$0.817/0.867$, respectively. Together with the teacher-side spectral
analysis in Figure~\ref{fig:velocity-span-grid}, these results provide
complementary evidence from both sides of the representability bottleneck.
Raw trajectory supervision is spectrally broader, while a substantial
fraction of its displacement energy lies outside the gradient span
reachable from the learned synthetic dataset. BTM concentrates the
teacher-side supervision and is accompanied by smaller student-side
projection residuals.
}

\revadd{
The projection residual is a lower bound and should not be interpreted as
accounting for the full realised matching error. As discussed in
Remark~\ref{rem:best-in-span}, the realised student displacement need not
attain the best approximation available within $\mathcal{G}_k$.
Trajectory-matching error may therefore substantially exceed the projection
residual. These experiments isolate the representability component
identified by Theorem~\ref{thm:projection-lb}, rather than claiming that it
alone determines the realised matching error.
}

\revadd{
\subsection{Empirical Magnitudes in Theorem~\ref{thm:bez-sur}}
\label{app:thm4-magnitudes}

We measure the two geometric quantities appearing in Theorem~\ref{thm:bez-sur} over the same 50 teacher--surrogate trajectory pairs used for Portsmouth, eICU, and MIMIC-III (IHM). Because parameter-space distances are architecture- and scale dependent, we normalise both terms by the endpoint chord length
\[
L
=
\|\boldsymbol{\theta}_T-\boldsymbol{\theta}_0\|_2,
\qquad
\bar{\kappa}
:=
\frac{\kappa/8}{L},
\qquad
\bar{D}
:=
\frac{D}{L}.
\]
By Theorem~\ref{thm:bez-sur}(i), \(\kappa/8\) is exactly the maximum Bézier-to-chord deviation. The teacher term \(D\) is also evaluated exactly from the stored trajectory checkpoints. On each linear segment of \(\gamma\), the difference \(\gamma(t)-c(t)\) is affine in \(t\), and its Euclidean norm is therefore convex, so its maximum over the segment is attained at an endpoint.

\begin{table}[t]
\centering
\caption{
\revadd{Normalised geometric quantities appearing in Theorem~\ref{thm:bez-sur}, measured over 50 teacher--surrogate trajectory pairs per dataset. Entries are medians with interquartile ranges, except for the ``max'' column, which reports the largest value of \(\bar{\kappa}\) across the 50 trajectory pairs. \(\bar{\kappa}\) is the maximum Bézier-to-chord deviation and \(\bar D\) is the maximum teacher-to-chord deviation, each normalised by the endpoint chord length. The final column reports \(\bar{\kappa}/(\bar{\kappa}+\bar D)\), i.e.\ the surrogate-dependent term as a share of the two geometric terms.}
}
\label{tab:magnitudes}
\small
\setlength{\tabcolsep}{3pt}
\begin{tabular}{lccccc}
\toprule
\textbf{Dataset}
& \textbf{Architecture}
& $\mathbf{\bar{\kappa}}$
& \textbf{max} $\mathbf{\bar{\kappa}}$
& $\mathbf{\bar D}$
& $\mathbf{\bar{\kappa}}$ \textbf{share} (\%) \\
\midrule
Portsmouth
& MLP
& $0.097$ [$0.093$, $0.102$]
& $0.109$
& $0.494$ [$0.488$, $0.501$]
& $16.5$ \\

eICU
& MLP
& $0.051$ [$0.050$, $0.052$]
& $0.055$
& $0.360$ [$0.359$, $0.361$]
& $12.5$ \\

MIMIC-III (IHM)
& TCN
& $0.036$ [$0.033$, $0.038$]
& $0.048$
& $0.288$ [$0.285$, $0.293$]
& $11.1$ \\
\bottomrule
\end{tabular}
\end{table}

Table~\ref{tab:magnitudes} shows that the surrogate-dependent geometric
term is modest relative to the teacher-dependent term. The median maximum
Bézier-to-chord deviation is \(9.7\%\) of the endpoint displacement on
Portsmouth, \(5.1\%\) on eICU, and \(3.6\%\) on MIMIC-III, with maxima of
\(10.9\%\), \(5.5\%\), and \(4.8\%\), respectively. In each case, the
teacher-to-chord term is substantially larger, and the surrogate-dependent
term accounts for only \(11\)--\(17\%\) of the sum
\(\bar{\kappa}+\bar D\). Thus, the learned Bézier surrogates introduce
limited curvature relative to the common endpoint chord, while most of the
chord-relative deviation is already present in the raw teacher trajectory.

We do not interpret \(L_f(\kappa/8+D)\) as a tight numerical certificate.
The teacher-dependent term \(D\) is non-negligible, and obtaining a useful
uniform Lipschitz constant \(L_f\) over paths from random initialisation to
the trained solution is not practical. We therefore report the directly
measurable geometric quantities rather than estimating \(L_f\). The purpose
of this analysis is to quantify the decomposition in
Theorem~\ref{thm:bez-sur}, rather than to demonstrate pointwise
reconstruction of the SGD trajectory. The surrogate introduces a small,
explicitly controlled excursion from the endpoint chord, whereas the larger
teacher-dependent term is already present in the raw trajectory.
}

\section{Subgroup Analysis}
\label{sec:subgroups}

\begin{table*}[t]
\centering
\caption{
\revadd{Subgroup performance on the Portsmouth cohort. Results are reported as
mean $\pm$ standard deviation across evaluation seeds. Prevalence is reported
for each subgroup. BTM is compared with DATM and full-data training across
age, sex, and broad ethnicity groups.}
}
\label{tab_subgroup_internal}

\small
\setlength{\tabcolsep}{3pt}

\begin{tabular}{lllrll}
\toprule
\textbf{Category} & \textbf{Subcategory} & \textbf{Method} & \textbf{Prev.} & \textbf{AUROC} & \textbf{AUPRC} \\
\midrule

\multirow{12}{*}{Age}
    & \multirow{3}{*}{<40}
        & BTM  & \multirow{3}{*}{0.023} & $0.831_{\pm 0.011}$ & $0.498_{\pm 0.009}$ \\
    &   & DATM &                        & $0.816_{\pm 0.010}$ & $0.497_{\pm 0.011}$ \\
    &   & Full &                        & $0.825_{\pm 0.007}$ & $0.508_{\pm 0.006}$ \\
    \cmidrule(lr){2-6}

    & \multirow{3}{*}{40--59}
        & BTM  & \multirow{3}{*}{0.065} & $0.924_{\pm 0.005}$ & $0.700_{\pm 0.012}$ \\
    &   & DATM &                        & $0.907_{\pm 0.005}$ & $0.685_{\pm 0.009}$ \\
    &   & Full &                        & $0.926_{\pm 0.002}$ & $0.706_{\pm 0.005}$ \\
    \cmidrule(lr){2-6}

    & \multirow{3}{*}{60--79}
        & BTM  & \multirow{3}{*}{0.065} & $0.930_{\pm 0.003}$ & $0.665_{\pm 0.012}$ \\
    &   & DATM &                        & $0.917_{\pm 0.004}$ & $0.660_{\pm 0.010}$ \\
    &   & Full &                        & $0.934_{\pm 0.002}$ & $0.669_{\pm 0.005}$ \\
    \cmidrule(lr){2-6}

    & \multirow{3}{*}{80+}
        & BTM  & \multirow{3}{*}{0.055} & $0.857_{\pm 0.004}$ & $0.484_{\pm 0.012}$ \\
    &   & DATM &                        & $0.823_{\pm 0.008}$ & $0.473_{\pm 0.007}$ \\
    &   & Full &                        & $0.864_{\pm 0.003}$ & $0.492_{\pm 0.008}$ \\
\midrule

\multirow{6}{*}{Sex}
    & \multirow{3}{*}{Female}
        & BTM  & \multirow{3}{*}{0.044} & $0.893_{\pm 0.003}$ & $0.598_{\pm 0.012}$ \\
    &   & DATM &                        & $0.868_{\pm 0.004}$ & $0.587_{\pm 0.006}$ \\
    &   & Full &                        & $0.894_{\pm 0.002}$ & $0.607_{\pm 0.006}$ \\
    \cmidrule(lr){2-6}

    & \multirow{3}{*}{Male}
        & BTM  & \multirow{3}{*}{0.066} & $0.907_{\pm 0.003}$ & $0.603_{\pm 0.011}$ \\
    &   & DATM &                        & $0.891_{\pm 0.005}$ & $0.597_{\pm 0.006}$ \\
    &   & Full &                        & $0.913_{\pm 0.001}$ & $0.618_{\pm 0.006}$ \\
\midrule

\multirow{9}{*}{Ethnicity}
    & \multirow{3}{*}{White}
        & BTM  & \multirow{3}{*}{0.052} & $0.901_{\pm 0.002}$ & $0.586_{\pm 0.007}$ \\
    &   & DATM &                        & $0.886_{\pm 0.005}$ & $0.578_{\pm 0.007}$ \\
    &   & Full &                        & $0.903_{\pm 0.001}$ & $0.600_{\pm 0.005}$ \\
    \cmidrule(lr){2-6}

    & \multirow{3}{*}{Non-White}
        & BTM  & \multirow{3}{*}{0.057} & $0.930_{\pm 0.014}$ & $0.696_{\pm 0.062}$ \\
    &   & DATM &                        & $0.931_{\pm 0.018}$ & $0.675_{\pm 0.041}$ \\
    &   & Full &                        & $0.939_{\pm 0.006}$ & $0.689_{\pm 0.029}$ \\
    \cmidrule(lr){2-6}

    & \multirow{3}{*}{Unknown}
        & BTM  & \multirow{3}{*}{0.058} & $0.917_{\pm 0.006}$ & $0.648_{\pm 0.012}$ \\
    &   & DATM &                        & $0.889_{\pm 0.005}$ & $0.642_{\pm 0.011}$ \\
    &   & Full &                        & $0.914_{\pm 0.002}$ & $0.656_{\pm 0.007}$ \\
\midrule

\multirow{3}{*}{Overall}
    &  & BTM  & \multirow{3}{*}{0.053} & $0.901_{\pm 0.001}$ & $0.609_{\pm 0.001}$ \\
    &  & DATM &                        & $0.890_{\pm 0.003}$ & $0.580_{\pm 0.001}$ \\
    &  & Full &                        & $0.906_{\pm 0.002}$ & $0.610_{\pm 0.004}$ \\

\bottomrule
\end{tabular}
\end{table*}

\revadd{
A potential concern with structured low-rank supervision is that aggregate performance may conceal degradation for particular patient groups. To examine this, we evaluate BTM on the Portsmouth NHS cohort across age, sex, and broad ethnicity subgroups. We report subgroup-specific AUROC and AUPRC to assess predictive discrimination within each group, using full-data training as the primary reference. DATM is included as a representative trajectory-matching baseline. We additionally report centred heterogeneity RMSE, which provides a complementary assessment of whether the relative differences in risk between subgroups are retained after condensation.
}

\revadd{
Table~\ref{tab_subgroup_internal} shows no systematic degradation of subgroup-level predictive performance under BTM. Across the examined age, sex, and ethnicity groups, BTM remains close to the corresponding full-data AUROC and retains substantial AUPRC despite condensation. This includes the lowest-prevalence age group ($<40$, prevalence $2.3\%$), for which BTM attains an AUROC of $0.831$ compared with $0.825$ under full-data training, although its AUPRC is lower (\revadd{$0.498$} versus $0.508$). Some loss relative to full-data performance is expected after condensation, but there is no subgroup for which BTM exhibits a pronounced collapse in either discrimination metric.
}

\begin{table}[t]
\centering
\caption{
\revadd{Centred heterogeneity RMSE on the Portsmouth cohort. The metric measures disagreement between predicted and observed relative subgroup risk deviations after removing the overall population risk level. Values are mean $\pm$ standard deviation across seeds. Lower values indicate closer preservation of relative subgroup risk structure.}
}
\label{tab:subgroup_heterogeneity}

\small
\setlength{\tabcolsep}{3pt}

\begin{tabular}{lcccccc}
\toprule
\textbf{Subgroup variable}
& \textbf{Full}
& \textbf{BTM}
& \textbf{DATM}
& \textbf{DC}
& \textbf{M3D}
& \textbf{Random} \\
\midrule
Age band
& $0.0049_{\pm0.0003}$
& $0.0058_{\pm0.0003}$
& $0.0047_{\pm0.0004}$
& $0.0541_{\pm0.0018}$
& $0.0466_{\pm0.0018}$
& $0.0145_{\pm0.0039}$ \\

Broad ethnicity
& $0.0032_{\pm0.0001}$
& $0.0022_{\pm0.0005}$
& $0.0027_{\pm0.0004}$
& $0.0108_{\pm0.0004}$
& $0.0092_{\pm0.0004}$
& $0.0031_{\pm0.0011}$ \\

Sex
& $0.0022_{\pm0.0005}$
& $0.0021_{\pm0.0004}$
& $0.0008_{\pm0.0005}$
& $0.0058_{\pm0.0014}$
& $0.0034_{\pm0.0021}$
& $0.0109_{\pm0.0064}$ \\
\bottomrule
\end{tabular}
\end{table}

\revadd{
The centred heterogeneity analysis provides a complementary view of subgroup preservation. Whereas subgroup AUROC and AUPRC measure discrimination within each group, centred heterogeneity RMSE measures disagreement between predicted and observed \emph{relative} subgroup risk deviations after removing the overall population risk level. BTM remains close to the full-data reference scale for all three subgroup variables, with RMSE values of $0.0058$ for age, $0.0022$ for broad ethnicity, and $0.0021$ for sex. This indicates that condensation does not substantially distort the relative risk structure across the examined groups.

The comparison with other condensation methods provides additional context. DATM is similarly close to the full-data scale. In contrast, DC and M3D exhibit substantially larger deviations, particularly for age, while the random-subset results are more variable.

The full-data model is used as a reference rather than as a quantity that must numerically minimise centred heterogeneity RMSE. The metric evaluates agreement with observed relative subgroup risks and is distinct from the predictive objective used to train the model. Consequently, a condensed model can obtain a lower heterogeneity RMSE than the full-data model without implying superior overall predictive performance.
}

\section{Additional Experiments}
\label{sec:add_exp}

\subsection{Inner-Loop Student Steps}
\label{app:inner_loop}

In standard trajectory matching, the student optimisation horizon is tied
to the discrete teacher trajectory segment. With Bézier surrogates, however,
the continuous parameterisation $t\in[0,1]$ decouples the teacher segment
length $(t_{\mathrm{start}},t_{\mathrm{end}})$ from the number of student
gradient updates. The appropriate number of inner-loop steps must therefore
be determined empirically. Figure~\ref{fig:inner_loop} ablates $N$ at
200 \emph{ipc}. BTM sustains strong AUROC and AUPRC at $N=30$, showing that
the structured surrogate provides effective supervision without requiring a
long student optimisation horizon.

\begin{figure}[t]
\centering
    \includegraphics[width=1\linewidth]{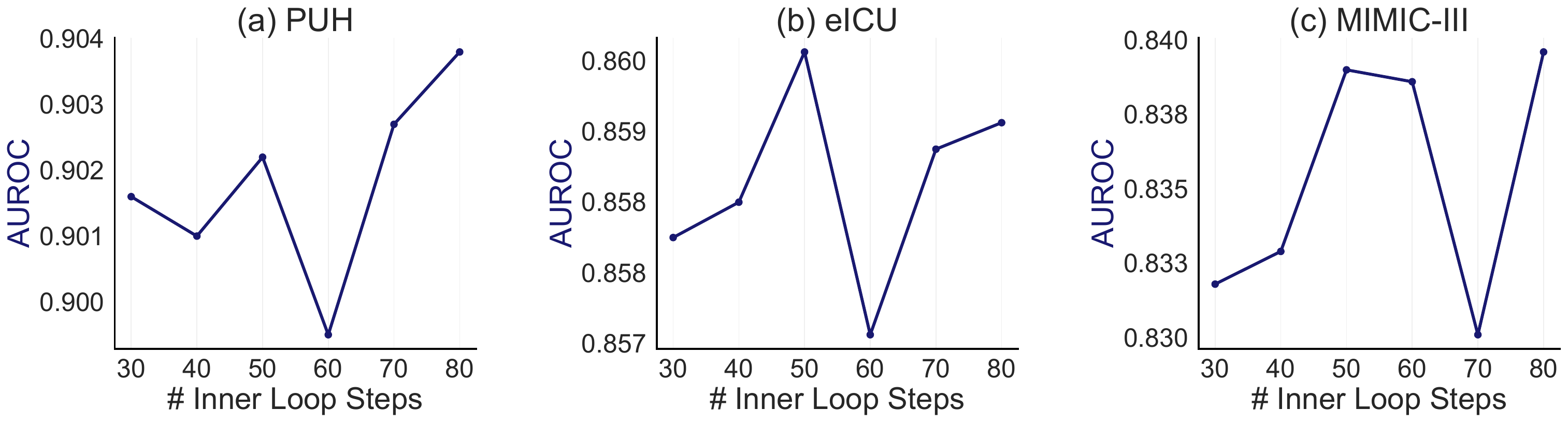}
    \caption{
    \revadd{\textbf{Sensitivity of BTM to the number of inner-loop student steps ($N$) at \emph{ipc}=200.} AUROC is shown for Portsmouth (PUH), eICU, and MIMIC-III. BTM remains competitive with relatively few inner-loop updates, indicating that Bézier surrogate supervision permits effective trajectory matching with a modest student optimisation horizon. Similar trends are observed for AUPRC.}
    }
\label{fig:inner_loop}
\end{figure}

\subsection{Initialisation Strategy for the Synthetic Dataset}
\label{app:init_strategy}

Table~\ref{tab:init} examines the impact of synthetic-data initialisation on
condensation performance in eICU. We compare \emph{real} initialisation,
which seeds synthetic inputs from real training samples, with \emph{random}
initialisation, which samples synthetic inputs from class-conditional Gaussian
distributions. The two strategies yield broadly comparable performance,
although real initialisation provides a modest advantage at 500 \emph{ipc}.
These results indicate that BTM is not strongly dependent on direct seeding
from individual real examples.


\revadd{This comparison is also relevant to how directly the condensation procedure
depends on observed examples at initialisation. Real initialisation may place
synthetic inputs closer to the empirical data manifold and thereby facilitate
optimisation, whereas random initialisation reduces direct dependence on
individual training examples at this initial stage. This distinction should
not, however, be interpreted as a formal privacy guarantee. Such guarantees
would require an explicit privacy mechanism, such as differential privacy,
applied to the condensation pipeline.}

\begin{table}[t]
\centering
\caption{
Initialisation strategy comparison on eICU. Random initialisation remains
competitive while
\revadd{reducing direct dependence on individual examples at initialisation.}
}
\label{tab:init}

\small
\setlength{\tabcolsep}{3pt}

\begin{tabular}{ll|cccc}
\toprule
& \multicolumn{1}{c}{\multirow{2}{*}{\textbf{Init.}}}
& \multicolumn{4}{c}{\textbf{\emph{ipc}}} \\
\cmidrule(lr){3-6}
& \multicolumn{1}{c}{}
& \textbf{50} & \textbf{100} & \textbf{200} & \textbf{500} \\
\midrule
\multirow{2}{*}{\textbf{AUROC}}
    & Real
    & $0.854_{\pm0.002}$
    & $0.859_{\pm0.001}$
    & $0.861_{\pm0.001}$
    & $0.874_{\pm0.002}$ \\
    & Random
    & $0.852_{\pm0.009}$
    & $0.858_{\pm0.004}$
    & $0.853_{\pm0.004}$
    & $0.862_{\pm0.002}$ \\
\cmidrule(lr){2-6}
\multirow{2}{*}{\textbf{AUPRC}}
    & Real
    & $0.479_{\pm0.002}$
    & $0.486_{\pm0.001}$
    & $0.476_{\pm0.001}$
    & $0.506_{\pm0.001}$ \\
    & Random
    & $0.463_{\pm0.015}$
    & $0.474_{\pm0.007}$
    & $0.475_{\pm0.008}$
    & $0.499_{\pm0.004}$ \\
\bottomrule
\end{tabular}
\end{table}

\subsection{Sensitivity of Bézier-Specific Choices}
\label{app:bezier_sensitivity}

\revadd{
We examine the main Bézier-specific choices on Portsmouth with the remaining
configuration fixed. These include the teacher segment length, the number of
surrogate trajectories, the control-point optimisation budget, and the number
of Monte Carlo samples used for the path objective.
}

\subsubsection{Teacher Segment Length}
\label{app:segment_length}

Table~\ref{tab:segment_length} examines the teacher segment length
$\Delta t=t_e-t_s$. Short segments provide limited displacement, while long
segments demand larger displacements within a fixed inner-loop budget. Among
the tested values, $\Delta t=0.2$ gives the strongest AUPRC at each budget
and is used in the main experiments.

\begin{table}[t]
\centering
\caption{
Sensitivity of BTM to continuous teacher segment length ($\Delta t$) on
Portsmouth. Results are reported as AUPRC across three synthetic-data budgets.
Among the tested values, $\Delta t=0.2$ gives the strongest performance at
each budget and is marked by
$\dagger$.
}
\label{tab:segment_length}

\small
\setlength{\tabcolsep}{3pt}

\begin{tabular}{c c c c}
\toprule
\multicolumn{1}{c}{
\multirow{2}{*}{\textbf{Segment Length ($\Delta t$)}}}
& \multicolumn{3}{c}{\textbf{\emph{ipc}}} \\
\cmidrule(lr){2-4}
& \textbf{100} & \textbf{200} & \textbf{500} \\
\midrule
$0.1$
& $0.558_{\pm0.012}$
& $0.582_{\pm0.007}$
& $0.589_{\pm0.007}$ \\

$0.2^\dagger$
& $\mathbf{0.572}_{\pm0.001}$
& $\mathbf{0.596}_{\pm0.001}$
& $\mathbf{0.609}_{\pm0.001}$ \\

$0.3$
& $0.564_{\pm0.010}$
& $0.585_{\pm0.008}$
& $0.593_{\pm0.005}$ \\

$0.4$
& $0.556_{\pm0.011}$
& $0.579_{\pm0.009}$
& $0.581_{\pm0.004}$ \\
\bottomrule
\end{tabular}
\end{table}

\subsubsection{Number of Bézier Surrogates}
\label{app:num_bezier}

\revadd{
We vary the number of Bézier surrogates from 1 to 50 at \emph{ipc}=50 and
\emph{ipc}=500, keeping the remaining configuration fixed. The 50-surrogate
setting matches the main-experiment supervision-pool size.
}

\begin{table}[t]
\centering
\caption{
\revadd{
Sensitivity of BTM to the number of Bézier surrogate trajectories on the
Portsmouth cohort. Results are reported for low (\emph{ipc}=50) and high
(\emph{ipc}=500) synthetic-data budgets as mean $\pm$ standard deviation.
The 50-surrogate configuration used in the main experiments is marked by
$\dagger$.
}
}
\label{tab:bezier_count_sensitivity}

\small
\setlength{\tabcolsep}{3pt}

\begin{tabular}{c cc cc}
\toprule
& \multicolumn{2}{c}{\textbf{\emph{ipc}=50}}
& \multicolumn{2}{c}{\textbf{\emph{ipc}=500}} \\
\cmidrule(lr){2-3}
\cmidrule(lr){4-5}
\textbf{\# Surrogates}
& \textbf{AUROC} & \textbf{AUPRC}
& \textbf{AUROC} & \textbf{AUPRC} \\
\midrule

1
& $0.868_{\pm0.002}$
& $0.525_{\pm0.004}$
& $0.892_{\pm0.001}$
& $0.577_{\pm0.001}$ \\

5
& $0.872_{\pm0.004}$
& $0.555_{\pm0.003}$
& $0.899_{\pm0.001}$
& $0.595_{\pm0.002}$ \\

10
& $\mathbf{0.878_{\pm0.002}}$
& $0.559_{\pm0.002}$
& $0.901_{\pm0.001}$
& $0.605_{\pm0.001}$ \\

20
& $0.872_{\pm0.003}$
& $\mathbf{0.560_{\pm0.009}}$
& $\mathbf{0.902_{\pm0.001}}$
& $0.607_{\pm0.002}$ \\

$50^\dagger$
& $0.876_{\pm0.002}$
& $0.557_{\pm0.001}$
& $0.901_{\pm0.001}$
& $\mathbf{0.609_{\pm0.001}}$ \\

\bottomrule
\end{tabular}
\end{table}

\revadd{
Table~\ref{tab:bezier_count_sensitivity} shows saturation once a modest
surrogate pool is available. At \emph{ipc}=50, AUROC peaks at 10 surrogates
and AUPRC at 20, with only small differences among 10--50. At \emph{ipc}=500,
AUPRC ranges from $0.605$ to $0.609$ from 10 surrogates onwards. We
nevertheless use 50 surrogates in the main experiments to match the number of
teacher trajectories available to other trajectory-matching methods, making
this a controlled-comparison choice rather than a performance-optimal
hyperparameter.
}

\subsubsection{Control-Point Optimisation Budget}
\label{app:control_point_budget}

\revadd{
We vary the number of control-point optimisation updates with $N_{\mathrm{MC}}=5$
fixed. This tests how much surrogate optimisation is needed for effective
condensation supervision.
}

\revadd{
Table~\ref{tab:control_point_steps_sensitivity} shows that very short
optimisation runs, corresponding to near-linear trajectories, give weaker
performance, while performance improves into the 200--300 update range. At
\emph{ipc}=50, AUPRC rises from $0.518$ after 5 updates to $0.560$ after 200
updates, and the 300-update setting gives a similar AUPRC of $0.557$ with the
highest AUROC of $0.876$. At \emph{ipc}=500, both metrics peak at 300 updates,
reaching $0.901$ and $0.609$. Increasing the budget to 500 updates reduces
AUPRC to $0.545$ and $0.599$ at \emph{ipc}=50 and 500. Thus additional
optimisation does not necessarily improve supervision, and the selected
300-update configuration lies in the saturated region while avoiding extra
cost.
}

\subsubsection{Monte Carlo Samples}
\label{app:mc_sensitivity}

\revadd{
Finally, we vary $N_{\mathrm{MC}}$, the number of Monte Carlo samples for the
path objective, with 300 control-point updates fixed. Table~\ref{tab:control_point_mc_sensitivity}
shows limited sensitivity. At \emph{ipc}=50, $N_{\mathrm{MC}}\in\{3,5,10\}$
are similar, while the single-sample setting has higher AUPRC variability. At
\emph{ipc}=500, $N_{\mathrm{MC}}=5$ is strongest, with 10 samples providing no
further gain. We therefore use $N_{\mathrm{MC}}=5$.

Overall, the four ablations confirm that BTM is robust across moderate
Bézier-specific settings. Performance saturates beyond roughly 10--20
surrogates, 200--300 control-point updates, and 5 Monte Carlo samples. Segment
length shows a clearer optimum at $\Delta t=0.2$.
}

\begin{table}[!htbp]
\centering
\caption{
\revadd{
Sensitivity of BTM to the control-point optimisation budget on the Portsmouth cohort. The number of Monte Carlo samples is fixed at $N_{\mathrm{MC}}=5$. Results are reported for low (\emph{ipc}=50) and high (\emph{ipc}=500) synthetic-data budgets as mean $\pm$ standard deviation. The selected 300-update Portsmouth configuration is marked by $\dagger$.
}
}
\label{tab:control_point_steps_sensitivity}

\small
\setlength{\tabcolsep}{3pt}

\begin{tabular}{c cc cc}
\toprule
& \multicolumn{2}{c}{\textbf{\emph{ipc}=50}}
& \multicolumn{2}{c}{\textbf{\emph{ipc}=500}} \\
\cmidrule(lr){2-3}
\cmidrule(lr){4-5}
\textbf{Maximum Optimisation steps}
& \textbf{AUROC} & \textbf{AUPRC}
& \textbf{AUROC} & \textbf{AUPRC} \\
\midrule

5
& $0.870_{\pm0.001}$
& $0.518_{\pm0.002}$
& $0.895_{\pm0.002}$
& $0.576_{\pm0.001}$ \\

25
& $0.871_{\pm0.002}$
& $0.520_{\pm0.003}$
& $0.896_{\pm0.001}$
& $0.575_{\pm0.001}$ \\

50
& $0.869_{\pm0.003}$
& $0.536_{\pm0.002}$
& $0.892_{\pm0.001}$
& $0.584_{\pm0.002}$ \\

100
& $0.871_{\pm0.004}$
& $0.548_{\pm0.013}$
& $0.898_{\pm0.002}$
& $0.588_{\pm0.002}$ \\

200
& $0.873_{\pm0.002}$
& $\mathbf{0.560_{\pm0.005}}$
& $0.900_{\pm0.001}$
& $0.592_{\pm0.002}$ \\

300$^\dagger$
& $\mathbf{0.876_{\pm0.002}}$
& $0.557_{\pm0.001}$
& $\mathbf{0.901_{\pm0.001}}$
& $\mathbf{0.609_{\pm0.001}}$ \\

500
& $0.864_{\pm0.004}$
& $0.545_{\pm0.004}$
& $0.896_{\pm0.001}$
& $0.599_{\pm0.003}$ \\

\bottomrule
\end{tabular}
\end{table}

\begin{table}[!htbp]
\centering
\caption{
\revadd{
Sensitivity of BTM to the number of Monte Carlo samples used for
control-point optimisation on the Portsmouth cohort. The optimisation budget
is fixed at 300 updates. Results are reported for low (\emph{ipc}=50) and
high (\emph{ipc}=500) synthetic-data budgets as mean $\pm$ standard
deviation. The selected configuration is marked by $\dagger$.
}
}
\label{tab:control_point_mc_sensitivity}

\small
\setlength{\tabcolsep}{3pt}

\begin{tabular}{c cc cc}
\toprule
& \multicolumn{2}{c}{\textbf{\emph{ipc}=50}}
& \multicolumn{2}{c}{\textbf{\emph{ipc}=500}} \\
\cmidrule(lr){2-3}
\cmidrule(lr){4-5}
\textbf{$N_{\mathrm{MC}}$}
& \textbf{AUROC} & \textbf{AUPRC}
& \textbf{AUROC} & \textbf{AUPRC} \\
\midrule

1
& $0.878_{\pm0.003}$
& $0.549_{\pm0.008}$
& $0.892_{\pm0.003}$
& $0.594_{\pm0.004}$ \\

3
& $0.879_{\pm0.006}$
& $0.554_{\pm0.004}$
& $0.895_{\pm0.001}$
& $0.597_{\pm0.001}$ \\

5$^\dagger$
& $0.876_{\pm0.002}$
& $0.557_{\pm0.001}$
& $\mathbf{0.901_{\pm0.001}}$
& $\mathbf{0.609_{\pm0.001}}$ \\

10
& $\mathbf{0.880_{\pm0.004}}$
& $\mathbf{0.559_{\pm0.002}}$
& $0.899_{\pm0.001}$
& $0.601_{\pm0.002}$ \\

\bottomrule
\end{tabular}
\end{table}

\end{document}